\documentclass{article}

\PassOptionsToPackage{numbers, compress}{natbib}
\usepackage[main,final]{neurips_2026}

\usepackage[utf8]{inputenc} 
\usepackage[T1]{fontenc}    
\usepackage{hyperref}       
\hypersetup{
    hidelinks
}
\usepackage{url}            
\usepackage{booktabs}       
\usepackage{amsfonts}       
\usepackage{nicefrac}       
\usepackage{microtype}      
\usepackage[dvipsnames]{xcolor}
\usepackage{amsmath}
\usepackage{amsthm}
\usepackage{graphicx}
\usepackage{multirow}
\usepackage{subcaption}
\usepackage{enumitem}
\usepackage{xcolor}
\usepackage[normalem]{ulem}
\usepackage{algorithm}
\usepackage{algpseudocode}
\usepackage{xspace}
\usepackage{etoc}
\usepackage{makecell}
\usepackage{thmtools}
\usepackage{thm-restate}
\usepackage{fontawesome5}

\definecolor{oursRA}{HTML}{003B67}
\definecolor{oursABP}{HTML}{8C3B00}
\definecolor{lqrYellow}{HTML}{D8A33A}

\newtheorem{theorem}{Theorem}[section]

\newtheorem{remark}{Remark}[section]
\newtheorem{definition}{Definition}[section]
\newtheorem{problem}{Problem}
\newtheorem{subproblem}{Subproblem}
\newtheorem{proposition}{Proposition}[section]
\newtheorem{assumption}{Assumption}[section]

\newcommand{\B}{{\mathcal B}}
\newcommand{\D}{{\mathcal D}}
\newcommand{\C}{{\mathcal C}}
\newcommand{\K}{{\mathcal K}}
\newcommand{\Kb}{{\mathbb K}}
\newcommand{\CT}{{\mathcal C_T}}
\newcommand{\X}{{\mathcal X}}

\newcommand{\F}{{\mathcal F}}
\newcommand{\calS}{{\mathcal S}}

\newcommand{\calH}{{\mathcal H}}
\newcommand{\U}{{\mathcal U}}
\newcommand{\R}{{\mathbb R}}
\newcommand{\calP}{{\mathcal P}}
\newcommand{\flowB}{{\Phi_{\pi_b}}}
\newcommand{\dt}{{\Delta t}}
\newcommand{\td}{{t'}}

\newcommand{\clip}{{\mathrm{clip}}}
\newcommand{\reftext}{{\mathrm{ref}}}
\newcommand{\safe}{{\mathrm{safe}}}
\newcommand{\des}{{\mathrm{des}}}
\newcommand{\err}{{\mathrm{err}}}
\newcommand{\cmdtext}{{\mathrm{cmd}}}
\newcommand{\diag}{{\mathrm{diag}}}
\newcommand{\BL}{{\mathrm{CIL}}}
\newcommand{\vb}[1]{\mathbf{#1}}

\newcommand{\acmd}{{a_{\text{cmd}}}}
\newcommand{\wcmd}{{\mathbf{\omega}_{\text{cmd}}}}
\newcommand{\cbfrl}{{\mathrm{cbf-rl}}}

\newcommand{\lbp}{{\pi_b^\theta}}
\newcommand{\lbpstar}{{\pi_b^\star}}
\newcommand{\bp}{{\pi_b}}
\newcommand{\piphistar}{{\pi_\phi^\star}}
\newcommand{\pinom}{{\pi_\phi}}
\newcommand{\piarr}{{\pi_{\mathrm{SA}}}}
\newcommand{\piarrL}{{\pi_{\mathrm{SA}}^\theta}}
\newcommand{\piarrstar}{{\pi_{\mathrm{SA}}^\star}}

\newcommand{\unom}{{u_{\mathrm{nom}}}}
\newcommand{\usafe}{{u_{\mathrm{safe}}}}

\newcommand{\bcbf}{{\mathrm{BCBF}}}

\newcommand{\hardCVX}{\texttt{HardNet-CVX}}
\newcommand{\hardnet}{\texttt{HardNet}}

\newcommand{\ours}{PS2}
\newcommand{\oursRA}{PS2-RL}
\newcommand{\oursSAV}{PS2$_{\mathrm{SA}}$}
\newcommand{\oursABP}{PS2$_{\mathrm{ABP}}$}

\newcommand{\rafull}{Safe-Arrival Value\xspace}
\newcommand{\ralow}{safe-arrival value\xspace}
\newcommand{\sapolicy}{safe-arrival policy\xspace}

\DeclareMathOperator*{\argmin}{arg\,min}

\newif\ifisrestate
\isrestatefalse

\title{Provably Safe, Yet Scalable\\Reinforcement Learning}

\author{
  Kai S.~Yun \\
  MIT  \\
  \texttt{kaisyun@mit.edu} \\
  \And
  Zeyang~Li\\
  MIT \\
  \texttt{zeyang@mit.edu} \\
  \And
  Navid~Azizan \\
  MIT \\
  \texttt{azizan@mit.edu} \\
}

\begin{document}

\maketitle

\begingroup
\renewcommand{\thefootnote}{}
\footnotetext[0]{Correspondence to Zeyang Li;  
\href{https://azizanlab.github.io/PS2-RL}{\faGlobe\ \textcolor{oursRA}{https://azizanlab.github.io/PS2-RL}},
\href{https://github.com/azizanlab/PS2-RL}{\faGithub\ \textcolor{oursRA}{https://github.com/azizanlab/PS2-RL}}.}
\endgroup

\addtocontents{toc}{\protect\setcounter{tocdepth}{0}}
\begin{abstract}
    Safe reinforcement learning (RL) aims to learn policies that optimize rewards while satisfying constraints.
    Predominant approaches rely on soft-constrained policy optimization, which has achieved empirical success but does not provide formal safety guarantees for the learned policy. 
    In contrast, methods with strict guarantees typically rely on explicit certificate functions, whose construction requires the direct synthesis and verification of control-invariant sets, a process that scales poorly with state dimension and often yields overly conservative behavior.
    In this paper, we present the Provably Safe, \emph{yet} Scalable RL (\oursRA) framework, a novel two-phase architecture for learning provably safe policies in a scalable manner, designed to overcome the key bottlenecks of prior methods.
    Rather than explicitly computing invariant sets, \oursRA~leverages a learned backup policy to forward-integrate the system dynamics, generating an implicit control-invariant set online. 
    In the first phase, the backup policy is trained with the \ralow{} function to achieve rapid, safe recovery to the certified base set for invariant-set construction.
    In the second phase, an RL policy is trained end-to-end through a differentiable projection layer that strictly enforces the safety guarantees induced by the learned backup policy. 
    By maximizing the volume of the implicit control-invariant set in the first phase, the resulting \ours{} policy from the second phase is performant and scalable, while maintaining provable safety. 
    Crucially, \oursRA{} imposes no restrictions on the underlying RL algorithm and can be plugged into any existing training pipeline. 
    We establish theoretical guarantees for the proposed framework and evaluate it on robotic control tasks with state dimensions up to 10, a regime in which prior provably safe RL methods struggle or become impractical.
\end{abstract}


\section{Introduction}\label{sec: intro}

Reinforcement learning (RL) has achieved remarkable success in controlling complex robotic systems~\citep{hwangbo2019anymal, pmlr-v164-rudin22a, kaufmann2023champion, renzhi2026rlHumanoid}. 
However, the lack of safety guarantees still hinders the deployment of RL on real-world systems.
To mitigate this issue, safe RL aims to learn policies that achieve high rewards while satisfying safety constraints. 
Predominant approaches use constrained policy optimization, such as Lagrangian-based methods, where constraints are typically imposed through either cost value functions or neural certificate functions~\citep{cmdp, cpo, chow2017risk, tessler2018reward, ha2021sacLag, ma2022joint, yu2022reachability, li2024safe}. 
A key limitation of this line of work is that, although the learned policies may perform well empirically in some cases, they provide no formal guarantee of constraint satisfaction; as a result, catastrophic failures can still occur during deployment.
Another line of work in safe RL seeks to learn provably safe policies \citep{cheng2019cbfRL, cbfRL, tonkens2022cbfWithHJ, alshiekh2018shielding, zhao2025implicitsafeset, bastani2021mps}. 
These methods typically require verified explicit certificate functions representing control-invariant sets, such as control barrier functions (CBFs)~\citep{ames2014control, ames2017cbf, ames2019control} or safety indices (SIs)~\citep{liu2014control, wei2019safe}, whose synthesis and verification scale poorly with state dimension~\citep{bansal2017hamilton, mitchell2005HJ, raCBF-ours, dai2022convex}.
Therefore, these certificate-based methods often yield overly conservative behavior~\citep{ncbf-simin, backupCBF}.

The backup control barrier function (BCBF) framework~\citep{backupCBF, gurriet2020backupcbf_access} offers a promising way to avoid the fundamental difficulty of directly synthesizing a valid explicit control-invariant set. 
Its key idea is to enlarge a small, known safe invariant set by forward-integrating a deterministic backup policy. 
The implicit set induced by this rollout is control-invariant by construction, respects bounded actuator limits, and reduces the safety condition to relative degree one regardless of the underlying dynamics. 
Thus, safety can be specified in its original form without high-order-style augmentations~\citep{xiao2022hocbf, xiao2022hocbfReview}. 
Crucially, this framework also preserves the convexity of the resulting safety-filter constraints, allowing them to be solved efficiently.

However, BCBF shifts rather than eliminates the difficulty: it replaces the direct synthesis of an explicit control-invariant set with the design of a backup policy, but does not provide a principled way to obtain such a policy. 
The quality of the backup policy critically affects the size of the resulting invariant set. 
Hand-designed analytic policies, such as linear quadratic regulators (LQRs), are easy to certify but typically recover only a small portion of the full invariant set~\citep{backupCBF, kim2026backupbasedsafetyfilterscomparative}. 
In this paper, we use BCBF as a building block and address this limitation by training the safe-arrival component of the backup policy via RL. 
The learned policy steers the system into a small certified region around an equilibrium, where a closed-form linear controller takes over. 
The BCBF construction then automatically extends safety from the certified base set to the entire implicit invariant set, without requiring certification of the learned safe-arrival policy itself.

\begin{figure}[t]
    \centering
    \includegraphics[width=1.0\linewidth]{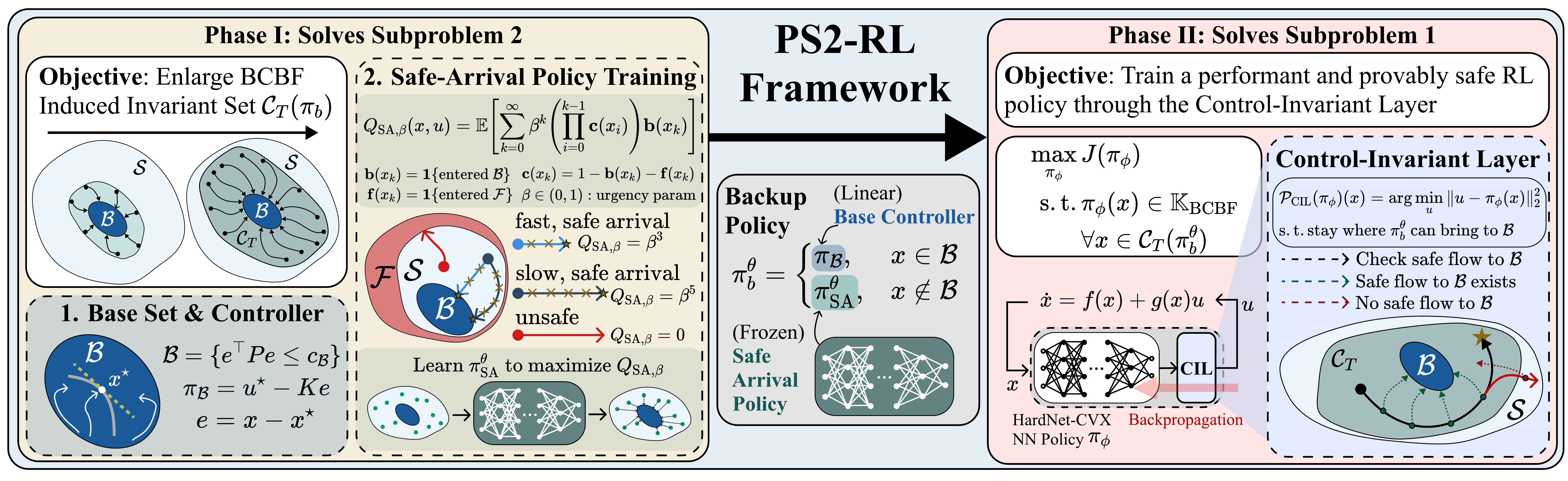}
    \caption{Overview of the \oursRA{} framework.}
    \label{fig: framework_hero}
\end{figure}

To this end, we introduce \textbf{\oursRA~(Provably Safe, \emph{yet} Scalable RL)}, a two-phase framework for training provably safe and scalable policies without excessive conservatism or performance sacrifice. 
A visual overview of our method is in Fig.~\ref{fig: framework_hero}. 
Phase I trains a safe-arrival policy with RL through the indicator-based \ralow{} function, whose unique optimum induces the time-optimal policy that drives the system safely to a target set. 
Phase II trains a \ours{} policy \emph{end-to-end} through a control-invariant layer (CIL), constructed from the learned backup policy in Phase I. 
On unicycle lane-keeping and powerloop tracking for a $10$-dimensional quadrotor, \oursRA{} achieves $100\%$ safety while exceeding the performance of all baselines. 
In summary, our \textbf{contributions} are:
\begin{itemize}[itemsep=0ex, topsep=0pt, parsep=0pt, leftmargin=1.5em]

    \item \textit{\oursRA{} framework}: 
    A two-phase RL framework with formal safety guarantees and scalability to high-dimensional, input-constrained systems, without synthesizing an explicit invariant set.

    \item \textit{Backup policy training}: We develop a reinforcement learning approach to training backup policies for implicit invariant-set construction. The safe-arrival value function promotes rapid, safe recovery to a certified region around an equilibrium. Its self-consistency condition and Bellman equation enable backup-policy training with standard RL algorithms. Empirically, the learned backup policies yield substantially larger implicit control-invariant sets than analytic alternatives.

    \item \textit{Control-invariant layer}: 
    A differentiable projection layer that enforces the BCBF constraints, enabling efficient end-to-end training of provably safe RL policies. 
    Crucially, we show that the differentiable projection does not compromise the expressiveness of the policy network, preserving universal approximation.

    \item \textit{Theoretical justifications}: 
    Comprehensive theoretical results for the proposed framework, including guarantees for each individual phase and their connection in the full \oursRA{} pipeline.

\end{itemize}

\paragraph{Related work.}
We group related studies into five categories: (i) safe RL via constrained policy optimization, (ii) safe RL with verified certificates, (iii) backup control barrier functions, (iv) RL for reach-avoid specifications, and (v) differentiable optimization layers.
See App.~\ref{app: related_work} for details.

\section{Preliminaries}

\subsection{Backup Control Barrier Functions}\label{subsec: safe_control}
Throughout this work, we consider the following control-affine, continuous-time system:
\begin{align}\label{eq: affine_dyn}
    \dot x = f(x) + g(x)u,\quad x\in\X\subseteq\R^n, \quad u\in\U\subseteq\R^m
\end{align}
where $x$ is the state, $\X$ is the state space, $u$ is the control input, and $\U$ is a compact set of admissible inputs encoding actuator limits. The functions $f$ and $g$ are locally Lipschitz continuous. Given a state-feedback policy $\pi:\X\rightarrow\U$, let $\Phi_{\pi}: \R^n\times\R \rightarrow\R^n$ denote the closed-loop flow map under $f_\pi(x):= f(x) + g(x) \pi(x)$, where $\Phi_\pi(x_0,t)$ is the state at time $t$ starting from $x_0:=x(t_0)$. 

Safety of the system~\eqref{eq: affine_dyn} is specified with a \textbf{safe set} $\calS:=\{ x\in\X: h_\calS(x)\ge0 \}$ for a continuously differentiable $h_\calS:\R^n\rightarrow\R$, and enforced in the sense of \textit{control invariance}~\citep{nagumo, ames2019control, liu2014control}.
Control barrier functions (CBFs) provide a differentiable certificate of control invariance via an inequality condition $\nabla h(x)^\top(f(x)+g(x)u)+\alpha(h(x))\ge 0$ enforced online within a safety filter~\citep{ames2014control, ames2017cbf}.
The key challenge lies in the synthesis of a CBF whose zero-superlevel set is invariant.
As discussed in Sec.~\ref{sec: intro} and App.~\ref{app: related_work}, classical tools for computing the explicit invariant set have various limitations.

The backup control barrier function (BCBF) framework~\citep{backupCBF, gurriet2020backupcbf_access} circumvents this by leveraging a deterministic \textbf{backup policy} $\pi_b:\X\rightarrow\U$ and a small, known control-invariant \textbf{base set} $\B:=\{x: h_\B(x)\ge0\}\subseteq \calS$. Then, the $T$-time constrained backup-induced control-invariant set under $\pi_b$ is: 
\begin{definition}[Backup-Induced Control-Invariant Set]
    $\CT(\bp)$ is the set of all states from which $\bp$ recovers to $\B$ within the backup horizon $T$ while keeping the trajectory in $\calS$:
    \begin{align}\label{eq: bcbf_condition}
        \C_T(\bp) = \mathcal R\left( \B,\calS, f_{\pi_b}, T \right) := \left\{ x\in\X: \flowB(x,T)\in\B \; \wedge \flowB (x,t')\in\calS \quad \forall t'\in[0,T] \right\}.
    \end{align}
\end{definition}
$\C_T(\bp)$ is control-invariant under $\bp$ with $\B\subseteq\C_T\subseteq\calS$, and yields the BCBF, an implicit CBF:
\begin{align}\label{eq: bcbf_implicit_set}
    h_\CT(x) = \min\left\{ \min_{t'\in[0,T]} \big\{h_\calS\left( \flowB(x,t') \right)\big\}, h_\B\left( \flowB(x,T) \right)\right\}
\end{align}
The non-smooth $\min$ in~\eqref{eq: bcbf_implicit_set} is handled by enforcing the CBF inequality condition pointwise along the backup flow and at the horizon. 
The resulting BCBF constraints are sufficient for the original CBF condition on $h_\CT$~\citep[Prop.~2]{backupCBF} and are always feasible by construction under control limits~\citep[Thm.~2]{backupCBF}, without ever requiring an explicit invariant set computation, making BCBF an ideal backbone for \oursRA.
We defer the explicit forms of CBF and BCBF constraints to Sec.~\ref{sec: sub1} and App.~\ref{app: safe_control_background}.

\subsection{Reinforcement Learning}\label{sec: rl-prelim}
We model the control task as a Markov decision process (MDP) $(\X,\U,F, r,\gamma)$ with deterministic transitions $F:\X\times\U\rightarrow\X$, reward $r:\X\times\U\rightarrow\R$, and discount factor $\gamma\in(0,1)$. 
A neural policy $\pinom:\X\to\U$, parameterized by $\phi$, maps states to control inputs. 
The objective is to find a policy that maximizes expected discounted return $J(\pinom)=\mathbb E[\sum_{k=0}^\infty \gamma^kr(x_k,u_k)]$ for an initial-state distribution $\rho_0$, under closed loop dynamics $x_{k+1} = F(x_k,\pinom(x_k))$ with $x_0\sim\rho_0$.
The state-action value function $Q_\gamma^\pinom(x,u):= \mathbb E_{x,u}^{\pinom}\!\left[\sum_{k=0}^\infty \gamma^k r(x_k,u_k)\right]$ satisfies the self-consistency condition $Q^{\pi_\phi}_\gamma(x,u) = r(x,u) + \gamma Q_\gamma^\pinom( F(x,u),\,\pinom( F(x,u) ) )$ for policy $\pinom$.

\subsection{Hard-Constrained Neural Networks}
A core requirement of \oursRA ~is to enforce hard constraints on a neural network's output, while preserving differentiability for end-to-end training. 
We build on \hardCVX~\citep{hardnet}, which projects a neural network's output onto a state-dependent convex set via a differentiable layer. 
In our notation, we refer to this construction as the \textit{constraint projection layer}:
\begin{definition}[Constraint Projection Layer]\label{def: safety_proj_layer}
Given an RL policy $\pinom:\X\rightarrow\U$ and a state-dependent convex set $\Kb(x)\subseteq\U$, the constraint projection layer $\calP$
 is defined as:
\begin{align}
    \calP(\pinom)(x) = \argmin_{u} \left\| u-\pinom(x) \right\|_2^2 \quad \mathrm{s.t.} \quad u \in \Kb(x).
\end{align}
\end{definition}
This projection is computed by a differentiable convex optimization layer, allowing policy gradients to propagate through $\calP$ during end-to-end training. 
Crucially, $\calP$ inherits the universal approximation property of \hardCVX~\citep[Thm~12]{hardnet}.
In Sec.~\ref{sec: main_method}, we instantiate $\Kb(x)$ with affine constraints derived from the BCBF framework to define \oursRA's \textit{control-invariant layer}.

\section{Problem Formulation}\label{sec: problem}
We now formalize our main problem: synthesize a reward-maximizing RL policy that is provably safe under control limits for every initial state of interest. 
First, we make the following connection:

\begin{remark}
\label{remark: ct_dt}
    The policy acts on the known deterministic dynamics~\eqref{eq: affine_dyn} via zero-order hold (ZOH): at each sampling instant $t_k=k\dt$, we define $x_k:=x(t_k)$ and $u_k:=\pi(x_k)$, and hold $u(t)=u_k$ for all $t\in[t_k,t_{k+1})$.
    $F$ in Sec.~\ref{sec: rl-prelim} is the exact one-step map of~\eqref{eq: affine_dyn} under ZOH.
    Our theoretical guarantees are for this sampled-data discrete dynamics. 
    Under standard Lipschitz regularity and a sufficiently fast control rate, the sampled-data set-invariance conditions guarantee that BCBF conditions enforced at sample instants imply continuous-time safety of the ZOH trajectory~\citep{singletary2020sampleddata, gurriet2019setInvariance, backupCBF}.
\end{remark}

With this connection between continuous- and discrete-time established, we present our main problem:

\begin{problem}[\oursRA{} Formulation]\label{prob: main}
Given the control-affine system~\eqref{eq: affine_dyn} and its ZOH discretization $F$, a safe set $\calS\subseteq\X$, an admissible control input set $\U\subseteq\R^m$, an initial-state distribution $\rho_0$ supported on $\X_0\subseteq\calS$, a reward function $r:\X\times\U\to\R$ for task performance, and a discount factor $\gamma\in(0,1)$, find the policy $\pi$ that solves
\begin{align}\label{eq: problem1}
    &\max_{\pi}\quad
    J(\pi) \;=\; \mathbb{E}_{\substack{
    x_0 \sim \rho_0,\;u_k=\pi(x_k) \\
    x_{k+1} = F(x_k,u_k)
    }}
    \!\left[\sum_{k=0}^\infty \gamma^k\, r(x_k, u_k)\right] \\[-1pt]
    &\mathrm{s.t.}\quad
    x_k \in \calS, \quad u_k \in \U,
    \quad \forall k \ge 0,\;\; \forall x_0 \in \X_0. \notag
\end{align}
\end{problem}
We use $\Pi_{\safe}$ to denote the feasible policy class of Problem~\ref{prob: main}, which constrains each realized control input $u_k$ and the resulting state $x_k$, i.e., the requirements $x_k\in\calS$ and $u_k\in\U$ must hold pointwise in $k$, for every trajectory induced by $\pi$, and for every initial state in $\X_0$. We require the learned policy for Problem~\ref{prob: main} to be provably safe, i.e., to provably belong to $\Pi_{\safe}$. 
Note that predominant safe RL algorithms based on constrained policy optimization~\citep{cmdp, cpo, chow2017risk, tessler2018reward, ha2021sacLag, ma2022joint, yu2022reachability, li2024safe, li2024safe2} may adopt similar formulations, such as state-wise safe RL~\citep{zhao2023statewisesafereinforcementlearning}. 
However, the policies learned by these methods generally do not come with formal guarantees of membership in $\Pi_{\safe}$.

\section{Provably Safe, \emph{yet} Scalable RL}\label{sec: our_method}

Problem~\ref{prob: main} requires maximizing task return while guaranteeing safety under bounded control. 
\oursRA{} decomposes this into two subproblems (Sec.~\ref{sec: decompose}), addressed sequentially in two phases.
We first train a safe-arrival policy using the \ralow function to quickly drive the system to a certified base set while avoiding unsafe states (Secs.~\ref{subsec: sav_func}--\ref{sec: phase1}). 
Second, the resulting composed backup policy instantiates the CIL, a differentiable projection layer mapping inputs to the BCBF-admissible set. 
This enables end-to-end training of an RL policy through the safety constraints (Sec.~\ref{sec: main_method}). 
By enlarging the implicit invariant set before optimizing within it, \oursRA{} achieves both high returns and guaranteed safety in a scalable manner.
All proofs are deferred to the Appendix.

\subsection{Decomposing Safety and Performance from Problem~\ref{prob: main}}\label{sec: decompose}

We decompose Problem~\ref{prob: main} into two subproblems: a BCBF-constrained policy optimization for a fixed backup policy $\pi_b$ and horizon $T$ (Sec.~\ref{sec: sub1}) and a safe-arrival set maximization problem for $\pi_b$ (Sec.~\ref{sec: sub2}), and show their composition certifies an inner approximation of Problem~\ref{prob: main} (Sec.~\ref{sec: decompose_guarantee}).

\subsubsection{BCBF-constrained policies under a fixed backup policy}\label{sec: sub1}

\paragraph{BCBF-admissible input set.} 
We enforce safety by restricting the current input so that a safe backup to $\B$ remains available. 
At each state $x$, we compute how the backup rollout under $\bp$ changes with its initial state. 
These sensitivities relate a candidate input's instantaneous effect to changes in the rollout's safety and terminal margins. 
For relative time $\td\in[0,T]$, let $\D_\td^\calS(x,u)$ denote the safety-margin derivative at $t+\td$, and let $\D_T^\B(x,u)$ denote the terminal-margin derivative at horizon $T$.
Their derivations are given in App.~\ref{app: bcbf_derivation}.
Given extended class-$\K_\infty$ functions $\alpha_\calS$ and $\alpha_\B$,
\begin{equation}\label{eq: bcbf_set}
\Kb_\bcbf(x;\bp) := \left\{u\in\U\,:
\begin{aligned}
    &\D_\td^\calS(x,u)
    +\alpha_\calS\!\left(h_\calS(\Phi_\bp(x,\td))\right)\ge0,
    \quad\forall\td\in[0,T],\\
    &\D_T^\B(x,u)
    +\alpha_\B\!\left(h_\B(\Phi_\bp(x,T))\right)\ge0
\end{aligned}
\right\}
\end{equation}
is the BCBF-admissible input set. 
Given $x$, these constraints are affine in $u$ since the backup rollout and sensitivity are computed under a fixed $\bp$. 
Hence $\Kb_\bcbf(x;\bp)$ is convex whenever $\U$ is convex. 
Moreover, $\bp(x)$ is feasible for~\eqref{eq: bcbf_set} for all $x\in\CT(\bp)$, ensuring nonemptiness~\citep[Thm.~2]{backupCBF}. 
With this BCBF-admissible input set, we present our first subproblem:

\begin{subproblem}[BCBF-Constrained Policy Optimization]\label{prob: sub1}
    Fix a backup policy $\bp$ and backup horizon $T$, and assume $\X_0\subseteq\CT(\bp)$. Find a $\phi$-parameterized policy $\pi_\phi$ that maximizes the RL objective while satisfying the BCBF constraints:
    \begin{align}
        \max_\pinom\quad J(\pinom)\quad
        \mathrm{s.t.}\quad \pi_\phi(x)\in\Kb_\bcbf(x;\bp)\quad\forall x\in\CT(\bp).
    \end{align}
\end{subproblem}
\begin{remark}\label{remark: bcbf_mesh}
    In practice, we enforce the constraints on a finite relative-time mesh $0=t'_0<t'_1<\cdots<t'_N=T$.
    Under standard Lipschitz regularity~\citep{backupCBF} and Remark~\ref{remark: ct_dt}, sampled BCBF enforcement at instants $t_k$ implies continuous-time safety of the ZOH trajectory. 
\end{remark}

\subsubsection{Learning the safe-arrival policy to enlarge the implicit invariant set}\label{sec: sub2}

The conservatism in Subproblem~\ref{prob: sub1} is governed by the backup-induced set $\CT(\pi_b)$: 
a small $\CT(\pi_b)$ gives the policy a narrow region in which to optimize.
To enlarge this region without losing the BCBF guarantee, we anchor the construction of the backup policy with a certified base set and use a learnable safe-arrival policy only to steer states to that base set, as follows:
\begin{align}\label{eq: bp_decompose}
    \bp(x) = \begin{cases}
        \pi_\B(x), & x\in\B \\
        \piarr(x), & x\notin\B.
    \end{cases}
\end{align}
Here, $\pi_\B:\B\to\U$ is the \textbf{base controller}, which renders the base set $\B$ forward invariant under $f_{\pi_\B}$, and $\piarr:\X\setminus\B\to\U$ is the \textbf{safe-arrival policy}, which is tasked with steering the system from a state outside $\B$ into $\B$ within the backup horizon $T$ while staying in the safe set $\calS$.
While the terminal base pair $(\B,\pi_\B)$ needs to be certified for forward invariance, the \sapolicy $\piarr$ does not, as its rollouts implicitly define the backup-induced invariant set~\eqref{eq: bcbf_condition}.

\paragraph{Certified base set and base controller.}
The BCBF framework and \oursRA{} require a certified base pair.
The following summarizes Theorem~\ref{thm: base_set}, whose construction and proof are provided in App.~\ref{app: base_set_cert}.

\begin{theorem}[Certified Base Pair]\label{thm: base_pair}
Under Assump.~\ref{assume: local_stable}, there exist a nonempty base set $\B=\{x\in\X:h_\B(x)\ge0\}\subseteq\calS$ and an admissible controller $\pi_\B:\B\to\U$ that renders $\B$ forward invariant.
\end{theorem}

The construction uses stabilizing linear feedback and a sufficiently small quadratic Lyapunov sublevel set, requiring certification of only the local invariant base set rather than synthesis of a global invariant set, which is often intractable for high-dimensional systems.
Hereafter, we fix one such pair $(\B,\pi_\B)$.

With $(\B,\pi_\B)$ fixed, $\CT(\bp)$ is governed by the \sapolicy $\piarr$ and the backup horizon $T$. 
Increasing $T$ monotonically enlarges $\CT(\bp)$ in the set-inclusion sense~\citep[Lemma 1]{backupCBF}, but it also increases the number of constraints in $\Kb_\bcbf$ and could render the optimization problem inefficient. 
A more straightforward solution is to come up with a better $\piarr$ that can bring more states in $\calS$ into $\B$ given a fixed horizon.
To this end, we introduce our second subproblem, learning a \sapolicy $\piarr$ to maximize the \textit{safe-arrival set}, and thereby the BCBF-induced invariant set $\CT$.

\begin{subproblem}[Safe-Arrival Set Maximization]\label{prob: sub2}
Given a certified base set $\B$ and base controller $\pi_\B$ from Theorem~\ref{thm: base_pair}, a backup horizon $T>0$, a compact design region $\Omega\subseteq\calS$, and a finite reference measure $\mu$ on $\Omega\setminus\B$, find a $\theta$-parameterized safe-arrival policy $\piarrL:\X\setminus\B\to\U$ that maximizes the measure of states in $\Omega\setminus\B$ from which it can safely arrive at $\B$ within $T$:
\begin{align}\label{eq: backup_recoverable_set_obj}
    \max_{\piarrL}\quad \mu_\mathrm{SA}\left(\piarrL;\,\Omega,\,\B\right) := 
    \mu\!\left(\CT( \lbp ) \cap \left(\Omega\setminus\B\right)\right)
    \quad
    \mathrm{s.t.} \quad
    \lbp(x)=
    \begin{cases}
        \pi_\B(x), & x\in\B,\\
        \piarrL(x), & x\notin\B.
    \end{cases}
\end{align}
\end{subproblem}
\subsubsection{Certified decomposition guarantee}\label{sec: decompose_guarantee}
Subproblem~\ref{prob: sub2} enlarges the backup-induced invariant set $\CT(\lbp)$ over which safety is guaranteed, and motivates the \ralow formulation (Sec.~\ref{subsec: sav_func}). 
Subproblem~\ref{prob: sub1} then optimizes the task return inside the corresponding BCBF-admissible input set~\eqref{eq: bcbf_set}. 
Accordingly, for a fixed backup policy $\bp$, we define the BCBF-feasible policy class as $\Pi_{\bcbf}(\bp):=\{\pi:\X\to\U : \pi(x)\in\Kb_{\bcbf}(x;\bp)\,\,\forall x\in\CT(\bp)\}$.
The following result formalizes this decomposition.
\begin{restatable}[Certified Decomposition and Exactness of \oursRA]{theorem}{certifythm}\label{thm: decomposition}
    Suppose $(\B,\pi_\B)$ satisfies Theorem~\ref{thm: base_pair}. 
    Let $\lbpstar$ be a solution of Subproblem~\ref{prob: sub2}, and let $\piphistar$ be a solution of Subproblem~\ref{prob: sub1} with the fixed backup policy $\lbpstar$.
    Assume the exact BCBF constraints in~\eqref{eq: bcbf_set} are enforced. Then: 
    (i) $\Pi_\bcbf\left( \lbpstar \right)\subseteq\Pi_\safe$ and hence $\piphistar$ is feasible for Problem~\ref{prob: main};
    (ii) $\piphistar$ is optimal among all policies certified by $\lbpstar$, i.e., $J( \piphistar ) = \sup_{\pi\in\Pi_\bcbf( \lbpstar)} J(\pi) \le \sup_{\pi\in\Pi_\safe} J(\pi)$; 
    and (iii) if there exists a globally optimal solution $\pi^\star_\safe$ of Problem~\ref{prob: main} such that $\pi^\star_\safe\in\Pi_\bcbf\left(\lbpstar\right)$, then $\piphistar$ is also globally optimal for Problem~\ref{prob: main}.
    Thus, the two subproblems solve a certified inner approximation of Problem~\ref{prob: main}. They solve Problem~\ref{prob: main} exactly whenever the learned backup policy induces a BCBF-feasible class containing an optimal safe policy. 
\end{restatable}

\subsection{\rafull Function}\label{subsec: sav_func}
Subproblem~\ref{prob: sub2} optimizes the backup policy to maximize the size of the induced implicit invariant set, thereby imposing fewer restrictions on the RL policy in Subproblem~\ref{prob: sub1} and enabling potentially higher rewards. 
To this end, we introduce the \ralow{} function, an indicator-style objective that quantifies a policy's ability to safely return the system to the target set \citep{summers2010verification,haesaert2016experiment}. 
Optimizing the \ralow{} function encourages the policy to reach $\B$ quickly while avoiding the failure set $\F := \X \setminus \calS$, thereby enlarging $\CT(\bp)$. 
We define $\vb b(x) := \vb1_\B(x)$, $\vb f(x) := \vb1_\F(x)$, and $\vb c(x) := 1 - \vb b(x) - \vb f(x)$, so that $\vb c(x)=1$ precisely on the continuation set $\X \setminus (\B \cup \F)$. 
We refer to the policy optimized for this purpose as the safe-arrival policy $\piarr$.

\begin{definition}[\rafull Function]\label{def: save_arrival_value}
    For a deterministic safe-arrival policy $\piarr$ and discount factor $\beta\in(0,1)$, the discounted safe-arrival Q-function is:
    \begin{align}\label{eq: save_q_func}
        Q_{\mathrm{SA},\beta}^\piarr(x,u):= \mathbb E_{x,u}^\piarr\left[ \sum_{k=0}^\infty \beta^k \left( \prod_{\tau=0}^{k-1} \mathbf{c}(x_\tau) \right) \mathbf{b}(x_k)\right],
    \end{align}
    where $x_0=x, u_0=u, x_{k+1}=F(x_k,u_k)$, and $u_k=\piarr(x_k)$ for $k\ge1$, with the empty product interpreted as $1$. The state-value function is $V_{\mathrm{SA},\beta}^\piarr(x):=Q_{\mathrm{SA},\beta}^\piarr(x,\piarr(x))$.
\end{definition}
The factor $\prod_{\tau=0}^{k-1}\mathbf{c}(x_\tau)$ is a ``survival'' gate: it is $1$ only while the rollout remains in the continuation set, and becomes $0$ immediately after the trajectory enters either $\B$ or $\F$. 
Hence~\eqref{eq: save_q_func} rewards exactly the first safe arrival to $\B$, discounted by how many steps it takes. 
Under deterministic dynamics,~\eqref{eq: save_q_func} equals $\beta^N$ if the rollout reaches $\B$ safely in exactly $N$ steps, and $0$ otherwise. 
Consequently, maximizing the \ralow function trains a policy that first prefers safe arrival over failure, and among safe-arrival policies prefers those that reach $\B$ sooner. 
Importantly, the \ralow{} function satisfies the following self-consistency condition:
\begin{equation}\label{eq: save_bellman}
    Q_{\mathrm{SA},\beta}^{\piarr}(x,u)
    =
    \vb b(x)
    +
    \beta \, \vb c(x)\,
    Q_{\mathrm{SA},\beta}^{\piarr}
    \left(F(x,u), \piarr\left(F(x,u)\right)\right),
\end{equation}
as well as a corresponding Bellman equation. 
This structure enables the safe-arrival policy to be trained using standard RL pipelines.
The full details and additional properties of the safe-arrival value function are deferred to App.~\ref{app: save}. 

\subsection{Phase I: Training the Safe-Arrival Policy}\label{sec: phase1}
Phase I addresses Subproblem~\ref{prob: sub2} (Fig.~\ref{fig: framework_hero}, left panel) by training a parameterized safe-arrival policy $\piarrL$ with the discounted Bellman recursion~\eqref{eq: save_bellman}. 
Letting $\rho_{\mathrm{arr}}$ be a design distribution supported on $\Omega\setminus\B$, we optimize the following objective:
\begin{equation}\label{eq: save_phase1_surrogate}
    \max_{\theta}\; J_{\mathrm{SA},\beta}(\theta)
    :=
    \mathbb E_{x\sim \rho_{\mathrm{arr}}}\!\left[Q_{\mathrm{SA},\beta}^{\piarrL}(x,\piarrL(x))\right].
\end{equation}
The general Phase~I training procedure and algorithm are in App.~\ref{app: save_beta}. 
Note that the only ingredients specific to the safe-arrival policy training are the indicators $\vb b,\vb f,\vb c$, first-hit termination at $\B$ and $\F$, and~\eqref{eq: save_phase1_surrogate}. 
Everything else can be supplied by a chosen RL backbone. 
Hence, the safe-arrival policy training can be instantiated with actor-critic, Q-learning, and related methods. 


\subsection{Phase II: Training with the Control-Invariant Layer}\label{sec: main_method}
Phase II solves Subproblem~\ref{prob: sub1} (Fig.~\ref{fig: framework_hero}, right panel). 
With $\piarrstar$, the safe-arrival policy trained from Phase I, we define the fixed composed backup policy by setting $\pi_b^\star(x)=\pi_B(x)$ for $x\in\mathcal B$ and $\pi_b^\star(x)=\pi^\star_{\mathrm{SA}}(x)$ otherwise. 
With $\lbpstar$ fixed, the backup-induced invariant set $\CT(\lbpstar)$ and the BCBF-admissible set $\Kb_\bcbf(x;\lbpstar)$ are fixed as well. Assuming $\U$ is convex, $\Kb_\bcbf(x;\lbpstar)$ is convex for each $x\in\CT(\lbpstar)$ because the BCBF constraints are affine in $u$ under a fixed backup flow.
Extending from the constraint projection layer (Def.~\ref{def: safety_proj_layer}), we define the \textit{control-invariant layer} as follows.

\begin{definition}[Control-Invariant Layer]\label{def: backup_layer}
Given the fixed composed backup policy $\lbpstar$ and backup horizon $T$, the control-invariant layer is the state-dependent projection operator
\begin{align}\label{eq: backup_layer}
    \calP_\BL(\pinom)(x)
    :=
    \argmin_{u}\,\|u-\pinom(x)\|_2^2
    \quad
    \mathrm{s.t.}\quad
    u\in \Kb_\bcbf(x;\lbpstar).
\end{align}
\end{definition}
\eqref{eq: backup_layer} is essentially the BCBF-QP written as a differentiable projection layer. 
Note that for a relative-time mesh of $N+1$ nodes (Remark~\ref{remark: bcbf_mesh}), this QP has at least $N + 2+ 2m$ affine constraints in $u$. The constraint count thus grows linearly with the backup horizon, which precludes a closed-form solution and motivates the usage of a differentiable QP solver (App.~\ref{app: backup_layer}).
We now define our \ours{} policy:
\begin{definition}[\ours\ Policy]\label{def: ps2_policy}
For $\pinom:\X\to\U$, the \ours\ policy is 
\begin{equation}\label{eq: ps2_policy}
    \pi^\phi_{\mathrm{\ours}}(x;\lbpstar)
    :=
    \calP_\BL(\pinom)(x).
\end{equation}
\end{definition}

This reparameterization turns Subproblem~\ref{prob: sub1} into unconstrained optimization over policy parameters: 
$\max_{\phi} J(\pi^\phi_{\mathrm{\ours}}(\cdot;\lbpstar))$.
Because the projection layer is differentiable, gradients propagate end-to-end through~\eqref{eq: backup_layer}. Hence, Phase II is agnostic to the RL backbone, with the only architectural modification being the control-invariant layer appended to the policy output. 
Regardless of the stochasticity of the policy output, the projection is applied deterministically. 
Therefore, \oursRA{} guarantees safety and is expressive, as the following theoretical results show.
\begin{restatable}[Safety Guarantee of \oursRA]{theorem}{safetythm}\label{thm: ps2_safety}
    Suppose $(\B,\pi_\B)$ satisfies Theorem~\ref{thm: base_pair}, let $\lbpstar$ be the fixed composed backup policy, and assume $\X_0\subseteq \CT(\lbpstar)$. Then for any policy $\pinom$, the corresponding \ours\ policy satisfies $\pi^\phi_{\mathrm{\ours}} \in \Pi_{\bcbf}(\lbpstar)$. 
    By Theorem~\ref{thm: decomposition}, $\pi^\phi_{\mathrm{\ours}}\in \Pi_{\safe}$ and is therefore feasible for Problem~\ref{prob: main}. 
    Every realized input satisfies $u_k\in\U$ and the sampled closed-loop trajectory satisfies $x_k\in\calS$ for all $k\ge 0$ and all $x_0\in\X_0$. 
    Under the sampled-data conditions, the continuous-time ZOH trajectory remains in $\calS$ as well. 
\end{restatable}

\begin{restatable}[Universal Approximation of \oursRA]{theorem}{uatthm}\label{thm: ps2_uat}
Fix $\lbpstar$, assume $\U$ is convex, and let $\Omega\subseteq \CT(\lbpstar)$ be compact. 
Let $C(\Omega,\U)$ denote the space of continuous maps $\Omega\to\U$ with norm $\|\pi\|_\infty:=\sup_{x\in\Omega}\|\pi(x)\|_\infty$. 
For function classes $\mathcal G_{\rm NN},\mathcal G\subseteq C(\Omega,\U)$, 
assume $\mathcal G_{\rm NN}$ universally approximates $\mathcal G$, i.e., 
for every $\pi\in\mathcal G$ and $\epsilon>0$, there exists 
$\pinom\in\mathcal G_{\rm NN}$ such that 
$\|\pi-\pinom\|_\infty<\epsilon$. 
Define the class of \ours{} policies obtained by composing~\eqref{eq: backup_layer} with $\mathcal G_{\rm NN}$-class policies as
\begin{equation}
    \Pi_{\ours}^{\mathcal G_{\rm NN}}(\lbpstar;\Omega)
    :=
    \left\{
        \calP_\BL(\pinom):\pinom\in\mathcal G_{\rm NN}
    \right\},
\end{equation}
and the BCBF-feasible policy class on $\Omega$ as
\begin{equation}
    \Pi_{\bcbf}(\lbpstar;\Omega)
    :=
    \left\{
        \pi\in\mathcal G:
        \pi(x)\in\Kb_\bcbf(x;\lbpstar)
        \quad \forall x\in\Omega
    \right\}.
\end{equation}
Then $\Pi_{\ours}^{\mathcal G_{\rm NN}}(\lbpstar;\Omega)$ universally 
approximates $\Pi_{\bcbf}(\lbpstar;\Omega)$ under $\|\cdot\|_\infty$.
\ifisrestate\else\footnote{The same statement holds in $L^p(\Omega,\U)$ by the corresponding $L^p$ version of \citep[Thm.~12]{hardnet}, where $L^p(\Omega,\U)$ denotes $p$-integrable maps with norm $\|\pi\|_p=(\int_\Omega \|\pi(x)\|_p^p dx)^{1/p}$.}\fi

\end{restatable}

Together, Theorems~\ref{thm: ps2_safety} and~\ref{thm: ps2_uat} show that the CIL enforces safety without sacrificing expressivity. 
Theorem~\ref{thm: ps2_safety} guarantees that every projected policy is BCBF-feasible and hence safe, even when CIL is used as a filter without Phase~II training, while Theorem~\ref{thm: ps2_uat} shows that the projection preserves the approximation power of the chosen policy architecture over the BCBF-feasible class. 
Thus, \oursRA{} combines hard safety guarantees with the policy expressivity needed for high-performance RL.

\section{Experiments}\label{sec: experiments}
We evaluate \oursRA{} on two tracking tasks whose reference trajectories intentionally leave the safe set. The first is a low-dimensional unicycle lane-keeping task, where dense grid evaluation lets us visualize how the learned safe-arrival policy enlarges the backup-induced invariant set. The second is an agile $10$-dimensional quadrotor powerloop task with body rate and thrust as control, which tests scalability under tight actuation and a high-relative-degree altitude constraint. In both cases, the reward is a negative weighted tracking cost. 
Full details and extended analyses are provided in App.~\ref{app: unicycle_experiment_details} and App.~\ref{app: unicycle_extended_analysis} for the unicycle environment, and App.~\ref{app: quadrotor_experiment_details} and App.~\ref{app: quadrotor_extended_analysis} for the quadrotor.

\paragraph{Baselines and protocol.}
We compare \textbf{\textcolor{oursRA}{\oursSAV{}}} (\ours+$\piarrL$) against seven baselines: (\texttt{Pen}) low/high-penalty SAC; (\texttt{Con}) CPO and SAC-Lag.~\citep{cpo,ha2021sacLag,yang2021wcsac}; and (\texttt{PSRL}) CBF-RL~\citep{cbfRL}, MPS~\citep{bastani2021mps}, and \textbf{\textcolor{oursABP}{\oursABP{}}}, which replaces the learned safe-arrival policy with an analytic one. 
CBF-RL uses a sum-of-squares CBF for unicycle and an HOCBF for quadrotor, with the filter retained for both training and evaluation. 
We give MPS a longer recovery horizon than \oursRA. 
Tables~\ref{tab: hero_unicycle}--\ref{tab: hero_quadrotor} report sampled-data metrics under each method's controller. 
All methods are trained with $10$ seeds and evaluated on $1{,}000$ episodes per seed, and we report the IQM with $95\%$ CIs and the aggregated safety. 
See App.~\ref{app: ps2_implementation_details} for \oursRA{} implementation details and App.~\ref{app: baseline_details} for the baselines.

\subsection{Unicycle: Lane-Keeping}\label{sec: unicycle_experiment}
The unicycle state is $x=[y,v,\psi]^\top$, 
and inputs are $a_\cmdtext\in[-5,5]\,\mathrm{m/s^2}$ and $r_\cmdtext\in[-1,1]$ rad/s. The reference is a 20 sec sinusoid with amplitude $2.5$ m and desired speed $v_{\mathrm{des}}=5$ m/s, while the safe set is $\calS=\{x: |y|\le1.8\text{ m},\;|\psi|\le\pi/3\text{ rad}\}$. The base set is an LQR-induced ellipsoid about $x^\star=[0,v_{\mathrm{des}},0]^\top$. 
Dense safe-arrival evaluation on a $201\times121\times201$ grid shows that the learned $\piarrL$ enlarges the safe-arrival measure $\mu_{\mathrm{SA}}$ from $0.23$ for the analytic policy to $0.33$.
Figure~\ref{fig: hero_unicycle}(a,b) visualizes this larger induced set and (c) the resulting less-conservative tracking behavior.

\begin{figure}[h!]
    \centering
    \includegraphics[width=1.0\linewidth]{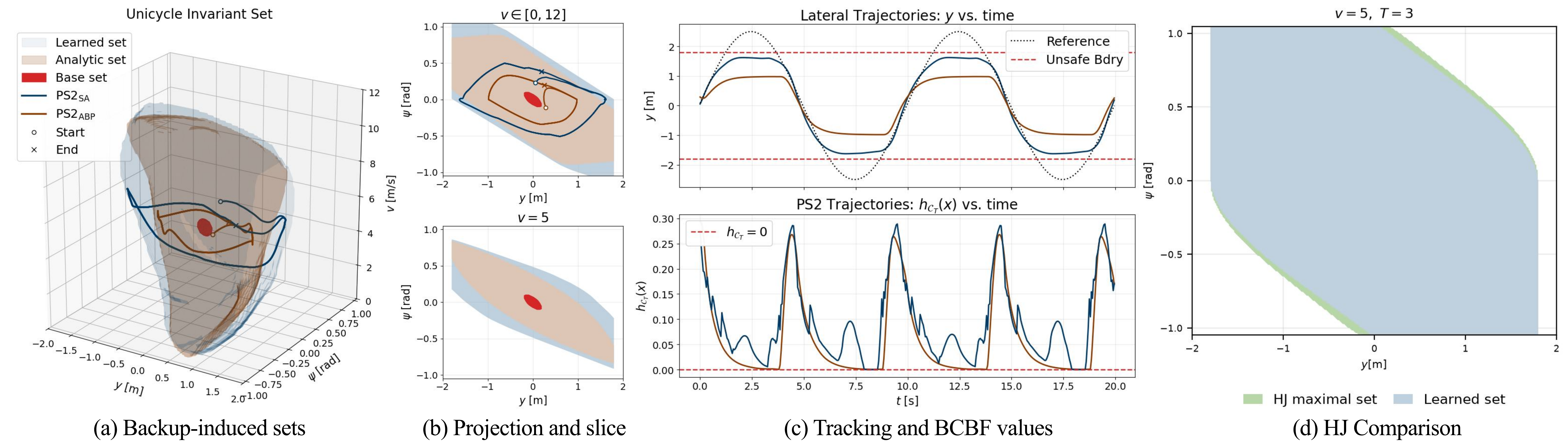}
    \caption{Unicycle recoverability and lane-keeping. 
    (a) \textcolor{oursRA}{Learned} and \textcolor{oursABP}{analytic} backup-induced sets, the base set, and trajectories of \textbf{\textcolor{oursRA}{\oursSAV}} and \textbf{\textcolor{oursABP}{\oursABP}}. 
    (b) Projection onto the $(y,\psi)$ plane over $v\in[0,12]$ m/s (top) and slice at $v=5$ m/s (bottom). 
    (c) Lateral tracking (top) and corresponding backup-induced barrier values (bottom). 
    (d) \textcolor{oursRA}{Learned backup-induced set} versus the \textcolor{green!25!black}{numerical HJ approximation} of the maximal control-invariant set, shown at $v=5$ m/s with a longer backup horizon.
    }
    \label{fig: hero_unicycle}
\end{figure}

Table~\ref{tab: hero_unicycle} shows that \textbf{\textcolor{oursRA}{\oursSAV{}}} achieves $100\%$ safety with the best lateral and heading tracking among methods that are also $100\%$ safe. Compared with \textbf{\textcolor{oursABP}{\oursABP{}}}, the learned policy reduces lateral RMSE from $0.97$ m to $0.53$ m, showing that Phase~I directly reduces downstream conservatism.
The higher velocity RMSE of \textbf{\textcolor{oursRA}{\oursSAV{}}} is consistent with the task objective, which prioritizes lateral tracking over speed tracking through a larger reward weight. 
\textbf{\textcolor{oursRA}{\oursSAV{}}} therefore trades speed accuracy for much lower lateral and heading errors. 
SAC-Pen$_{\text{high}}$ is also safe, but is substantially more conservative, while the low-penalty and CMDP baselines achieve tracking only by allowing frequent violations. 
With its runtime filter active, CBF-RL achieves $100\%$ safety, but exhibits substantially worse tracking than \textbf{\textcolor{oursRA}{\oursSAV{}}}, while MPS still incurs occasional violations under the evaluated sampled-data deployment.

\begin{table}[htpb]
    \centering
    \caption{Unicycle lane-keeping experiment results across models trained with 10 different seeds.
    }
    \label{tab: hero_unicycle}
    \resizebox{1.0\linewidth}{!}{%
    \begin{tabular}{cl ccc ccc}
        \toprule
        & & \multicolumn{3}{c}{\textbf{Tracking Performance (RMSE)}} & \multicolumn{3}{c}{\textbf{Safety Performance}}\\
        \cmidrule(lr){3-5} \cmidrule(lr){6-8}
        & \textbf{Safe RL} & \makecell{$y$ (m) ($\downarrow$)\\IQM [95\% CI]} & \makecell{$v$ (m/s) ($\downarrow$)\\IQM [95\% CI]} & \makecell{$\psi$ (rad) ($\downarrow$)\\IQM [95\% CI]} & \makecell{Total Safety ($\uparrow$)\\Safe Ep/Total Ep} & \makecell{Per-seed Safety ($\uparrow$)\\IQM [95\% CI]} & \makecell{Worst Viol. ($\downarrow$)\\$y$ (m)}\\
        \midrule
        \multirow{2}{*}{\rotatebox[origin=c]{90}{\texttt{Pen}}}
        & SAC-Pen$_{\text{low}}$ & 0.53 [0.34, 0.96] & 2.14 [1.20, 3.11] & 0.19 [0.09, 0.38] & 1.8\% & 0.0\% [0.0\%, 3.0\%] & 8.12 \\
        & SAC-Pen$_{\text{high}}$ & 1.11 [0.99, 1.25] & 1.85 [1.26, 2.84] & 0.15 [0.13, 0.19] & \textbf{100\%} & \textbf{100\% [100\%, 100\%]} & \textbf{0.00} \\
        \cmidrule(lr){2-8}
        \multirow{2}{*}{\rotatebox[origin=c]{90}{\texttt{Con}}}
        & SAC-Lag. & 0.52 [0.40, 0.66] & 1.22 [0.99, 1.54] & 0.12 [0.09, 0.14] & 30.0\% & 16.7\% [0.0\%, 66.7\%] & 1.03\\
        & CPO & 2.51 [1.66, 7.29] & 4.44 [3.05, 5.36] & 0.27 [0.20, 0.57] & 64.2\% & 73.7\% [32.5\%, 99.6\%] & 76.94 \\
        \cmidrule(lr){2-8}
        \multirow{4}{*}{\rotatebox[origin=c]{90}{\texttt{PSRL}}}
        & CBF-RL & 1.75 [1.72, 1.79] & 3.71 [2.73, 4.44] & 0.23 [0.21, 0.25] & \textbf{100\%} & \textbf{100\% [100\%, 100\%]} & \textbf{0.00} \\
        & MPS & 0.68 [0.55, 0.78] & 0.48 [0.43, 0.53] & 0.17 [0.14, 0.19] & 97.8\% & 97.8\% [97.5\%, 98.2\%] & 0.67 \\
        & \textcolor{oursABP}{\oursABP} & 0.97 [0.97, 0.98] & \textbf{0.43 [0.39, 0.56]} & 0.15 [0.15, 0.15] & \textbf{100\%} & \textbf{100\% [100\%, 100\%]} & \textbf{0.00} \\
        & \textbf{\textcolor{oursRA}{\oursSAV}} & \textbf{0.53 [0.52, 0.54]} & 0.80 [0.77, 0.84] & \textbf{0.11 [0.10, 0.11]} & \textbf{100\%} & \textbf{100\% [100\%, 100\%]} & \textbf{0.00} \\
        \bottomrule
    \end{tabular}
    }
\end{table}

\paragraph{Effect of Phase~I on invariant-set coverage.}
App.~\ref{app: hj_comparison} evaluates how much of the maximal control-invariant set Phase~I can recover with a longer backup horizon for the unicycle.
In this separate set-coverage study, the learned set covers $98.18\%$ of the maximal control-invariant set computed via Hamilton--Jacobi (HJ) reachability on the evaluated cross-section (Fig.~\ref{fig: hero_unicycle}(d)), demonstrating near-complete recovery.
This shows the efficacy of Phase~I in enlarging the implicit invariant set.

\subsection{Quadrotor: Powerloop Maneuver}\label{sec: quadrotor_experiment}
The quadrotor state is $x=[\mathbf p^\top,\mathbf v^\top,\mathbf q^\top]^\top\in\R^{10}$, 
and the input is thrust $\acmd\in[0,4g]$ and body rates $\wcmd=[\omega_x,\omega_y,\omega_z]^\top\in[-18,18]^3\,\mathrm{rad/s}$. The reference is a powerloop inspired by~\citep{deepdroneacrobatics}: starting at $4.5$ m/s, the quadrotor tracks a vertical loop of radius $1.5$ m centered at $[0,0,2]^\top$ m while executing a $360^\circ$ flip. The safe set is $\calS=\{x:p_z\le 3 \text{ m}\}$. 
Although the safety specification is a single altitude bound, the constraint is enforced through the full $10$-dimensional dynamics and all actuations.

\begin{figure}[ht!]
    \centering
    \includegraphics[width=1.0\linewidth]{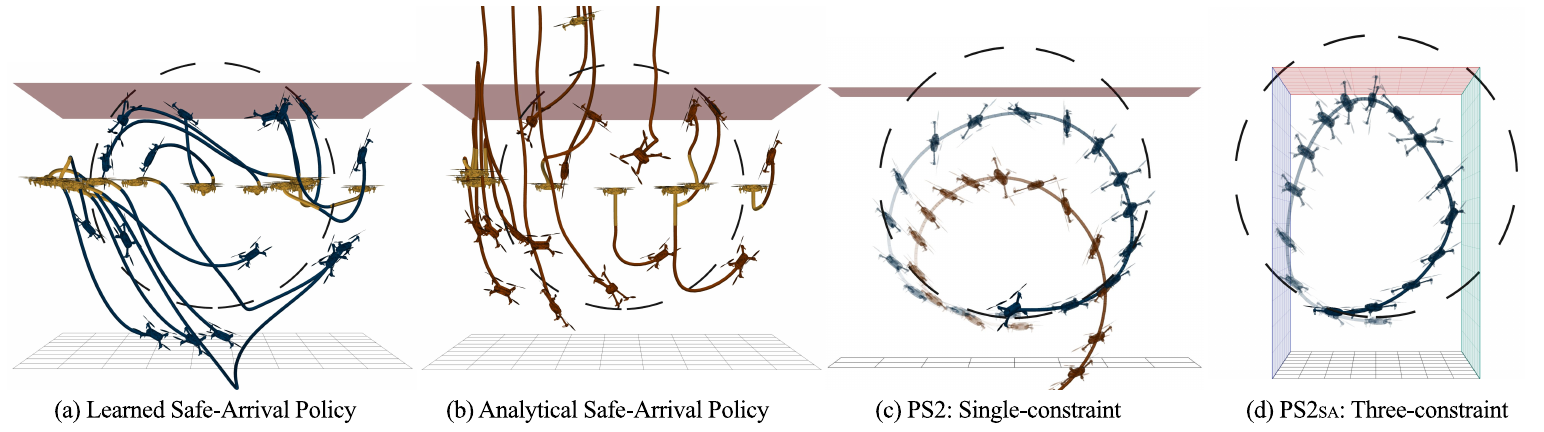}
    \caption{Quadrotor recovery and constrained tracking.
    (a,b) Rollouts under the \textbf{\textcolor{oursRA}{learned $\piarrL$}} and \textbf{\textcolor{oursABP}{analytic $\piarr$}} policies, respectively; the \textcolor{lqrYellow}{base set trajectory segments} are controlled by the \textbf{\textcolor{lqrYellow}{base controller $\pi_\B$}}.
    (c) Powerloop tracking with ceiling constraint $p_z\le 3\text{ m}$, comparing \textbf{\textcolor{oursRA}{\oursSAV}} and \textbf{\textcolor{oursABP}{\oursABP}}.
    (d) \textbf{\textcolor{oursRA}{\oursSAV}} tracking with three simultaneous constraints: $p_z\le 3\text{ m}$ and $-1 \text{ m} \le p_x \le 1\text{ m}$.
    In (c,d), dashed curves show the reference trajectory; shaded surfaces indicate constraint boundaries.}
    \label{fig: hero_quadrotor}
\end{figure}

Table~\ref{tab: hero_quadrotor} shows that \textbf{\textcolor{oursRA}{\oursSAV{}}} is the best-performing method among all policies with $100\%$ safety. 
Relative to \textbf{\textcolor{oursABP}{\oursABP{}}}, it reduces all RMSEs by approximately $46$--$56\%$. 
Compared to MPS, it reduces $\vb p$-RMSE by about $70\%$ despite MPS being given a longer recovery horizon. 
CBF-RL achieves $100\%$ safety with its runtime filter, but has a substantially higher tracking error. 
Because $\omega_{\cmdtext}$ affects the ceiling constraint at a higher derivative order than $a_{\cmdtext}$, its HOCBF condition permits thrust only, without exploiting body-rate control.
App.~\ref{app: safe_rl_baseline_details} provides the rationale for using HOCBF in CBF-RL.
The penalty-based and CMDP baselines still exhibit safety violations.
The learned $\piarrL$ is especially important in this agile setting. 
Fig.~\ref{fig: hero_quadrotor}(a,b) show that $\piarrL$ recovers from task-relevant near-ceiling states much better than the analytic policy.
Quantitatively, the learned safe-arrival policy increases task-relevant recoverability from $60\%$ to $84\%$ over the analytic backup, with the largest gains near the ceiling ($35.5\%\!\to\!69.3\%$). 
This larger near-ceiling safe-arrival set accounts for the downstream tracking performance gap, as reported in Table~\ref{tab: hero_quadrotor} and demonstrated in Fig.~\ref{fig: hero_quadrotor}(c).

\begin{table}[htpb!]
    \centering
    \caption{Quadrotor powerloop experiment results across models trained with 10 different seeds.
    }
    \label{tab: hero_quadrotor}
    \resizebox{1.0\linewidth}{!}{%
    \begin{tabular}{clccc ccc}
        \toprule
        & & \multicolumn{3}{c}{\textbf{Tracking Performance (RMSE)}} & \multicolumn{3}{c}{\textbf{Safety Performance}}\\
        \cmidrule(lr){3-5} \cmidrule(lr){6-8}
        & \textbf{Safe RL} & \makecell{$\mathbf{p}$ (m) ($\downarrow$)\\IQM [95\% CI]} & \makecell{$\mathbf{v}$ (m/s) ($\downarrow$)\\IQM [95\% CI]} & \makecell{$\theta_\mathbf{q}$ (rad) ($\downarrow$)\\IQM [95\% CI]}  & \makecell{Total Safety ($\uparrow$)\\Safe Ep/Total Ep} & \makecell{Per-seed Safety ($\uparrow$)\\IQM [95\% CI]} & \makecell{Worst Viol. ($\downarrow$)\\$z$ (m)}\\
        \midrule
        \multirow{2}{*}{\rotatebox[origin=c]{90}{\texttt{Pen}}}
        & SAC-Pen$_{\text{low}}$ & 1.00 [0.72, 1.36] & 1.88 [1.51, 2.36] & 1.15 [0.74, 1.59] & 47.3\% & 45.5\% [7.2\%, 89.7\%] & 0.62  \\
        & SAC-Pen$_{\text{high}}$ & 1.14 [1.02, 1.48] & 2.24 [1.84, 2.63] & 1.54 [1.30, 1.79] & 99.9\% & 100\% [99.9\%, 100\%] & 0.004 \\
        \cmidrule(lr){2-8}
        \multirow{2}{*}{\rotatebox[origin=c]{90}{\texttt{Con}}}
        & SAC-Lag. & 1.94 [1.70, 2.53] & 4.30 [3.97, 4.58] & 2.12 [1.77, 2.26] & 58.6\% & 64.3\% [19.0\%, 100\%] & 1.79 \\
        & CPO & 10.62 [7.08, 13.17] & 11.55 [9.28, 13.94] & 2.17 [2.04, 2.36] & 70.0\% & 83.3\% [33.3\%, 100\%] & 2.41 \\
        \cmidrule(lr){2-8}
        \multirow{4}{*}{\rotatebox[origin=c]{90}{\texttt{PSRL}}}
        & CBF-RL & 3.71 [2.78, 4.98] & 5.27 [4.49, 6.35] & 2.02 [1.78, 2.25] & \textbf{100\%} & \textbf{100\% [100\%, 100\%]} & \textbf{0.00} \\
        & MPS & 1.99 [1.95, 2.05] & 4.07 [3.89, 4.21] & 1.38 [1.31, 1.51] & \textbf{100\%} & \textbf{100\% [100\%, 100\%]} & \textbf{0.00} \\
        & \textcolor{oursABP}{\oursABP} & 1.34 [1.25, 1.44] & 1.69 [1.61, 1.78] & 0.78 [0.71, 0.84] & \textbf{100\%} & \textbf{100\% [100\%, 100\%]} & \textbf{0.00} \\
        & \textbf{\textcolor{oursRA}{\oursSAV}} & \textbf{0.59 [0.54, 0.67]} & \textbf{0.78 [0.67, 1.05]} & \textbf{0.42 [0.33, 0.49]} & \textbf{100\%} & \textbf{100\% [100\%, 100\%]} & \textbf{0.00} \\
        \bottomrule
    \end{tabular}
    }
\end{table}

\paragraph{Multiple safety constraints.}
To evaluate scaling with the number of safety constraints, we augment the ceiling constraint with two lateral walls, yielding $p_z\le 3$ and $-1 \le p_x\le 1$ (Fig.~\ref{fig: hero_quadrotor}(d)).
The base pair is redesigned, and the policies are retrained with expanded Phase~I sampling, as detailed in App.~\ref{app: quadrotor_three_constraints}. 
Across $10$ seeds and $1{,}000$ episodes per seed, \textbf{\textcolor{oursRA}{\oursSAV{}}} achieves $100\%$ safety, with position, velocity, and attitude RMSE IQMs of $0.74$ m, $1.72$ m/s, and $1.14$ rad, respectively.

\paragraph{Effect of Phase~II on safety and performance.}
App.~\ref{app: backup-layer-ablation} isolates the CIL's impact on safety and performance in the quadrotor task. 
Removing it after Phase~II drops safety to $0\%$, while adding it only as a post-hoc filter achieves $100\%$ safety but yields $41$--$45\%$ worse tracking.
This demonstrates that the CIL enforces safety independently of Phase~II training (Theorem~\ref{thm: ps2_safety}), while end-to-end Phase~II training with CIL substantially improves performance within the certified policy class (Theorem~\ref{thm: ps2_uat}).

\subsection{Computational Cost and Scalability Analysis}\label{subsec: scalability}

\paragraph{Computational cost.}
App.~\ref{app: computation_time} shows that the computational overhead is modest. 
On a single NVIDIA V100 GPU, Phase~I takes 14 minutes for unicycle and 40/43 minutes for the single-/three-constraint quadrotor. 
Phase~II takes 4.35 hours for the unicycle and 12.75/14.38 hours for the two quadrotor settings. 
Corresponding inference takes 0.34 and 0.62/0.82 ms per decision, respectively.
Phase~II training costs concern task performance rather than the establishment of safety: under the assumptions of Theorem~\ref{thm: ps2_safety}, the fixed backup and CIL enforce safety for any nominal task policy, independently of convergence, as shown by the ablation study in App.~\ref{app: backup-layer-ablation}.

\paragraph{Scalability analysis.}
PS2-RL avoids explicit global invariant-set computation: only the local base pair requires explicit certification, and the fixed backup extends it through rollouts. 
The computational overhead mainly comes from the CIL.
With $N+1$ backup rollout nodes, $n_S$ safety inequalities, $n_B$ base inequalities, and $m$ controls, the CIL has $m$ variables and $M=(N+1)n_S+n_B+2m$ affine inequalities (App.~\ref{app: backup_layer_qp}). 
For a condensed primal-dual solver, Newton-system assembly and factorization cost $O(I_{\mathrm{QP}}(Mm^2+m^3))$ over $I_{\mathrm{QP}}$ iterations~\citep{qpax}, linear in $M$ for fixed $m$ and $I_{\mathrm{QP}}$.
State dimension $n$ instead affects rollout and sensitivity construction, where dense full-sensitivity propagation costs $O(Nn^3)$. 
This avoids the exponential state-grid growth of conventional methods~\citep{bansal2017hamilton,jarvis2003some,clark2022semi}. 
Actual runtime also depends on conditioning and implementation.
Furthermore, adapting the directional-sensitivity construction of \citep{wijk2026closedFormBCBF} to propagate only $\Psi[f,g]$ would reduce the cost of sensitivity construction to $O(Nn^2(m+1))$, which is left as future work. 

Our three-constraint quadrotor experiment increases the number of inequalities from $110$ to $312$ while retaining four variables.
Despite this nearly threefold increase, Phase~II training and inference costs rise by only $12.8\%$ and $31.9\%$, respectively, with unchanged training budgets.
These results demonstrate practical scalability of \oursRA{} with modest computational overhead.

Finally, although our work uses an analytic control-affine model, \oursRA{} does not fundamentally rely on a hand-derived symbolic decomposition $f(x)+g(x)u$.
More generally, given a differentiable simulator, backup rollouts and the sensitivities needed by the CIL can be computed by automatic differentiation for only the safety-relevant outputs. 
This future work is further discussed in App.~\ref{app: ps2_scalability_argument}.

\section{Conclusion and Limitations}\label{sec: conclusion}
In this work, we present \oursRA{}, a two-phase framework for safe RL that learns a backup policy with the \ralow{} function and then trains the task policy end-to-end through a differentiable control-invariant layer. 
By replacing explicit synthesis of control-invariant sets with learned backup rollouts and a safe control-space projection, \oursRA{} is a scalable RL framework that attains hard safety guarantees while retaining strong task performance. 
Experiments on a unicycle and a $10$-dimensional quadrotor powerloop show $100\%$ safety for \oursRA{} and substantially lower tracking error than the baselines, demonstrating its scalability, performance, and safety guarantees.

\paragraph{Impact.}
\oursRA{} provides formal safety guarantees for RL while maintaining strong task performance, helping to close the gap between provably safe methods and the performance levels needed for practical adoption of RL-based controllers in safety-critical settings, such as autonomous vehicles and robotics. 
By making provable safety compatible with high task performance, our framework can broaden the set of applications in which RL-based controllers can be reliably deployed.

\paragraph{Limitations.}
The formal guarantees of \oursRA{} require a deterministic control-affine analytic dynamics model or a differentiable simulator. 
These assumptions are reasonable for provable safety: without prior knowledge of the system dynamics, formal guarantees for a policy are generally impossible.
Extending \oursRA{} to uncertain dynamics and perception-based constraints, as well as evaluating it on hardware, are important directions for future work.

\newpage
{\small
\bibliographystyle{plainnat}
\bibliography{main}
}

\newpage
\appendix

\etocdepthtag.toc{appendix}
{\LARGE\bf Appendix}
\vspace{1em}

\etocsettagdepth{}{none}              
\etocsettagdepth{appendix}{subsection}

\etocsettocstyle{}{}

\tableofcontents
\vspace{1.5em}
\hrule
\vspace{1.5em}

\newpage
\section{Notations and Symbols}\label{app: notations}

For convenience, we summarize the notation used in the paper's main text.
Appendix-only proof variables, algorithmic implementation variables, and experimental-detail symbols that do not appear in the main text are omitted.
Tables~\ref{tab: notation-sets}--\ref{tab: notation-bcbf-projection} group the remaining symbols by category, with the right column indicating where each symbol is first introduced.

\begin{table}[htpb]
    \centering
    \caption{Sets, spaces, and policy classes used in the main text.}
    \label{tab: notation-sets}
    \small
    \begin{tabular}{@{}p{0.20\linewidth}p{0.58\linewidth}p{0.16\linewidth}@{}}
        \toprule
        Symbol & Description & First used \\
        \midrule
        $\R$, $\R^n$, $\R^m$ & Real numbers; $n$-dim. real space; $m$-dim. real space & Sec.~\ref{subsec: safe_control} \\
        $n$, $m$ & State and control-input dimensions & Eq.~\eqref{eq: affine_dyn} \\
        $\X\subseteq\R^n$ & State space of interest & Eq.~\eqref{eq: affine_dyn} \\
        $\U\subseteq\R^m$ & Compact admissible input set encoding actuator limits & Eq.~\eqref{eq: affine_dyn} \\
        $\calS$ & Safe set, $\calS:=\{x\in\X:h_\calS(x)\ge0\}$ & Sec.~\ref{subsec: safe_control} \\
        $\B$ & Certified base set, $\B:=\{x:h_\B(x)\ge0\}\subseteq\calS$ & Sec.~\ref{subsec: safe_control} \\
        $\CT(\bp)$ & $T$-time backup-induced control-invariant set under policy $\bp$ & Def.~2.1 \\
        $\mathcal R(\B,\calS,f_{\pi_b},T)$ & States that reach $\B$ within horizon $T$ under $\bp$ while remaining in $\calS$ & Def.~2.1 \\
        $\F$ & Failure set, $\F:=\X\setminus\calS$ & Sec.~\ref{subsec: sav_func} \\
        $\Omega$, $\mu$ & Compact design region and finite reference measure for safe-arrival set maximization & Subprob.~\ref{prob: sub2} \\
        $\X_0$, $\rho_0$ & Initial-state support and initial-state distribution & Prob.~\ref{prob: main} \\
        $C(\Omega,\U)$, $\|\cdot\|_\infty$ & Space of continuous maps $\Omega\to\U$ endowed with the sup norm $\|\pi\|_\infty:=\sup_{x\in\Omega}\|\pi(x)\|_\infty$ & Thm.~\ref{thm: ps2_uat} \\
        $\mathcal G_{\rm NN}$, $\mathcal G$ & Chosen pre-CIL policy architecture class and target policy class, with $\mathcal G_{\rm NN}$ assumed to universally approximate $\mathcal G$ & Thm.~\ref{thm: ps2_uat} \\
        $\Pi_{\ours}^{\mathcal G_{\rm NN}}(\lbpstar;\Omega)$ & Class of \ours{} policies obtained by composing the CIL with policies in $\mathcal G_{\rm NN}$ & Thm.~\ref{thm: ps2_uat} \\
        $\Pi_{\safe}$ & Feasible policy class for Problem~\ref{prob: main}, enforcing $x_k\in\calS$ and $u_k\in\U$ pointwise & Problem~\ref{prob: main} \\
        $\Pi_{\bcbf}(\bp)$ & BCBF-feasible policy class certified by a fixed backup policy $\bp$ & Sec.~\ref{sec: decompose_guarantee} \\
        $\Pi_{\bcbf}(\lbpstar;\Omega)$ & BCBF-feasible target policy subclass restricted to compact $\Omega$  & Thm.~\ref{thm: ps2_uat} \\
        \bottomrule
    \end{tabular}
\end{table}

\begin{table}[htpb]
    \centering
    \caption{Dynamics, time, and local-base-set quantities used in the main text.}
    \label{tab: notation-dynamics}
    \small
    \begin{tabular}{@{}p{0.20\linewidth}p{0.58\linewidth}p{0.16\linewidth}@{}}
        \toprule
        Symbol & Description & First used \\
        \midrule
        $x$, $u$ & Continuous-time state and control input; also state/action in the sampled MDP & Eq.~\eqref{eq: affine_dyn} \\
        $\dot x$ & Time derivative of the state & Eq.~\eqref{eq: affine_dyn} \\
        $f:\R^n\to\R^n$ & Drift of the control-affine system & Eq.~\eqref{eq: affine_dyn} \\
        $g:\R^n\to\R^{n\times m}$ & Control vector field of the control-affine system & Eq.~\eqref{eq: affine_dyn} \\
        $f_\pi(x)$ & Closed-loop vector field under policy $\pi$, $f_\pi(x):=f(x)+g(x)\pi(x)$ & Sec.~\ref{subsec: safe_control} \\
        $\Phi_\pi(x_0,t)$ & Closed-loop flow map under policy $\pi$ from initial state $x_0$ & Sec.~\ref{subsec: safe_control} \\
        $\flowB(x,t)$ & Backup flow map under $\bp$ & Def.~\ref{subsec: safe_control} \\
        $F$ & Deterministic sampled/ZOH transition map, $x_{k+1}=F(x_k,u_k)$ & Sec.~\ref{sec: rl-prelim} \\
        $t$, $t'$ & Continuous time and relative backup-rollout time & Sec.~\ref{subsec: safe_control},~\ref{sec: sub1} \\
        $k$, $\tau$ & Discrete timestep and dummy product/summation index & Sec.~\ref{sec: rl-prelim}, Eq.~\eqref{eq: save_q_func} \\
        $t_k$, $\dt$ & Sampling instant $t_k=k\dt$ for a constant sampling period $\dt$ & Rem.~\ref{remark: ct_dt} \\
        $T$, $N$ & Backup horizon and backup-mesh size; for a mesh $0=t'_0<\cdots<t'_N=T$, there are $N+1$ rollout nodes and $T=N\dt$ & Def.~2.1, Sec.~\ref{sec: phase1} \\
        \bottomrule
    \end{tabular}
\end{table}

\begin{table}[htpb]
    \centering
    \caption{Policies, objectives, and value functions used in the main text.}
    \label{tab: notation-policies-values}
    \small
    \begin{tabular}{@{}p{0.20\linewidth}p{0.58\linewidth}p{0.16\linewidth}@{}}
        \toprule
        Symbol & Description & First used \\
        \midrule
        $\pi:\X\to\U$ & Generic state-feedback policy & Sec.~\ref{subsec: safe_control} \\
        $\pinom$ & Neural RL policy parameterized by $\phi$ & Sec.~\ref{sec: rl-prelim} \\
        $\bp$ & Backup policy used to induce $\CT(\bp)$ & Sec.~\ref{subsec: safe_control} \\
        $\pi_\B$ & Base controller that renders $\B$ forward invariant & Sec.~\ref{sec: sub2} \\
        $\piarr$ & Safe-arrival policy used outside $\B$ in the composed backup policy & Eq.~\eqref{eq: bp_decompose} \\
        $\piarrL$ & Learned safe-arrival policy with parameters $\theta$ & Subprob.~\ref{prob: sub2} \\
        $\lbp$ & Composed learned backup policy induced by $\pi_\B$ and $\piarrL$ & Subprob.~\ref{prob: sub2} \\
        $\piarrstar$, $\lbpstar$ & Trained safe-arrival policy and fixed composed backup policy used in Phase~II & Sec.~\ref{sec: main_method} \\
        $\pi^\phi_{\mathrm{\ours}}(x;\lbpstar)$ & \ours{} policy obtained by projecting $\pinom$ through the control-invariant layer & Def.~\ref{def: ps2_policy} \\
        $r(x,u)$ & MDP reward function for task performance & Sec.~\ref{sec: rl-prelim} \\
        $J(\pi)$, $J(\pinom)$ & Expected discounted return objective & Sec.~\ref{sec: rl-prelim}, Prob.~\ref{prob: main} \\
        $\gamma\in(0,1)$ & Discount factor for the task RL objective & Sec.~\ref{sec: rl-prelim} \\
        $Q_\gamma^\pinom(x,u)$ & State-action value function for the task RL objective under $\pinom$ & Sec.~\ref{sec: rl-prelim} \\
        $\beta\in(0,1)$ & Discount factor for the safe-arrival value objective & Def.~\ref{def: save_arrival_value} \\
        $Q_{\mathrm{SA},\beta}^{\piarr}$, $V_{\mathrm{SA},\beta}^{\piarr}$ & Safe-arrival Q- and state-value functions under $\piarr$ & Eq.~\eqref{eq: save_q_func} \\
        $J_{\mathrm{SA},\beta}(\theta)$ & Safe-arrival objective optimized in Phase~I & Eq.~\eqref{eq: save_phase1_surrogate} \\
        $\rho_{\mathrm{arr}}$ & Design distribution for Phase~I safe-arrival training & Sec.~\ref{sec: phase1} \\
        $\mu_{\mathrm{SA}}(\piarrL;\Omega,\B)$ & Measure objective for safe-arrival set maximization & Eq.~\eqref{eq: backup_recoverable_set_obj} \\
        \bottomrule
    \end{tabular}
\end{table}

\begin{table}[htpb]
    \centering
    \caption{Barrier, indicator, and projection quantities used in the main text.}
    \label{tab: notation-bcbf-projection}
    \small
    \begin{tabular}{@{}p{0.20\linewidth}p{0.58\linewidth}p{0.16\linewidth}@{}}
        \toprule
        Symbol & Description & First used \\
        \midrule
        $h_\calS$, $h_\B$ & Differentiable set functions defining $\calS$ and $\B$ & Sec.~\ref{subsec: safe_control} \\
        $\nabla h(x)$ & Gradient of a differentiable barrier/certificate function & Sec.~\ref{subsec: safe_control} \\
        $\alpha$, $\alpha_\calS$, $\alpha_\B$ & Extended class-$\K_\infty$ functions used in CBF/BCBF inequalities & Sec.~\ref{subsec: safe_control}, Eq.~\eqref{eq: bcbf_set} \\
        $h_\CT$ & Implicit BCBF defining the backup-induced set $\CT$ & Eq.~\eqref{eq: bcbf_implicit_set} \\
        $\D_{t'}^\calS(x,u)$, $\D_T^\B(x,u)$ & BCBF derivative terms for the safe-set rollout and terminal base-set condition & Sec.~\ref{sec: sub1}\\
        $\Kb(x)$ & Generic state-dependent convex constraint set in a projection layer & Def.~\ref{def: safety_proj_layer} \\
        $\Kb_{\bcbf}(x;\bp)$ & BCBF-admissible input set for a fixed backup policy $\bp$ & Eq.~\eqref{eq: bcbf_set} \\
        $\calP$ & Constraint projection layer based on \hardCVX & Def.~\ref{def: safety_proj_layer} \\
        $\calP_\BL$ & Control-invariant layer, i.e., the BCBF-QP projection onto $\Kb_{\bcbf}(x;\lbpstar)$ used by \oursRA & Def.~\ref{def: backup_layer} \\
        $\|\cdot\|_2$ & Euclidean norm used in projection objectives & Def.~\ref{def: safety_proj_layer} \\
        $\vb1_\B$, $\vb1_\F$ & Indicators of the base set $\B$ and failure set $\F$ & Sec.~\ref{subsec: sav_func} \\
        $\vb b(x)$, $\vb f(x)$, $\vb c(x)$ & Indicators of $\B$, $\F$, and the continuation set, with $\vb c=1-\vb b-\vb f$ & Sec.~\ref{subsec: sav_func} \\
        $n_\calS$, $n_\B$ & Numbers of scalar safe-set rollout constraints and terminal base-set constraints used in the finite-mesh CIL construction & Sec.~\ref{subsec: scalability} \\
        \bottomrule
    \end{tabular}
\end{table}

\newpage
\section{Related Work}\label{app: related_work}

In this section, we review prior work related to our proposed framework. We organize existing methods into five categories.

\paragraph{Safe RL via Constrained Policy Optimization.}
This line of work imposes constraints from learned value functions through constrained policy optimization, such as Lagrangian methods. 
A widely adopted formulation is the constrained Markov decision process (CMDP) framework~\citep{cmdp}, where an auxiliary cost function is introduced and the expected trajectory cost is constrained to remain below a predefined threshold~\citep{cpo, chow2017risk, tessler2018reward, ha2021sacLag, brunke2021safeRL}. 
More recently, state-wise safe RL methods \citep{zhao2023statewisesafereinforcementlearning} have gained attention: instead of enforcing safety in expectation through a cost value function, they learn a neural certificate function and impose it as a constraint to encourage state-wise safety \citep{ma2022joint, yu2022reachability, li2024safe}. However, neither class of methods can guarantee the safety of the learned policy. 
First, the constraints are provided by learned neural-network values, which are not provably reliable in general. 
Second, constrained policy optimization itself does not guarantee that the outputs of the policy network will satisfy the imposed constraints. 
Although these methods scale to expressive policies and high-dimensional systems, and can achieve good empirical performance in some cases, catastrophic failures may still occur during deployment due to the lack of formal safety guarantees. 
This gap motivates the study of provably safe RL.

\paragraph{Safe RL via Verified Certificates.}
A second line of work attaches a formal safety certificate to the policy. 
Unlike neural certificate functions learned purely from loss functions, formal safety guarantees require the synthesis and verification of a valid certificate. 
Common certificate functions include control barrier function (CBF)~\citep{ames2014control, ames2017cbf, ames2019control}, safety index (SI)~\citep{liu2014control, wei2019safe}, and reachability value function \citep{bansal2017hamilton, mitchell2005HJ}. 
Although they take different forms, these certificates can be viewed as functional representations of control-invariant sets, characterizing states from which persistent safety can be maintained under an admissible controller. 
In safe RL, such certificates are typically used as safety filters or shields that modify unsafe actions proposed by a learned policy ~\citep{krasowski2023provablysafeRL, cbfRL, tonkens2022cbfWithHJ, alshiekh2018shielding, zhao2025implicitsafeset, bastani2021mps}.

The main challenge is that obtaining a valid certificate amounts to synthesizing or verifying a control-invariant set, which poses significant scalability challenges. 
Hamilton--Jacobi (HJ) reachability analysis~\citep{bansal2017hamilton, mitchell2005HJ} provides a formal mathematical framework for computing maximal control-invariant sets for nonlinear systems. 
However, it requires solving a nonlinear partial differential equation (PDE) over a discretized state space, whose computational cost scales exponentially with the state dimension.
Sum-of-squares (SOS) optimization~\citep{jarvis2003some, korda2014convex, siaGo2025yun, raCBF-ours, dai2022convex, clark2022semi} synthesizes polynomial safety certificates by relaxing polynomial nonnegativity conditions into SOS constraints, which can be solved via semidefinite programming. 
While effective for polynomial dynamics and semialgebraic safe sets, SOS methods scale poorly with state dimension and polynomial degree, and often produce conservative invariant sets that are much smaller than the maximal sets computed by HJ reachability when the latter is tractable.

Compared to the above approaches, \oursRA{} avoids the fundamental bottleneck of explicit invariant-set synthesis by learning a backup policy that induces an implicit control-invariant set online. 
This implicit set is then used to provide formal safety guarantees while enabling scalability to higher-dimensional systems.

\paragraph{Backup Control Barrier Functions.}
Our framework builds on the backup control barrier function (BCBF) framework~\citep{backupCBF, chen2021guaranteed, gurriet2020backupcbf_access, singletary2020sampleddata}. 
During online deployment, a BCBF certifies the set of states from which a fixed deterministic backup policy can return the system to a small base set within a finite horizon while remaining inside the safe set. 
This construction yields an implicit control-invariant set, whose associated safety-filter constraint is of relative degree one by construction and can explicitly account for input limits~\citep{ames2019control}. 
However, the backup policy is typically hand-designed, such as an LQR around an equilibrium~\citep{anderson2007optimal, khalil2002nonlinear}, and the resulting implicit invariant set is therefore often limited in size~\citep{kim2026backupbasedsafetyfilterscomparative}. 
\oursRA{} departs from this line of work in two key aspects. 
First, we learn the backup policy using the proposed safe-arrival value function, which characterizes time-optimal safe recovery to the base set.
Second, the BCBF constraints are embedded into a control-invariant layer via differentiable projection~\citep{hardnet, qpax}, enabling the RL policy to be trained end-to-end while preserving formal safety guarantees, rather than using the safety filter only as a post-hoc correction during deployment.

\paragraph{RL for Reach-Avoid Specifications.}
Reach-avoid specifications capture the objective of reaching a target set before entering a failure set~\citep{margellos2011hamilton,summers2010verification}, making them a natural choice for training backup policies to recover safely to a certified base set.
One line of work represents the target and failure sets through dense margin functions, as in reach-avoid differential games~\citep{margellos2011hamilton}. To address the computational cost of classical solvers~\citep{bokanowski2010reachability,chen2017multiplayer} in high-dimensional systems, subsequent work adapted this formulation to deep RL~\citep{hsu2021reachavoidRL,li2023learning}. Although this formulation captures reach-avoid semantics, the learned critic depends not only on the logical event of reaching the target before failure, but also on the choice, scale, and geometry of these task-specific shaping functions. This makes the optimization sensitive and difficult in practice: stable learning requires warm-starting with a standard RL objective and progressively annealing the discount factor.
Another line of work uses set indicators to encode whether a trajectory reaches the target before failure~\citep{summers2010verification}. Research along this line has focused primarily on classical optimal-control methods, including dynamic programming and linear programming~\citep{summers2010verification,haesaert2016experiment,kariotoglou2017linear,zhao2024inner}. Adapting this indicator-based formulation to deep RL for high-dimensional systems remains relatively underexplored. We build on this formulation by formalizing policy iteration for the safe-arrival objective. The policy-evaluation and policy-improvement steps translate directly into critic and actor updates, yielding a deep RL algorithm for training backup policies in continuous action spaces. The critic assigns credit according to whether and how quickly the system reaches the base set safely. This direct alignment with the reach-avoid specification simplifies credit assignment and optimization, allowing us to learn effective backup policies with a standard RL training pipeline and without task-specific reward-shaping heuristics. Although our algorithmic development focuses on deterministic dynamics and deterministic backup policies, the safe-arrival value function is probabilistic by construction and extends naturally to stochastic dynamics and policies.

\paragraph{Differentiable Optimization Layers.}
This line of work embeds optimization problems as differentiable layers inside neural networks. 
OptNet~\citep{amos2017optnet} differentiates through quadratic programs via implicit differentiation of the Karush--Kuhn--Tucker (KKT) conditions, enabling constrained optimization modules to be trained end-to-end. 
\hardnet~\citep{hardnet} provides a unified framework for efficiently training neural networks under general affine and convex constraints, with universal-approximation guarantees for networks equipped with such constrained layers. Recent variants include LMI-Net~\citep{tang2026lmi}, which specializes to linear matrix inequalities, and HardNet++~\citep{goertzen2026hardnet++}, which handles more general nonlinear constraints.
The control-invariant layer proposed in \oursRA{} inherits the universal-approximation guarantee of \hardCVX{}. 
Unlike prior differentiable layers, however, its constraints are not specified directly in closed form; instead, they are induced by the rollout of a learned backup policy. 
This allows \oursRA{} to enforce rollout-based control invariance while retaining the expressiveness needed for high-performance policies on high-dimensional systems.

\section{Background on Safe Control Theory}\label{app: safe_control_background}

This appendix provides the safe control background used in Sec.~\ref{subsec: safe_control} and Sec.~\ref{sec: sub1}. 
We focus on the deterministic, control-affine system~\eqref{eq: affine_dyn} in Sec.~\ref{subsec: safe_control}. 
For convenience, the system is restated below:
\begin{align}\label{app_eq: affine_system}
    \dot x = f(x)+g(x)u,\qquad x\in\X\subseteq\R^n,\quad u\in\U\subseteq\R^m,
\end{align}
where $\U$ encodes actuator limits and $\X$ is the state space of interest.
The functions $f:\R^n\to\R^n$ and $g:\R^n\to\R^{n\times m}$ are assumed locally Lipschitz, so that the closed-loop trajectories are well defined under the feedback policies considered in this work.

\subsection{Control Invariance}
As mentioned in Sec.~\ref{subsec: safe_control}, safety of the system~\eqref{app_eq: affine_system} is specified through a safe set 
\begin{align} \notag
    \calS:=\{ x\in\X: h_\calS(x)\ge0 \}
\end{align}
where $h^\calS:\R^n\rightarrow\R$ is continuously differentiable. 
Importantly, not every state in $\calS$ is necessarily safe in the dynamical sense. 
Under bounded inputs, there may be states inside $\calS$ from which all admissible control inputs eventually lead to a safety violation.
Thus, safe control typically seeks an invariant subset of $\calS$, rather than enforcing membership in $\calS$ alone~\citep{ames2019control, liu2014control, wei2019safe}.
\begin{definition}[Control Invariance]
     A set $\mathcal I\subseteq \calS$ is control-invariant if there exists a state-feedback policy $\pi:\X\rightarrow\U$ such that, for every initial condition $x(0)\in \mathcal I$, the corresponding closed-loop trajectory satisfies $x(t)\in\mathcal I$ for all $t\ge0$. 
\end{definition}
This definition is the controlled analogue of Nagumo's set-invariance condition~\citep{nagumo}.
In the ideal case, one would compute the maximal control-invariant subset of $\calS$ and restrict the controller to remain inside it.
However, computing such a set is difficult for nonlinear systems with input limits.

\subsection{Control Barrier Function}
Control barrier functions provide a differential certificate for such invariant subsets.
Let 
\begin{align} \notag
    \C := \{ x\in\X: h_\C(x)\ge0 \}\subseteq\calS
\end{align}
be a candidate control-invariant set.
A control barrier function certifies that, at each point in $\C$, there exists an admissible input that prevents $h_\C$ from decreasing too quickly~\citep{ames2014control}.

\begin{definition}[Control Barrier Function]
     A continuously differentiable function $h_\C:\R^n\rightarrow\R$ is a control barrier function (CBF)~\citep{ames2014control} for the zero-superlevel set $\C:=\{ x: h_\C(x)\ge0\}\subseteq \calS$ if there exists an extended class-$\mathcal K_\infty$ function $\alpha:\R\to\R$ such that, for all $x\in\C$,
    \begin{align}\label{app_eq: cbf_constraint}
        \sup_{u\in\U}\left\{ \nabla h_\C(x)^\top \left( f(x)+g(x)u \right) +\alpha(h_\C(x))\right\} \ge0.
    \end{align}
\end{definition}

The condition~\eqref{app_eq: cbf_constraint} guarantees that the following state-dependent feasible input set is nonempty:
\begin{align} \notag
    \Kb_\C(x) : = \left\{ u\in\U: \nabla h_\C(x)^\top \left(f(x)+g(x)u\right)+\alpha(h_\C(x))\ge 0 \right\}.
\end{align}
A standard CBF safety filter minimally modifies a nominal input $\unom$ by solving
\begin{align}
    &\usafe(x) = \argmin_{u\in\U} \left\| u-\unom \right\|_2^2\quad  \notag
    \\
    &\mathrm{s.t.} \quad \nabla h_\C(x)^\top\left( f(x) + g(x)u \right) + \alpha\left(h_\C(x)\right) \ge0. \notag
\end{align}
When $\U$ is convex, this is a convex quadratic program since the CBF constraint~\eqref{app_eq: cbf_constraint} is affine in $u$~\citep{ames2017cbf}.

\begin{proposition}[CBF Invariance Guarantee]
    Suppose $h_\C$ is a CBF for $\C$ and a locally Lipschitz feedback policy
    $\pi$ satisfies $\pi(x)\in\Kb_\C(x)$ for all $x\in\C$. Then $\C$ is forward
    invariant under $\pi$. Consequently, if $x(0)\in\C$, then $x(t)\in\C\subseteq
    \calS$ for all $t\ge0$.
\end{proposition}
\begin{proof}
    Proof in \citep[Cor.~2]{ames2017cbf}.
\end{proof}

The main challenge is not using a CBF once it is known, but synthesizing a valid $h_\C$ whose superlevel set is control-invariant under the bounded input set $\U$.
The backup-CBF construction addresses this by avoiding the explicit synthesis of a global invariant set.

\subsection{Backup Control Barrier Function}\label{app: bcbf_derivation}
Backup control barrier functions~\citep{backupCBF, gurriet2020backupcbf_access} start from a small certified base set and enlarge it implicitly using the rollout of a deterministic backup policy. Let
\begin{align} \notag
    \B := \{x\in\X:h_\B(x)\ge0\}\subseteq\calS
\end{align}
be a known control-invariant base set, and let $\pi_B:\B\to\U$ be a certified base controller. 
In \oursRA, this base pair is constructed locally around an equilibrium in Theorem~\ref{thm: base_set}. 
Let $\bp:\X\to\U$ be a deterministic backup policy that agrees with the certified base controller on $\B$, i.e., $\bp(x)=\pi_B(x)$ for $x\in\B$. 
Define the closed-loop backup vector field as $f_\bp(x):=f(x)+g(x)\bp(x)$, and let $\flowB(x,t)$ denote the corresponding flow map.

\begin{definition}[Backup-induced Control-Invariant Set]
    Given a backup horizon $T>0$, the $T$-time backup-induced set is
    \begin{align} \notag
        \CT(\bp)
        :=
        \mathcal R(\B,\calS,f_\bp,T)
        :=
        \left\{
        x\in\X:
        \flowB(x,T)\in\B
        \;\wedge\;
        \flowB(x,\td)\in\calS\quad \forall \td\in[0,T]
        \right\}.
    \end{align}
\end{definition}
Thus, $\CT(\bp)$ contains exactly the states from which the backup policy reaches the certified base set within time $T$ while remaining in the safe set throughout the backup rollout.

\begin{proposition}[Backup-Induced Invariance]
    If $\B$ is control-invariant under $\pi_B$ and $\bp=\pi_B$ on $\B$, then     $\B\subseteq\CT(\bp)\subseteq\calS$, and $\CT(\bp)$ is control-invariant.
\end{proposition}
\begin{proof}
    Proof in~\citep[Thm.~1]{backupCBF}.
\end{proof}

The set $\CT(\bp)$ can be represented as the zero-superlevel set of the implicit CBF
\begin{align}\label{app_eq: bcbf_barrier}
    h_\CT(x)
    =
    \min
    \left\{
    \min_{\td\in[0,T]} h_\calS\!\left(\flowB(x,\td)\right),
    \;
    h_\B\!\left(\flowB(x,T)\right)
    \right\}.
\end{align}
The nonsmooth minimum in~\eqref{app_eq: bcbf_barrier} makes a direct CBF-QP inconvenient. Instead, BCBF enforces sufficient CBF inequalities pointwise along the backup rollout and at the terminal base-set condition.

Let
\begin{align} \notag
    \Psi_\bp(x,\td):=\frac{\partial \flowB(x,\td)}{\partial x}
\end{align}
denote the sensitivity of the backup flow with respect to its initial condition.
When $f_\bp$ is differentiable, $\Psi_\bp$ satisfies the variational equation
\begin{align} \notag
    \frac{d}{d\td}\Psi_\bp(x,\td)
    =
    \frac{\partial f_\bp}{\partial x}\!\left(\flowB(x,\td)\right)
    \Psi_\bp(x,\td),
    \qquad
    \Psi_\bp(x,0)=I.
\end{align}

Here $\td\in[0,T]$ is the elapsed backup time, so the prediction $\Phi_\bp(x(t),\td)$ corresponds to the physical time $t+\td$.
For the safety constraint, differentiate while keeping this physical evaluation time fixed, giving $d\td/dt=-1$.
For the terminal constraint, keep $T$ fixed instead.
Under the differentiability assumptions for the backup flow, the chain rule gives
\begin{align*}
    \left.\frac{d}{dt}\Phi_\bp(x(t),\td) \right|_{t+\td\,\mathrm{fixed}} &=\Psi_\bp(x(t),\td)\dot x(t)-f_\bp(\Phi_\bp(x(t),\td)),\\
    \frac{d}{dt}\Phi_\bp(x(t),T) &=\Psi_\bp(x(t),T)\dot x(t).
\end{align*}
With $x=x(t)$ and $\dot x(t)=f(x)+g(x)u$, the corresponding barrier derivatives are
\begin{align}
    \D_\td^\calS(x,u) &:=\nabla h_\calS(\Phi_\bp(x,\td))^\top \left[\Psi_\bp(x,\td)(f(x)+g(x)u) - f_\bp(\Phi_\bp(x,\td))\right], \notag
    \\
    \D_T^\B(x,u) &:=\nabla h_\B(\Phi_\bp(x,T))^\top \Psi_\bp(x,T)(f(x)+g(x)u). \notag
\end{align}
Thus $\D_\td^\calS$ is the fixed-physical-time derivative, expressed using the relative-time index $\td$.
The terminal prediction has fixed duration $T$ and moving physical endpoint $t+T$, so its derivative has no subtraction.
These conventions reproduce the BCBF construction of~\citet{backupCBF}.

Given extended class-$\K_\infty$ functions $\alpha_\calS$ and $\alpha_\B$, the BCBF-admissible input set is
\begin{equation}\label{eq: app_bcbf_admissible_set}
\Kb_\bcbf(x;\bp) := \left\{u\in\U\,:
\begin{aligned}
    &\D_\td^\calS(x,u)
    +\alpha_\calS\!\left(h_\calS(\Phi_\bp(x,\td))\right)\ge0,
    \quad\forall\td\in[0,T],\\
    &\D_T^\B(x,u)
    +\alpha_\B\!\left(h_\B(\Phi_\bp(x,T))\right)\ge0
\end{aligned}
\right\}.
\end{equation}
For fixed $x$ and fixed $\bp$, the rollout $\flowB(x,\td)$ and sensitivity $\Psi_\bp(x,\td)$ are constants with respect to the optimization variable $u$.
Therefore, all constraints in~\eqref{eq: app_bcbf_admissible_set} are affine in $u$, and $\Kb_\bcbf(x;\bp)$ is convex whenever $\U$ is convex.

The corresponding BCBF-QP is
\begin{align}
    \usafe(x)
    =
    \argmin_{u}\quad
    &\|u-\unom(x)\|_2^2 \notag
    \\
    \mathrm{s.t.}\quad
    &u\in \Kb_\bcbf(x;\bp). \notag
\end{align}

\begin{proposition}[BCBF Feasibility and Safety]
    For every $x\in\CT(\bp)$, the set $\Kb_\bcbf(x;\bp)$ is nonempty. 
    In particular, $\bp(x)\in\Kb_\bcbf(x;\bp)$. 
    Moreover, any locally Lipschitz policy $\pi$ that satisfies $\pi(x)\in\Kb_\bcbf(x;\bp)$ for all $x\in\CT(\bp)$ renders $\CT(\bp)$ forward invariant, and therefore keeps the trajectory inside $\calS$.
\end{proposition}
\begin{proof}
    Proof in \citep[Thm.~2]{backupCBF}.    
\end{proof}

Finally, BCBFs also address the relative-degree issue that arises when the original safety function $h_\calS$ depends on states that are not directly actuated. 
The composed functions $h_\calS\circ\flowB(\cdot,\td)$ depend on the current input through the flow sensitivity $\Psi_\bp(x,\td)g(x)$. 
Under mild controllability and nondegeneracy assumptions, these composed constraints have relative degree one for positive rollout times~\citep[Thm.~3]{backupCBF}. 
This is why \oursRA{} can specify safety directly in physically meaningful coordinates, such as altitude or lane position, without constructing high-order CBFs for the full system.

\section{\rafull Function}\label{app: save}

The \ralow{} function is a central ingredient of Phase I (Sec.~\ref{sec: phase1}). 
This section provides the formal characterization and its key properties. 
Although we use it for training the backup policy in our \oursRA{} framework, the construction may be of independent interest and is not limited to the deterministic setting: the same methodology extends naturally to stochastic dynamics and randomized policies. 
For this reason, we present the definitions in terms of trajectory laws, so that the notation applies directly to stochastic settings while reducing to pathwise statements in the deterministic case considered here.

We consider a deterministic control system
\begin{align}\label{app_eq: deterministic_system}
    x_{k+1}=F(x_k,u_k),\qquad x_k\in\X,\quad u_k\in\U,
\end{align}
together with deterministic policies $\pi:\X\to\U$. Let $\B\subseteq\X$ denote the base set, let $\calS\subseteq\X$ denote the safe set, and define the failure set by
\begin{align}
    \F:=\X\setminus\calS.
\end{align}
The continuation set is
\begin{align}\label{app_eq: save_continuation_set}
    \calH:=\X\setminus(\B\cup\F).
\end{align}
Throughout, we assume $\B\cap\F=\emptyset$, so that $\B$, $\F$, and $\calH$ form a disjoint partition of $\X$. We use the indicators
\begin{align}\label{app_eq: save_indicators}
    \vb b(x)=\mathbf{1}_{\B}(x),\qquad
    \vb f(x)=\mathbf{1}_{\F}(x),\qquad
    \vb c(x)=1-\vb b(x)-\vb f(x)=\mathbf{1}_{\calH}(x).
\end{align}

For an initial state-action pair $(x,u)\in\X\times\U$ and a policy $\pi$, let
\begin{align}
    \omega_{x,u}^{\pi}=(x_0,u_0,x_1,u_1,\ldots)
\end{align}
denote the rollout generated by
\begin{align}\label{app_eq: save_rollout}
    x_0=x,\qquad
    u_0=u,\qquad
    x_{k+1}=F(x_k,u_k)\ \text{for }k\ge0,\qquad
    u_k=\pi(x_k)\ \text{for }k\ge1.
\end{align}
Because the dynamics and policy are deterministic, the rollout $\omega_{x,u}^{\pi}$ is unique.

We write $\mathbb P_{x,u}^{\pi}$ for the induced law on trajectory space and $\mathbb E_{x,u}^{\pi}$ for expectation with respect to this law. In the deterministic setting,
\begin{align}\label{app_eq: deterministic_dirac_law}
    \mathbb P_{x,u}^{\pi}=\delta_{\omega_{x,u}^{\pi}},
    \qquad
    \mathbb E_{x,u}^{\pi}[Z]=Z(\omega_{x,u}^{\pi})
\end{align}
for every trajectory functional $Z$. Thus any event depending only on the controlled rollout has probability either $0$ or $1$. We keep the law-and-expectation notation because it remains unchanged when randomness is introduced into the policy or the dynamics.

The hitting times of the base and failure sets are the trajectory functionals
\begin{align}\label{app_eq: save_hitting_times}
    \tau_\B(\omega)=\inf\{k\ge0:x_k\in\B\},
    \qquad
    \tau_\F(\omega)=\inf\{k\ge0:x_k\in\F\},
\end{align}
with the convention $\inf\emptyset=\infty$. When the rollout is clear, we simply write $\tau_\B$ and $\tau_\F$. The safe-arrival event is
\begin{align}\label{app_eq: save_safe_arrival_event}
    \mathcal E_{\mathrm{SA}}
    :=
    \{\omega:\tau_\B(\omega)<\tau_\F(\omega),\ \tau_\B(\omega)<\infty\}.
\end{align}
This event occurs exactly when the trajectory arrives at the base set in finite time before entering the failure set.

\subsection{Undiscounted \rafull}\label{app: undiscounted_save}

We first define the undiscounted \ralow. This value is the feasibility identifier associated with the safe-arrival specification: it identifies whether a policy succeeds in arriving at the base set before failure~\citep{summers2010verification}.

\begin{definition}[Undiscounted \rafull]
    For a deterministic policy $\pi$, define
    \begin{align}\label{app_eq: save_undiscounted_def}
        Q_{\mathrm{SA}}^{\pi}(x,u)
        :=
        \mathbb P_{x,u}^{\pi}(\mathcal E_{\mathrm{SA}})
        =
        \mathbb P_{x,u}^{\pi}(\tau_\B<\tau_\F,\ \tau_\B<\infty).
    \end{align}
    The associated state-value function is
    \begin{align}\label{app_eq: save_undiscounted_state_def}
        V_{\mathrm{SA}}^{\pi}(x)
        :=
        Q_{\mathrm{SA}}^{\pi}(x,\pi(x)).
    \end{align}
\end{definition}

The event-based definition admits an equivalent pathwise representation in terms of the indicator functions in~\eqref{app_eq: save_indicators}.

\begin{proposition}[Indicator Representation of Undiscounted \rafull]
\label{app_prop: indicator_representation}
    For every deterministic policy $\pi$ and every $(x,u)\in\X\times\U$,
    \begin{align}\label{app_eq: indicator_representation}
        Q_{\mathrm{SA}}^{\pi}(x,u)
        =
        \mathbb E_{x,u}^{\pi}
        \left[
            \sum_{k=0}^{\infty}
            \left(
                \prod_{i=0}^{k-1}\vb c(x_i)
            \right)
            \vb b(x_k)
        \right],
    \end{align}
    where the empty product is interpreted as $1$.
\end{proposition}

\begin{proof}
    Define the pathwise quantity
    \begin{align}
        Y
        :=
        \sum_{k=0}^{\infty}
        \left(
            \prod_{i=0}^{k-1}\vb c(x_i)
        \right)
        \vb b(x_k).
    \end{align}
    We show that
    \begin{align}\label{app_eq: indicator_Y}
        Y
        =
        \mathbf{1}_{\{\tau_\B<\tau_\F,\ \tau_\B<\infty\}}.
    \end{align}

    Suppose first that $\tau_\B<\tau_\F$ and $\tau_\B<\infty$. Then $x_i\in\calH$ for all $i<\tau_\B$, so $\vb c(x_i)=1$ for $i=0,\ldots,\tau_\B-1$. Moreover, $x_{\tau_\B}\in\B$, so $\vb b(x_{\tau_\B})=1$. Hence the term in $Y$ with $k=\tau_\B$ equals $1$. For $k<\tau_\B$, we have $x_k\notin\B$, so $\vb b(x_k)=0$. For $k>\tau_\B$, the product $\prod_{i=0}^{k-1}\vb c(x_i)$ contains the factor $\vb c(x_{\tau_\B})=0$. Therefore all terms except the term with $k=\tau_\B$ vanish, and $Y=1$.

    Conversely, suppose that the event $\{\tau_\B<\tau_\F,\ \tau_\B<\infty\}$ does not occur. If the trajectory enters $\F$ before entering $\B$, then $\vb b(x_k)=0$ for all $k\le\tau_\F$, and for all $k>\tau_\F$ the product contains the factor $\vb c(x_{\tau_\F})=0$. Hence every term in $Y$ is zero. If the trajectory never enters $\B$, then $\vb b(x_k)=0$ for all $k$, so again $Y=0$.

    Thus~\eqref{app_eq: indicator_Y} holds pathwise. Taking expectation with respect to $\mathbb P_{x,u}^{\pi}$ gives~\eqref{app_eq: indicator_representation}.
\end{proof}

The indicator representation yields the following one-step self-consistency structure.

\begin{proposition}[Self-Consistency Condition of Undiscounted \rafull]
    For every deterministic policy $\pi$ and every $(x,u)\in\X\times\U$,
    \begin{align}\label{app_eq: save_undiscounted_self_consistency}
        Q_{\mathrm{SA}}^{\pi}(x,u)
        =
        \vb b(x)
        +
        \vb c(x)\,
        Q_{\mathrm{SA}}^{\pi}\!\left(F(x,u),\pi(F(x,u))\right).
    \end{align}
\end{proposition}

\begin{proof}
    Starting from~\eqref{app_eq: indicator_representation},
    \begin{align}
        Q_{\mathrm{SA}}^{\pi}(x,u)
        &=
        \mathbb E_{x,u}^{\pi}
        \left[
            \sum_{k=0}^{\infty}
            \left(
                \prod_{i=0}^{k-1}\vb c(x_i)
            \right)
            \vb b(x_k)
        \right] \\
        &=
        \mathbb E_{x,u}^{\pi}
        \left[
            \vb b(x_0)
            +
            \sum_{k=1}^{\infty}
            \left(
                \prod_{i=0}^{k-1}\vb c(x_i)
            \right)
            \vb b(x_k)
        \right] \\
        &=
        \vb b(x)
        +
        \vb c(x)\,
        \mathbb E_{x,u}^{\pi}
        \left[
            \sum_{k=1}^{\infty}
            \left(
                \prod_{i=1}^{k-1}\vb c(x_i)
            \right)
            \vb b(x_k)
        \right].
    \end{align}
    Re-indexing with $j=k-1$ gives
    \begin{align}
        Q_{\mathrm{SA}}^{\pi}(x,u)
        =
        \vb b(x)
        +
        \vb c(x)\,
        \mathbb E_{x,u}^{\pi}
        \left[
            \sum_{j=0}^{\infty}
            \left(
                \prod_{i=1}^{j}\vb c(x_i)
            \right)
            \vb b(x_{j+1})
        \right].
    \end{align}
    Let $x^+=F(x,u)=x_1$. From step $1$ onward, the shifted rollout is the rollout initialized at $(x^+,\pi(x^+))$ under the same policy $\pi$. Therefore,
    \begin{align}
        \mathbb E_{x,u}^{\pi}
        \left[
            \sum_{j=0}^{\infty}
            \left(
                \prod_{i=1}^{j}\vb c(x_i)
            \right)
            \vb b(x_{j+1})
        \right]
        =
        Q_{\mathrm{SA}}^{\pi}(x^+,\pi(x^+)).
    \end{align}
    Substituting $x^+=F(x,u)$ proves the result.
\end{proof}

The optimal undiscounted values are
\begin{align}\label{app_eq: save_optimal_undiscounted_values}
    Q_{\mathrm{SA}}^{\star}(x,u)
    :=
    \sup_{\pi}Q_{\mathrm{SA}}^{\pi}(x,u),
    \qquad
    V_{\mathrm{SA}}^{\star}(x)
    :=
    \max_{u\in\U}Q_{\mathrm{SA}}^{\star}(x,u).
\end{align}
Since the first action $u$ is fixed in $Q_{\mathrm{SA}}^{\star}(x,u)$ and only the continuation policy remains to be optimized after the successor state, the optimal values satisfy
\begin{align}\label{app_eq: save_optimal_undiscounted_bellman}
    Q_{\mathrm{SA}}^{\star}(x,u)
    =
    \vb b(x)
    +
    \vb c(x)
    \max_{u^+\in\U}
    Q_{\mathrm{SA}}^{\star}\!\left(F(x,u),u^+\right).
\end{align}
Equivalently,
\begin{align}\label{app_eq: save_optimal_undiscounted_bellman_state}
    V_{\mathrm{SA}}^{\star}(x)
    =
    \max_{u\in\U}
    \left[
        \vb b(x)
        +
        \vb c(x)
        V_{\mathrm{SA}}^{\star}(F(x,u))
    \right].
\end{align}

The undiscounted value defines the states from which safe arrival is feasible.

\begin{definition}[Safe-arrival region]
    The deterministic safe-arrival region is
    \begin{align}\label{app_eq: save_safe_arrival_region}
        \mathcal R_{\mathrm{SA}}
        :=
        \{x\in\X:V_{\mathrm{SA}}^{\star}(x)=1\}.
    \end{align}
\end{definition}

Equivalently, $\mathcal R_{\mathrm{SA}}$ is the set of states from which there exists a deterministic policy that arrives at $\B$ before entering $\F$.

\begin{definition}[Minimum safe arrival time]\label{app_def: save_minimum_time_to_go}
    Define
    \begin{align}\label{app_eq: save_minimum_time_to_go}
        d(x)
        :=
        \inf
        \Bigl\{
            k\in\{0,1,2,\ldots\}:
            &\ \exists (u_0,\ldots,u_{k-1})\in\U^k
            \ \text{such that} \notag\\
            &\ x_0=x,\quad x_{i+1}=F(x_i,u_i)\ \text{for }i=0,\ldots,k-1, \notag\\
            &\ x_i\in\calH\ \text{for all }i<k,\quad x_k\in\B
        \Bigr\},
    \end{align}
    with the convention $d(x)=\infty$ if the set above is empty. For $k=0$, the safety condition over $i<k$ is vacuous, so $d(x)=0$ exactly when $x\in\B$.
\end{definition}

The following finite-step construction gives an equivalent characterization of the safe-arrival region.

\begin{proposition}[Finite-step characterization of the safe-arrival region]
    Define
    \begin{align}\label{app_eq: save_finite_step_sets}
        \mathcal R_0&:=\B, \\
        \mathcal R_{n+1}
        &:=
        \mathcal R_n
        \cup
        \left\{
            x\in\calH:
            \exists u\in\U\ \text{such that}\ F(x,u)\in\mathcal R_n
        \right\},
        \qquad n\ge0.
    \end{align}
    Then
    \begin{align}\label{app_eq: save_region_finite_step}
        \mathcal R_{\mathrm{SA}}
        =
        \bigcup_{n=0}^{\infty}\mathcal R_n
        =
        \{x\in\X:d(x)<\infty\}.
    \end{align}
    Moreover, for every $x\in\mathcal R_{\mathrm{SA}}$,
    \begin{align}\label{app_eq: save_d_from_finite_step_sets}
        d(x)=\min\{n\ge0:x\in\mathcal R_n\}.
    \end{align}
\end{proposition}

\begin{proof}
    We prove by induction that $x\in\mathcal R_n$ if and only if safe arrival from $x$ can be achieved in at most $n$ steps.

    For $n=0$, this holds because $\mathcal R_0=\B$. Suppose the claim holds for some $n\ge0$. If $x\in\mathcal R_{n+1}$, then either $x\in\mathcal R_n$, in which case the induction hypothesis applies, or $x\in\calH$ and there exists $u\in\U$ such that $F(x,u)\in\mathcal R_n$. In the latter case, applying $u$ for one step and then using the induction hypothesis from $F(x,u)$ gives safe arrival in at most $n+1$ steps.

    Conversely, suppose safe arrival from $x$ can be achieved in at most $n+1$ steps. If it can be achieved in at most $n$ steps, then $x\in\mathcal R_n\subseteq\mathcal R_{n+1}$. Otherwise, the first step is taken from a state in $\calH$, and some first action $u$ leads to a successor from which safe arrival can be achieved in at most $n$ steps. By the induction hypothesis, this successor lies in $\mathcal R_n$, so $x\in\mathcal R_{n+1}$.

    Thus $\bigcup_{n=0}^{\infty}\mathcal R_n$ is exactly the set of states with finite minimum safe arrival time. This set is also $\mathcal R_{\mathrm{SA}}$ by the definition of the optimal undiscounted value. The identity~\eqref{app_eq: save_d_from_finite_step_sets} follows from the same induction argument and the definition of $d(x)$.
\end{proof}

\subsection{Discounted \rafull}\label{app: discounted_save}

We now define the discounted \ralow. The discounted value retains the same first-arrival structure as the undiscounted one, while assigning larger value to earlier safe arrival at the base set~\citep{haesaert2016experiment}.

\begin{definition}[Discounted \rafull]
    Fix a discount factor $\beta\in(0,1)$. For a deterministic policy $\pi$, define
    \begin{align}\label{app_eq: save_discounted_def}
        Q_{\mathrm{SA},\beta}^{\pi}(x,u)
        :=
        \mathbb E_{x,u}^{\pi}
        \left[
            \sum_{k=0}^{\infty}
            \beta^k
            \left(
                \prod_{i=0}^{k-1}\vb c(x_i)
            \right)
            \vb b(x_k)
        \right].
    \end{align}
    The associated state-value function is
    \begin{align}\label{app_eq: save_discounted_state_def}
        V_{\mathrm{SA},\beta}^{\pi}(x)
        :=
        Q_{\mathrm{SA},\beta}^{\pi}(x,\pi(x)).
    \end{align}
\end{definition}

The discounted sum pays only at the first safe arrival time. This gives both a hitting-time interpretation and a random-horizon interpretation.

\begin{proposition}[Interpretations of Discounted \rafull]
    For every deterministic policy $\pi$ and every $(x,u)\in\X\times\U$,
    \begin{align}\label{app_eq: save_discounted_hitting}
        Q_{\mathrm{SA},\beta}^{\pi}(x,u)
        =
        \mathbb E_{x,u}^{\pi}
        \left[
            \beta^{\tau_\B}
            \mathbf{1}_{\{\tau_\B<\tau_\F,\ \tau_\B<\infty\}}
        \right].
    \end{align}
    Moreover, let $N$ be independent of the controlled rollout and satisfy
    \begin{align}\label{app_eq: save_discounted_geometric_horizon}
        \mathbb P_N(N=k)=(1-\beta)\beta^k,
        \qquad k=0,1,2,\ldots.
    \end{align}
    Then
    \begin{align}\label{app_eq: save_discounted_random_horizon}
        Q_{\mathrm{SA},\beta}^{\pi}(x,u)
        =
        \left(\mathbb P_{x,u}^{\pi}\otimes\mathbb P_N\right)
        \left(
            \tau_\B<\tau_\F,\ \tau_\B\le N
        \right).
    \end{align}
\end{proposition}

\begin{proof}
    Define
    \begin{align}
        Y_\beta
        :=
        \sum_{k=0}^{\infty}
        \beta^k
        \left(
            \prod_{i=0}^{k-1}\vb c(x_i)
        \right)
        \vb b(x_k).
    \end{align}
    The same pathwise argument used in the proof of Proposition~\ref{app_prop: indicator_representation} shows that, on the event $\{\tau_\B<\tau_\F,\ \tau_\B<\infty\}$, the only nonzero term in $Y_\beta$ occurs at $k=\tau_\B$, and its value is $\beta^{\tau_\B}$. Outside this event, all terms vanish. Hence
    \begin{align}
        Y_\beta
        =
        \beta^{\tau_\B}
        \mathbf{1}_{\{\tau_\B<\tau_\F,\ \tau_\B<\infty\}}.
    \end{align}
    Taking expectation proves~\eqref{app_eq: save_discounted_hitting}.

    For the random-horizon identity, note that
    \begin{align}
        \mathbb P_N(N\ge k)
        =
        \sum_{n=k}^{\infty}(1-\beta)\beta^n
        =
        \beta^k.
    \end{align}
    By independence of $N$ and the controlled rollout,
    \begin{align}
        &\left(\mathbb P_{x,u}^{\pi}\otimes\mathbb P_N\right)
        \left(
            \tau_\B<\tau_\F,\ \tau_\B\le N
        \right) \notag\\
        &\qquad
        =
        \sum_{k=0}^{\infty}
        \mathbb P_{x,u}^{\pi}(\tau_\B=k<\tau_\F)\,
        \mathbb P_N(N\ge k) \\
        &\qquad
        =
        \sum_{k=0}^{\infty}
        \beta^k
        \mathbb P_{x,u}^{\pi}(\tau_\B=k<\tau_\F) \\
        &\qquad
        =
        \mathbb E_{x,u}^{\pi}
        \left[
            \beta^{\tau_\B}
            \mathbf{1}_{\{\tau_\B<\tau_\F,\ \tau_\B<\infty\}}
        \right].
    \end{align}
    Combining this identity with~\eqref{app_eq: save_discounted_hitting} proves~\eqref{app_eq: save_discounted_random_horizon}.
\end{proof}

The discounted value satisfies the following one-step self-consistency relation.

\begin{proposition}[Self-Consistency Condition of Discounted \rafull]
    For every deterministic policy $\pi$ and every $(x,u)\in\X\times\U$,
    \begin{align}\label{app_eq: save_discounted_self_consistency}
        Q_{\mathrm{SA},\beta}^{\pi}(x,u)
        =
        \vb b(x)
        +
        \beta\,\vb c(x)\,
        Q_{\mathrm{SA},\beta}^{\pi}\!\left(F(x,u),\pi(F(x,u))\right).
    \end{align}
\end{proposition}

\begin{proof}
    Starting from~\eqref{app_eq: save_discounted_def},
    \begin{align}
        Q_{\mathrm{SA},\beta}^{\pi}(x,u)
        &=
        \mathbb E_{x,u}^{\pi}
        \left[
            \vb b(x_0)
            +
            \sum_{k=1}^{\infty}
            \beta^k
            \left(
                \prod_{i=0}^{k-1}\vb c(x_i)
            \right)
            \vb b(x_k)
        \right] \\
        &=
        \vb b(x)
        +
        \beta\,\vb c(x)\,
        \mathbb E_{x,u}^{\pi}
        \left[
            \sum_{k=1}^{\infty}
            \beta^{k-1}
            \left(
                \prod_{i=1}^{k-1}\vb c(x_i)
            \right)
            \vb b(x_k)
        \right].
    \end{align}
    Re-indexing with $j=k-1$ gives
    \begin{align}
        Q_{\mathrm{SA},\beta}^{\pi}(x,u)
        =
        \vb b(x)
        +
        \beta\,\vb c(x)\,
        \mathbb E_{x,u}^{\pi}
        \left[
            \sum_{j=0}^{\infty}
            \beta^j
            \left(
                \prod_{i=1}^{j}\vb c(x_i)
            \right)
            \vb b(x_{j+1})
        \right].
    \end{align}
    Let $x^+=F(x,u)=x_1$. The shifted rollout from step $1$ onward is the rollout initialized at $(x^+,\pi(x^+))$ under the same policy $\pi$. Therefore,
    \begin{align}
        \mathbb E_{x,u}^{\pi}
        \left[
            \sum_{j=0}^{\infty}
            \beta^j
            \left(
                \prod_{i=1}^{j}\vb c(x_i)
            \right)
            \vb b(x_{j+1})
        \right]
        =
        Q_{\mathrm{SA},\beta}^{\pi}(x^+,\pi(x^+)).
    \end{align}
    Substituting $x^+=F(x,u)$ proves the result.
\end{proof}

Define the optimal discounted values by
\begin{align}\label{app_eq: save_optimal_discounted_values}
    Q_{\mathrm{SA},\beta}^{\star}(x,u)
    :=
    \sup_{\pi}Q_{\mathrm{SA},\beta}^{\pi}(x,u),
    \qquad
    V_{\mathrm{SA},\beta}^{\star}(x)
    :=
    \max_{u\in\U}Q_{\mathrm{SA},\beta}^{\star}(x,u).
\end{align}
The corresponding Bellman equation is
\begin{align}\label{app_eq: save_optimal_discounted_bellman}
    Q_{\mathrm{SA},\beta}^{\star}(x,u)
    =
    \vb b(x)
    +
    \beta\,\vb c(x)
    \max_{u^+\in\U}
    Q_{\mathrm{SA},\beta}^{\star}\!\left(F(x,u),u^+\right).
\end{align}
Equivalently,
\begin{align}\label{app_eq: save_optimal_discounted_bellman_state}
    V_{\mathrm{SA},\beta}^{\star}(x)
    =
    \max_{u\in\U}
    \left[
        \vb b(x)
        +
        \beta\,\vb c(x)
        V_{\mathrm{SA},\beta}^{\star}(F(x,u))
    \right].
\end{align}

The discounted Bellman operator
\begin{align}
    (\mathcal T_\beta Q)(x,u)
    :=
    \vb b(x)
    +
    \beta\,\vb c(x)
    \max_{u^+\in\U}
    Q(F(x,u),u^+)
\end{align}
is a $\beta$-contraction in the sup norm. Indeed, for any two bounded action-value functions $Q_1,Q_2$,
\begin{align}
    \left|
        (\mathcal T_\beta Q_1)(x,u)
        -
        (\mathcal T_\beta Q_2)(x,u)
    \right|
    &\le
    \beta\,\vb c(x)
    \max_{u^+\in\U}
    \left|
        Q_1(F(x,u),u^+)
        -
        Q_2(F(x,u),u^+)
    \right| \\
    &\le
    \beta
    \|Q_1-Q_2\|_\infty.
\end{align}
Consequently, the discounted Bellman equation has a unique bounded fixed point, namely $Q_{\mathrm{SA},\beta}^{\star}$.

The discounted and undiscounted fixed-policy values have the same success support.

\begin{proposition}[Discounting Preserves Safe-Arrival Feasibility]
    For every deterministic policy $\pi$, every $(x,u)\in\X\times\U$, and every $\beta\in(0,1)$,
    \begin{align}\label{app_eq: save_fixed_policy_support_equivalence}
        Q_{\mathrm{SA}}^{\pi}(x,u)=1
        \quad\Longleftrightarrow\quad
        Q_{\mathrm{SA},\beta}^{\pi}(x,u)>0.
    \end{align}
    Moreover,
    \begin{align}\label{app_eq: save_fixed_policy_beta_limit}
        \lim_{\beta\uparrow1}Q_{\mathrm{SA},\beta}^{\pi}(x,u)
        =
        Q_{\mathrm{SA}}^{\pi}(x,u).
    \end{align}
\end{proposition}

\begin{proof}
    Since the rollout is deterministic, either the safe-arrival event occurs or it does not. If it occurs, then $\tau_\B=N$ for some finite $N$, and
    \begin{align}
        Q_{\mathrm{SA}}^{\pi}(x,u)=1,
        \qquad
        Q_{\mathrm{SA},\beta}^{\pi}(x,u)=\beta^N>0.
    \end{align}
    If it does not occur, then both values are zero. This proves~\eqref{app_eq: save_fixed_policy_support_equivalence}. The limit~\eqref{app_eq: save_fixed_policy_beta_limit} follows from $\beta^N\to1$ as $\beta\uparrow1$ for every finite $N$.
\end{proof}

The next proposition gives the precise time-optimal implication of discounting.

\begin{proposition}[Arrival-Time Optimality of the Discounted Value]\label{app_prop: save_discounted_time_optimality}
    For every $x\in\X$,
    \begin{align}\label{app_eq: save_time_optimal_value}
        V_{\mathrm{SA},\beta}^{\star}(x)
        =
        \begin{cases}
            \beta^{d(x)}, & d(x)<\infty,\\
            0, & d(x)=\infty.
        \end{cases}
    \end{align}
    Moreover, for every $(x,u)\in\X\times\U$,
    \begin{align}\label{app_eq: save_time_optimal_q_value}
        Q_{\mathrm{SA},\beta}^{\star}(x,u)
        =
        \begin{cases}
            1, & x\in\B,\\
            \beta^{1+d(F(x,u))}, & x\in\calH\ \text{and}\ d(F(x,u))<\infty,\\
            0, & \text{otherwise}.
        \end{cases}
    \end{align}
    Consequently, for every $\beta\in(0,1)$,
    \begin{align}\label{app_eq: save_region_support_equivalence}
        \mathcal R_{\mathrm{SA}}
        =
        \{x\in\X:V_{\mathrm{SA}}^{\star}(x)=1\}
        =
        \{x\in\X:d(x)<\infty\}
        =
        \{x\in\X:V_{\mathrm{SA},\beta}^{\star}(x)>0\},
    \end{align}
    and
    \begin{align}\label{app_eq: save_optimal_beta_limit}
        \lim_{\beta\uparrow1}
        V_{\mathrm{SA},\beta}^{\star}(x)
        =
        V_{\mathrm{SA}}^{\star}(x).
    \end{align}
\end{proposition}

\begin{proof}
    Fix $x\in\X$. Under deterministic dynamics and a deterministic policy, either the trajectory arrives at $\B$ before $\F$ in exactly $N$ steps for some finite $N$, or the safe-arrival event does not occur. In the first case, the discounted value is $\beta^N$; in the second case, it is zero. Since $\beta\in(0,1)$, maximizing the discounted value among successful policies is equivalent to minimizing the number of steps required for safe arrival. The minimum such number is $d(x)$. Hence~\eqref{app_eq: save_time_optimal_value} holds.

    The action-value expression~\eqref{app_eq: save_time_optimal_q_value} follows by separating the first action. If $x\in\B$, safe arrival has already occurred at time $0$, so the value is $1$. If $x\in\F$, failure has already occurred at time $0$, so the value is $0$. If $x\in\calH$, taking action $u$ first moves the system to $F(x,u)$. Safe arrival remains possible exactly when $d(F(x,u))<\infty$, in which case the shortest safe arrival time after taking $u$ is $1+d(F(x,u))$.

    The equivalence~\eqref{app_eq: save_region_support_equivalence} follows from~\eqref{app_eq: save_time_optimal_value} and the definition of $\mathcal R_{\mathrm{SA}}$. Finally, if $d(x)<\infty$, then $\beta^{d(x)}\to1$ as $\beta\uparrow1$; if $d(x)=\infty$, then $V_{\mathrm{SA},\beta}^{\star}(x)=0$ for every $\beta\in(0,1)$. This proves~\eqref{app_eq: save_optimal_beta_limit}.
\end{proof}

\subsection{Behavior of Optimal Safe-Arrival Actions}\label{app: save_behavior}

The preceding results imply a complete qualitative description of optimal safe-arrival actions across the state space. The relevant state-space partition is
\begin{align}\label{app_eq: save_state_partition}
    \X
    =
    \B
    \,\dot\cup\,
    \F
    \,\dot\cup\,
    (\mathcal R_{\mathrm{SA}}\cap\calH)
    \,\dot\cup\,
    (\calH\setminus\mathcal R_{\mathrm{SA}}).
\end{align}

\begin{proposition}[Qualitative Structure of Optimal Safe-Arrival Actions]
    The optimal safe-arrival values and actions satisfy the following four cases.
    \begin{enumerate}[label=\arabic*.,itemsep=0ex,topsep=0.25ex]
        \item If $x\in\B$, then
        \begin{align}
            V_{\mathrm{SA}}^{\star}(x)=1,
            \qquad
            V_{\mathrm{SA},\beta}^{\star}(x)=1.
        \end{align}
        Every action has the same value, because safe arrival has already occurred at time $0$.

        \item If $x\in\F$, then
        \begin{align}
            V_{\mathrm{SA}}^{\star}(x)=0,
            \qquad
            V_{\mathrm{SA},\beta}^{\star}(x)=0.
        \end{align}
        Every action has the same value, because failure has already occurred at time $0$.

        \item If $x\in\mathcal R_{\mathrm{SA}}\cap\calH$, then an action $u\in\U$ is optimal for the undiscounted optimal action-value if and only if
        \begin{align}\label{app_eq: save_undiscounted_optimal_action_condition}
            Q_{\mathrm{SA}}^{\star}(x,u)=1.
        \end{align}
        Equivalently,
        \begin{align}\label{app_eq: save_undiscounted_successor_condition}
            F(x,u)\in\mathcal R_{\mathrm{SA}}.
        \end{align}
        Thus, an undiscounted optimal first action preserves safe-arrival feasibility. A complete policy must still select future actions so that $\B$ is reached in finite time; the undiscounted criterion does not distinguish among successful policies with different arrival times.

        For the discounted criterion,
        \begin{align}\label{app_eq: save_discounted_action_value_inside_region}
            Q_{\mathrm{SA},\beta}^{\star}(x,u)
            =
            \begin{cases}
                \beta^{1+d(F(x,u))}, & F(x,u)\in\mathcal R_{\mathrm{SA}},\\
                0, & F(x,u)\notin\mathcal R_{\mathrm{SA}}.
            \end{cases}
        \end{align}
        Hence $u$ is discounted-optimal at $x$ if and only if
        \begin{align}\label{app_eq: save_discounted_optimal_action_condition}
            F(x,u)\in\mathcal R_{\mathrm{SA}}
            \qquad\text{and}\qquad
            d(F(x,u))=d(x)-1.
        \end{align}
        Therefore, a discounted optimal policy decreases the minimum safe arrival time by one at every pre-arrival step and reaches $\B$ in exactly $d(x)$ steps.

        \item If $x\in\calH\setminus\mathcal R_{\mathrm{SA}}$, then
        \begin{align}
            V_{\mathrm{SA}}^{\star}(x)=0,
            \qquad
            V_{\mathrm{SA},\beta}^{\star}(x)=0.
        \end{align}
        No policy can safely arrive at $\B$ from such a state. Hence every action has value zero under the safe-arrival objectives unless an additional secondary criterion is imposed.
    \end{enumerate}
\end{proposition}

\begin{proof}
    If $x\in\B$, then $\tau_\B=0<\tau_\F$, so both the undiscounted and discounted optimal values equal $1$, independently of the chosen action. If $x\in\F$, then $\tau_\F=0$, so the safe-arrival event cannot occur, and both values equal $0$.

    Now let $x\in\mathcal R_{\mathrm{SA}}\cap\calH$. Since $x\in\mathcal R_{\mathrm{SA}}$, we have $V_{\mathrm{SA}}^{\star}(x)=1$. Therefore an action $u$ is optimal for the undiscounted objective exactly when $Q_{\mathrm{SA}}^{\star}(x,u)=1$. By~\eqref{app_eq: save_optimal_undiscounted_bellman}, and since $\vb b(x)=0$ and $\vb c(x)=1$ for $x\in\calH$,
    \begin{align}
        Q_{\mathrm{SA}}^{\star}(x,u)
        =
        V_{\mathrm{SA}}^{\star}(F(x,u)).
    \end{align}
    Thus $Q_{\mathrm{SA}}^{\star}(x,u)=1$ if and only if $F(x,u)\in\mathcal R_{\mathrm{SA}}$.

    The discounted action-value expression~\eqref{app_eq: save_discounted_action_value_inside_region} follows from~\eqref{app_eq: save_time_optimal_q_value}. Since $\beta\in(0,1)$, maximizing $\beta^{1+d(F(x,u))}$ over feasible successors is equivalent to minimizing $d(F(x,u))$. By the definition of $d(x)$, the minimum feasible successor distance is $d(x)-1$. Therefore $u$ is discounted-optimal if and only if~\eqref{app_eq: save_discounted_optimal_action_condition} holds.

    Finally, if $x\in\calH\setminus\mathcal R_{\mathrm{SA}}$, then $d(x)=\infty$. Hence no policy can safely arrive at $\B$ from $x$, so both the undiscounted and discounted optimal values are zero. This proves the final case.
\end{proof}

\begin{remark}[Roles of Undiscounted and Discounted \rafull]
    The undiscounted \ralow defines the safe-arrival specification and the feasible safe-arrival region
    \begin{align}
        \mathcal R_{\mathrm{SA}}
        =
        \{x\in\X:V_{\mathrm{SA}}^{\star}(x)=1\}.
    \end{align}
    The discounted \ralow preserves this region,
    \begin{align}
        \mathcal R_{\mathrm{SA}}
        =
        \{x\in\X:V_{\mathrm{SA},\beta}^{\star}(x)>0\},
        \qquad \beta\in(0,1),
    \end{align}
    and further ranks successful policies according to their arrival time. Thus the undiscounted value provides the feasibility semantics, while the discounted value provides a contraction-based, time-sensitive objective for RL training.
\end{remark}

\section{\oursRA{} Details and Proofs}\label{app: ours_details_and_proofs}

\subsection{Construction and Certification of the Base Pair}\label{app: base_set_cert}

We give sufficient conditions and a local Lyapunov construction for obtaining a base pair satisfying Theorem~\ref{thm: base_pair}.

\begin{assumption}[Local Stabilizability]\label{assume: local_stable}
There exists $x^\star\in\mathrm{int}(\calS)$ and $u^\star\in\mathrm{int}(\U)$ such that: (a) $(x^\star,u^\star)$ is an equilibrium of~\eqref{eq: affine_dyn}, i.e., $f(x^\star)+g(x^\star)u^\star=0$; and (b) the linearization $(A,B)$ of~\eqref{eq: affine_dyn} about $(x^\star, u^\star)$, with $A=\frac{\partial(f+gu^\star)}{\partial x}\big|_{x=x^\star}$ and $B=g(x^\star)$, is locally stabilizable.
\end{assumption}

Assumption~\ref{assume: local_stable} is standard in nonlinear control~\citep{khalil2002nonlinear} and is satisfied by any continuously differentiable system in a neighborhood of an equilibrium $(x^\star,u^\star)$ that admits first-order-controllable linearization. 

\begin{theorem}[Existence of a Certified Local Base Pair]\label{thm: base_set}
    Under Assumption~\ref{assume: local_stable}, there exists a linear feedback gain $K\in\R^{m\times n}$, a symmetric positive-definite matrix $P\succ0$, and a constant $\bar c>0$ such that the base controller $\pi_\B(x) = u^\star - K(x-x^\star)$ and the sublevel set $\B_c = \left\{ x\in\X: (x-x^\star)^\top P (x-x^\star)\le c\right\}, \, c\in(0,\bar c]$ 
    together satisfy the following properties: (i) $\B_c\subseteq\calS$; (ii) $\pi_\B(x)\in\U$ for every $x\in\B_c$; (iii) $\B_c$ is control-invariant under $\pi_\B$; and (iv) every trajectory with $x(0)\in\B_c$ satisfies $x(t)\to x^\star$ as $t\to\infty$ and $x(t)\in\B_c$ for all $t\ge 0$. 
\end{theorem}

\begin{proof}\label{proof: base_set}
    By Assumption~\ref{assume: local_stable}(b), there exists $K\in\R^{m\times n}$ such that $A-BK$ is Hurwitz. Fix any such $K$ and any $Q\succ0$. By the Lyapunov equation for stable linear systems~\citep[Thm 4.6]{khalil2002nonlinear}, there exists a unique $P\succ0$ with
    \begin{align}\label{eq: lyap_general}
        (A-BK)^\top P+ P(A-BK) = -Q.
    \end{align}
    Define the error coordinate $e:=x-x^\star$ and the candidate Lyapunov function as
    \begin{align}\notag%
        V(e) = e^\top Pe.
    \end{align}
    
    The closed-loop nonlinear dynamics under $\pi_\B$ are $\dot x = f(x) + g(x)(u^\star - K(x-x^\star))$. Setting $F(e) := f(x^\star + e) + g(x^\star+e)(u^\star - K e)$, a Taylor expansion of $F$ at $e=0$ combined with $F(0) = f(x^\star) + g(x^\star) u^\star = 0$ (Assumption~\ref{assume: local_stable}(a)) gives
    \begin{align}\label{eq: taylor}
        \dot e \;=\; F(e) \;=\; (A-BK)\,e \,+\, \rho(e),
        \qquad \|\rho(e)\| = o(\|e\|) \text{ as } e\to 0,
    \end{align}
    where the remainder $\rho$ inherits the smoothness of $f$ and $g$. 
    Using (\ref{eq: lyap_general}) and (\ref{eq: taylor}),
    \begin{align}\notag
        \dot V(e)
        \;=\; 2\,e^\top P\,\dot e
        \;=\; -\, e^\top Q\, e \;+\; \underbrace{2\,e^\top P\,\rho(e)}_{=:\,R(e)},
    \end{align}
    with $|R(e)| \le 2\|P\|\,\|e\|\,\|\rho(e)\| = o(\|e\|^2)$. Hence there exists $r_1>0$ such that for every $e$ with $0 < \|e\| \le r_1$,
    \begin{align}\notag%
        \dot V(e) \;\le\; -\tfrac{1}{2}\lambda_{\min}(Q)\,\|e\|^2 \;<\; 0.
    \end{align}
    Let $c_1 := \lambda_{\min}(P)\, r_1^2$. 
    Then $\{e : V(e)\le c_1\}\subseteq\{e:\|e\|\le r_1\}$ and $\dot V(e)<0$ on $\{0 < V(e)\le c_1\}$, establishing Theorem~\ref{thm: base_set}(iii) and (iv) for any $c\le c_1$ by the standard Lyapunov invariance theorem~\citep[Thm.~4.1]{khalil2002nonlinear}.
    
    It remains to shrink $c_1$ so that Theorem~\ref{thm: base_set}(i) and (ii) also hold. Since $x^\star\in\mathrm{int}(\calS)$ and $\calS = \{x : h_\calS(x)\ge 0\}$ is the zero-superlevel set of the continuous function $h_\calS$, there exists $c_2>0$ such that $\B_{c_2}\subseteq\calS$. Analogously, since $u^\star \in\mathrm{int}(\U)$, $\U$ is compact, and the map $\pi_\B(x)=u^\star - K(x-x^\star)$ is continuous, there exists $c_3>0$ such that $\pi_\B(\B_{c_3})\subseteq\U$. Taking
    \begin{align}\notag%
        \bar c \;:=\; \min\{c_1, c_2, c_3\} \;>\; 0
    \end{align}
    yields all four properties simultaneously for every $c\in(0,\bar c]$.
\end{proof}

\subsection{\oursRA ~Decomposition Proof}\label{app: decomposition_proof}

\certifythm*

\begin{proof}
    We first prove (i). Let $\pi\in\Pi_{\bcbf}(\lbpstar)$ and let
    $x_0\in\X_0$. By Subproblem~\ref{prob: sub1},
    $\X_0\subseteq\CT(\lbpstar)$, so $x_0\in\CT(\lbpstar)$. The set
    $\CT(\lbpstar)$ is the zero-superlevel set of the implicit BCBF
    \begin{align}\notag%
        h_{\CT}(x)
        =
        \min\left\{
            \min_{\td\in[0,T]}
            h_\calS(\Phi_{\lbpstar}(x,\td)),
            \;
            h_\B(\Phi_{\lbpstar}(x,T))
        \right\}.
    \end{align}
    By construction of $\Kb_{\bcbf}(x;\lbpstar)$, every action
    $u=\pi(x)$ with $x\in\CT(\lbpstar)$ satisfies the pointwise BCBF
    inequalities on the backup rollout and the terminal base-set condition.
    These pointwise inequalities are a sufficient condition for the CBF condition
    on $h_{\CT}$~\citep[Prop.~2]{backupCBF}. Therefore $\CT(\lbpstar)$ is
    forward invariant under any policy in $\Pi_{\bcbf}(\lbpstar)$. Since
    $\CT(\lbpstar)\subseteq\calS$, the resulting trajectory satisfies
    $x^\pi(t;x_0)\in\calS$ for all $t\ge0$. Moreover,
    $\Kb_{\bcbf}(x;\lbpstar)\subseteq\U$ by definition, so the executed input
    also satisfies the actuator constraint. Hence
    $\pi\in\Pi_{\safe}$, proving
    $\Pi_{\bcbf}(\lbpstar)\subseteq\Pi_{\safe}$.
    
    Statement (ii) follows directly from the definition of
    $\piphistar$ as a solution of Subproblem~\ref{prob: sub1}. Since
    Subproblem~\ref{prob: sub1} optimizes $J$ over
    $\Pi_{\bcbf}(\lbpstar)$,
    \[
        J(\piphistar)
        =
        \sup_{\pi\in\Pi_{\bcbf}(\lbpstar)}J(\pi).
    \]
    The inequality
    \[
        \sup_{\pi\in\Pi_{\bcbf}(\lbpstar)}J(\pi)
        \quad\le\quad
        \sup_{\pi\in\Pi_{\safe}}J(\pi)
    \]
    then follows from part (i).
    
    Finally, suppose there exists a globally optimal safe policy
    $\pi^\star_{\safe}\in\Pi_{\bcbf}(\lbpstar)$. Since
    $\piphistar$ maximizes $J$ over $\Pi_{\bcbf}(\lbpstar)$,
    \[
        J(\piphistar)
        \ge
        J(\pi^\star_{\safe}).
    \]
    But by part (i), $\piphistar\in\Pi_{\safe}$, and
    $\pi^\star_{\safe}$ is globally optimal over $\Pi_{\safe}$, so
    \[
        J(\piphistar)
        \le
        J(\pi^\star_{\safe}).
    \]
    Thus $J(\piphistar)=J(\pi^\star_{\safe})$, proving that
    $\piphistar$ is globally optimal for Problem~\ref{prob: main}.
\end{proof}

\subsection{\oursRA ~Safety Guarantee Proof}\label{app: ps2_safety_proof}

\safetythm*

\begin{proof}
    For every $x\in \CT(\lbpstar)$, the control-invariant layer returns the projection of $\pinom(x)$ onto $\Kb_\bcbf(x;\lbpstar)$. Hence
    \begin{equation}\notag
        \pi^\phi_{\mathrm{\ours}}(x;\lbpstar)\in \Kb_\bcbf(x;\lbpstar)
        \qquad
        \forall x\in \CT(\lbpstar),
    \end{equation}
    which is exactly the defining property of $\Pi_{\bcbf}(\lbpstar)$. Thus, the following holds:
    \begin{equation}\notag
        \pi^\phi_{\mathrm{\ours}} \in \Pi_{\bcbf}(\lbpstar).
    \end{equation}
    The conclusion then follows immediately from Theorem~\ref{thm: decomposition}. 
\end{proof}

\subsection{\oursRA ~Universal Approximation Proof}\label{app: ps2_uat}

\isrestatetrue
\uatthm*
\isrestatefalse

\begin{proof}
    For every $x\in\Omega\subseteq\CT(\lbpstar)$, the set $\Kb_\bcbf(x;\lbpstar)$ is nonempty by BCBF feasibility and convex because $\U$ is convex and the BCBF constraints are affine in $u$ under the fixed backup flow. 
    Thus~\eqref{eq: backup_layer} is precisely a \hardCVX{} projection with state-dependent convex feasible set $\Kb_\bcbf(x;\lbpstar)$. 
    The result follows directly from the universal-approximation theorem for \hardCVX~\citep[Thm.~12]{hardnet}.
\end{proof}

\subsection{Constructing the CIL from Differentiable Simulators}\label{app: ps2_scalability_argument}
As briefly discussed in Sec.~\ref{subsec: scalability}, \oursRA{} does not fundamentally require a hand-derived symbolic decomposition $f(x)+g(x)u$. 
Suppose instead that a differentiable simulator $F_{\dt}$ is available. For a fixed backup policy, define the candidate-control-input rollout
\begin{align}\notag%
    z_0(u)=F_{\dt}(x,u),\qquad
    z_{i+1}(u)=F_{\dt}\bigl(z_i(u),\lbpstar(z_i(u))\bigr).
\end{align}
For each backup rollout condition $\ell_i(z_i(u))\ge 0$, corresponding to either a sampled safe-set condition or the terminal base-set condition, automatic differentiation gives
\begin{align}\notag%
    a_i(x)=\nabla_u \ell_i(z_i(u))\big|_{u=u_b},
    \qquad u_b=\lbpstar(x),
\end{align}
and hence the local affine row
\begin{align}\notag%
    \ell_i(z_i(u_b)) + a_i(x)^\top (u-u_b) \ge 0.
\end{align}
These simulator-generated rows can be inserted into the same control-invariant layer. 
For high-dimensional simulators, one need not materialize the full state sensitivity matrix, as vector--Jacobian products or reduced sensitivities for only the constraint-relevant outputs are sufficient.

This simulator-based variant should be interpreted as a discrete-time, locally affine realization of \oursRA{}. 
With an exact control-affine model, the rows in Eq.~\eqref{eq: bcbf_set} are exact and the guarantees in Theorem~\ref{thm: ps2_safety} apply directly. 
With a general differentiable simulator, the local rows are first-order approximations, and in that case one can add conservative margins or impose a trust region around $u_b$. 
Thus, the main scaling bottleneck is not closed-form dynamics or explicit invariant-set synthesis, but access to a differentiable rollout model whose accuracy is sufficient over the backup horizon.
Employing \oursRA{} with a differentiable simulator, without the analytical dynamics, is left as future work.


\begin{figure}[ht!]
    \centering
    \includegraphics[width=0.8\linewidth]{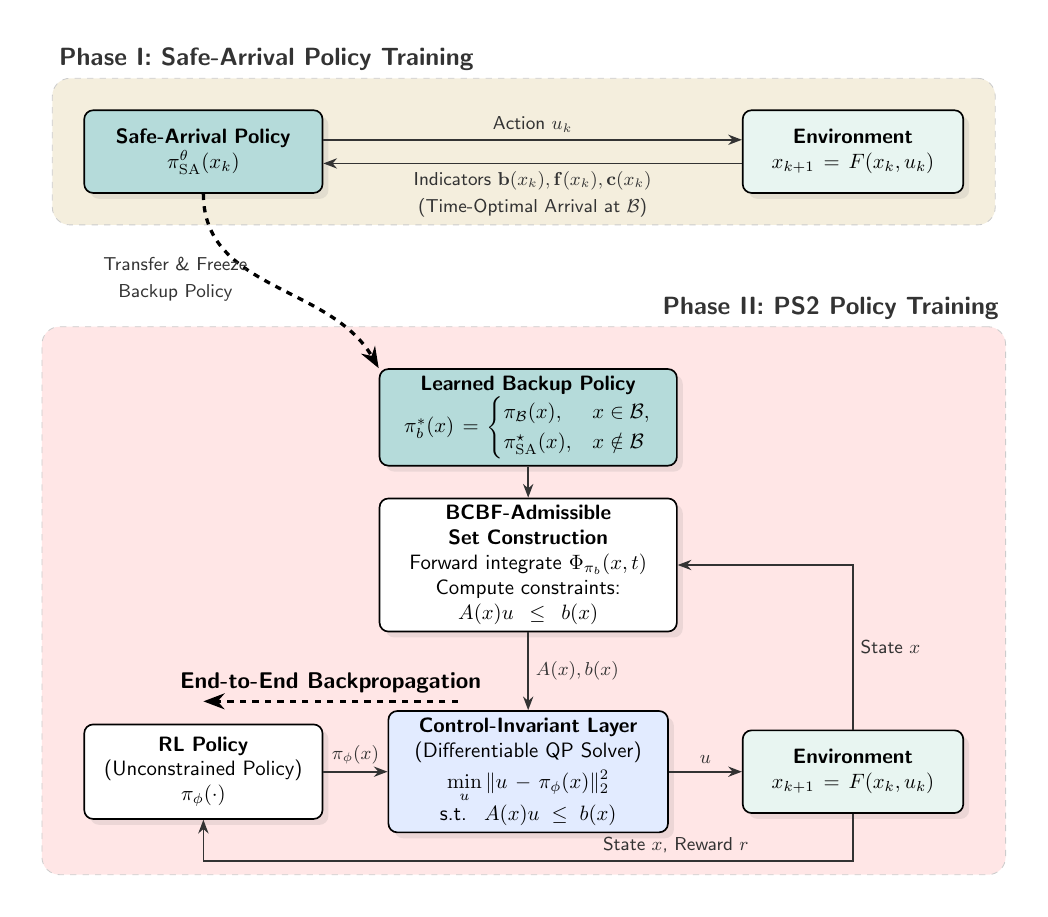}
    \caption{Block-diagram view of the two-phase \oursRA{} framework, complementing Fig.~\ref{fig: framework_hero}. 
    \textbf{Phase I (top)}: the safe-arrival policy $\piarrL$ is trained against the environment using the indicators $\vb b(x_k),\vb f(x_k), \vb c(x_k)$ of $\B,\F$, and the continuation set, optimizing the discounted \ralow{} Bellman recursion.
    \textbf{Phase II (bottom)}: the composed backup policy $\pi_b^\star$ is frozen and used to construct the BCBF-admissible set online for each timestep. The control value from the RL policy $\pinom$ is projected onto this set by the control-invariant layer (a differentiable QP), and task-reward gradients propagate end-to-end through the projection back into $\pinom$.
    }
    \label{fig: framework_block}
\end{figure}

\subsection{Phase I: Safe-Arrival Policy Training}\label{app: save_beta}

Here, we expand on Phase~I beyond the high-level description in Sec.~\ref{sec: phase1}. 
Phase~I addresses Subproblem~\ref{prob: sub2} by learning a parameterized safe-arrival policy $\piarrL$ that enlarges the backup-induced control-invariant set $\CT(\lbp)$,  given the certified base pair $(\B,\pi_\B)$ from Theorem~\ref{thm: base_set}.
Fig.~\ref{fig: framework_block} (top) summarizes the resulting training loop: the policy interacts with the environment, stepwise indicators of $\B$, $\F$, and the continuation set are recorded along each rollout, and the parameters are updated with the discounted \ralow Bellman recursion~\eqref{eq: save_bellman}.


\begin{algorithm}[ht]
\caption{Safe-Arrival Policy Training}
\label{alg: save_training}
\begin{algorithmic}[1]
\Require sampler $\rho_{\mathrm{arr}}$ on $\Omega\setminus\B$, dynamics $F$, sets $\B$ and $\calS$, discount $\beta$
\State Initialize safe-arrival policy parameters $\theta$ and value/Q-function parameters
\While{not converged}
    \State Sample initial states $x_0\sim\rho_{\mathrm{arr}}$
    \State Collect rollouts under $\piarrL$ until $x_k\in\B$ or $x_k\in\F=\X\setminus\calS,$
    \State For each $x_k$, assign indicators $\vb b(x_k),\, \vb f(x_k),$ and $\vb c(x_k)=1-\vb b(x_k) - \vb f(x_k)$
    \State Update the chosen value/Q estimator using the safe-arrival value Bellman recursion~\eqref{eq: save_bellman}
    \State Improve $\theta$ with the chosen RL backbone
\EndWhile
\State \Return $\piarrL$ and freeze the composed $\lbp$ for Phase II
\end{algorithmic}
\end{algorithm}

The Phase I objective can be written as
\begin{equation} \notag
    \max_{\theta}\; J_{\mathrm{SA},\beta}(\theta)
    :=
    \mathbb E_{x\sim \rho_{\mathrm{arr}}}\!\left[Q_{\mathrm{SA},\beta}^{\piarrL}(x,\piarrL(x))\right].
\end{equation}

Algorithm~\ref{alg: save_training} summarizes Phase I at the framework level. 
In practice, one can take $\Omega=\calS$, or a smaller design region of interest inside $\calS$, and choose $\rho_{\mathrm{arr}}$ to focus training on the portion of state space from which safe arrival is desired. 
Details on $\rho_{\mathrm{arr}}$ for each experiment are provided in App.~\ref{app: ps2_implementation_details}.
The only safe-arrival policy-specific ingredients are the indicators $\vb b,\vb f,\vb c$, first-hit termination at $\B$ or $\F$, and the objective function with the safe-arrival value. 
Everything else can be supplied by a chosen RL backbone, such as actor-critic, Q-learning, fitted value iteration, and related methods. 

\subsection{Phase II: Control-Invariant Layer and \ours~Policy}\label{app: backup_layer}
This appendix expands the control-invariant layer construction and \oursRA~policy training in Sec.~\ref{sec: main_method}, which corresponds to the bottom panel of Fig.~\ref{fig: framework_block}.
The main text defines the control-invariant layer as the projection of a nominal policy output onto the BCBF-admissible set $\Kb_{\bcbf}(x;\lbpstar)$. 
Here, we write the finite-mesh QP explicitly in the standard affine form and describe the differentiable solver used in implementation.

\subsubsection{Control-Invariant Layer Construction and Solver}\label{app: backup_layer_qp}

\paragraph{Construction of the BCBF-admissible input set.}
We allow the safe set and base set to be defined by multiple differentiable inequalities,
\begin{align}
    \calS &= \{x\in\X: h_{\calS,j}(x)\ge 0,\; j=1,\ldots,n_{\calS}\}, \notag\\
    \B &= \{x\in\X: h_{\B,\ell}(x)\ge 0,\; \ell=1,\ldots,n_{\B}\}. \notag
\end{align}
Let the finite backup mesh be $0=t'_0<t'_1<\cdots<t'_N=T$. 
For fixed $x$ and $\lbpstar$, define
\begin{align} \notag
    x_i^b(x) &:= \Phi_{\lbpstar}(x,t'_i), &
    \Psi_i(x) &:= \Psi_{\lbpstar}(x,t'_i)
    = \frac{\partial \Phi_{\lbpstar}(x,t'_i)}{\partial x}, &
    f_i^b(x) &:= f_{\lbpstar}\!\left(x_i^b(x)\right).
\end{align}
The current-state control-affine terms are denoted by $f_0:=f(x)$ and $g_0:=g(x)$.
Then, for every safety constraint component $h_{\calS,j}$ and backup node $t'_i$, the BCBF rollout condition is
\begin{align}\notag
    \nabla h_{\calS,j}(x_i^b)^\top
    \Big[
        \Psi_i(x)\big(f_0+g_0u\big)-f_i^b(x)
    \Big]
    +
    \alpha_{\calS}\!\left(h_{\calS,j}(x_i^b)\right)
    \ge 0.
\end{align}
Equivalently, this is an affine inequality in the current input $u$:
\begin{align}
    a^{\calS}_{i,j}(x)^\top u &\le b^{\calS}_{i,j}(x),  \notag
    \\
    a^{\calS}_{i,j}(x)^\top
    &:=
    -\nabla h_{\calS,j}(x_i^b)^\top \Psi_i(x) g_0, \notag
    \\
    b^{\calS}_{i,j}(x)
    &:=
    \alpha_{\calS}\!\left(h_{\calS,j}(x_i^b)\right)
    +
    \nabla h_{\calS,j}(x_i^b)^\top
    \Big[
        \Psi_i(x)f_0 - f_i^b(x)
    \Big]. \notag
\end{align}

At the terminal node $T=t'_N$, each base-set constraint component gives
\begin{align} \notag
    \nabla h_{\B,\ell}(x_N^b)^\top
    \Psi_N(x)\big(f_0+g_0u\big)
    +
    \alpha_{\B}\!\left(h_{\B,\ell}(x_N^b)\right)
    \ge 0,
\end{align}
or equivalently
\begin{align}
    a^{\B}_{\ell}(x)^\top u &\le b^{\B}_{\ell}(x), \notag\\
    a^{\B}_{\ell}(x)^\top
    &:=
    -\nabla h_{\B,\ell}(x_N^b)^\top \Psi_N(x) g_0, \notag
    \\
    b^{\B}_{\ell}(x)
    &:=
    \alpha_{\B}\!\left(h_{\B,\ell}(x_N^b)\right)
    +
    \nabla h_{\B,\ell}(x_N^b)^\top \Psi_N(x) f_0.  \notag
\end{align}

Stacking all safety rows and terminal rows gives
\begin{align} \notag
    A_{\bcbf}(x)
    &:=
    \begin{bmatrix}
        A_{\calS}(x)\\
        A_{\B}(x)
    \end{bmatrix}
    \in \R^{n_c\times m},
    &
    b_{\bcbf}(x)
    &:=
    \begin{bmatrix}
        b_{\calS}(x)\\
        b_{\B}(x)
    \end{bmatrix}
    \in \R^{n_c},
\end{align}
where $n_c=(N+1)n_{\calS}+n_{\B}$. 
Thus the finite-mesh BCBF constraints can be written compactly as
\begin{align} \notag
    A_{\bcbf}(x)u \le b_{\bcbf}(x).
\end{align}
For a boxed input limit set
\begin{align} \notag
    \U = \{u\in\R^m: u_{\min}\le u\le u_{\max}\},
\end{align}
the actuator constraints are
\begin{align} \notag
    A_{\U}u\le b_{\U},
    \qquad
    A_{\U}:=
    \begin{bmatrix}
        I_m\\
        -I_m
    \end{bmatrix},
    \qquad
    b_{\U}:=
    \begin{bmatrix}
        u_{\max}\\
        -u_{\min}
    \end{bmatrix}.
\end{align}
Then, the control-invariant layer is a quadratic program:
\begin{align}
    \label{eq: app_zero_slack_backup_layer}
    \begin{split}
    &\calP_{\BL}(\pinom)(x)
    =
    \argmin_{u}
    \|u-\pinom(x)\|_2^{2}
    \\
    &\mathrm{s.t.}\quad
    A_{\bcbf}(x)u\le b_{\bcbf}(x),\;\;
    A_{\U}u\le b_{\U}.
    \end{split}
\end{align}

\paragraph{Differentiable QP solver.}
The optimization problem~\eqref{eq: app_zero_slack_backup_layer} 
is a convex QP by construction. 
Our implementation solves it with \texttt{qpax (v0.0.9)}~\citep{qpax, qpax_jon}. 
All quantities used to assemble $A_{\bcbf}(x)$ and $b_{\bcbf}(x)$ are computed with differentiable JAX operations, including the backup rollout, sensitivity propagation, row construction, and the QP solve. 
Consequently, the policy update backpropagates directly through the network without an estimator or a separate surrogate loss. 
The differentiable solver provides the corresponding gradient used by the RL optimizer. The \texttt{qpax} solver settings are reported in Table~\ref{tab: qpax_settings}.

\begin{table}[H]
    \centering
    \caption{\texttt{qpax} settings, with the predictor--corrector interior-point solver.}
    \label{tab: qpax_settings}
    \resizebox{\linewidth}{!}{%
    \begin{tabular}{lccccc}
        \toprule
        & \multicolumn{2}{c}{\textbf{Unicycle}}
        & \multicolumn{2}{c}{\textbf{Quadrotor (1 constraint)}}
        & \makecell{\textbf{Quadrotor}\\\textbf{(3 constraints)}} \\
        \cmidrule(lr){2-3} \cmidrule(lr){4-5} \cmidrule(lr){6-6}
        \textbf{Setting}
        & \textbf{\textcolor{oursRA}{\oursSAV}} & \textbf{\textcolor{oursABP}{\oursABP}} & \textbf{\textcolor{oursRA}{\oursSAV}} & \textbf{\textcolor{oursABP}{\oursABP}} & \textbf{\textcolor{oursRA}{\oursSAV}} \\
        \midrule
        Control variables
        & 2 & 2 & 4 & 4 & 4 \\
        QP inequalities
        & 89 & 89 & 110 & 110 & 312 \\
        Solver tolerance
        & $10^{-4}$ & $10^{-4}$
        & $5\times10^{-4}$ & $5\times10^{-4}$ & $5\times10^{-4}$ \\
        Backward complementarity target $\kappa$
        & $10^{-2}$ & $10^{-2}$ & $10^{-2}$ & $10^{-2}$ & $10^{-2}$ \\
        Iteration caps (forward / backward)
        & 30 / 30 & 30 / 30 & 30 / 30 & 30 / 30 & 30 / 30 \\
        \bottomrule
    \end{tabular}%
    }
\end{table}

\subsubsection{\oursRA{} Policy Training through the Control-Invariant Layer}\label{app: ps2_policy_training_backup_layer}

Phase II freezes the composed backup policy
\begin{align}\notag%
    \lbpstar(x)=
    \begin{cases}
        \pi_{\B}(x), & x\in\B,\\
        \pi_{\mathrm{SA}}^\star(x), & x\notin\B,
    \end{cases}
\end{align}
and trains only the nominal task policy $\pinom$. 
For a deterministic actor, the executed \ours\ action is
\begin{align}\notag%
    u_k
    =
    \pi^{\phi}_{\mathrm{\ours}}(x_k;\lbpstar)
    :=
    \calP_{\BL}(\pinom)(x_k).
    \label{eq: app_ps2_policy_deterministic}
\end{align}
For a stochastic actor, the nominal action is first sampled from the policy distribution and the control-invariant layer is then applied deterministically:
\begin{align}\notag%
    \unom_k \sim \pinom(\cdot\mid x_k),
    \qquad
    u_k=\calP_{\BL}(\unom_k)(x_k). 
\end{align}
The environment and replay buffer therefore only see projected controls $u_k$. 
In our implementation, the actor loss evaluates the critic at the projected action $u^\star$, and target-policy actions in Bellman backups are projected through the same control-invariant layer. 
Thus the actor is optimized for task return inside the BCBF-admissible action set.

Because the control-invariant layer is differentiable, the actor-gradient path contains the QP solution.
This is the key distinction between \oursRA{} and a post-hoc non-differentiable shield: the projection layer is present during both data collection and gradient-based policy improvement. 
The resulting policy is therefore trained to produce high-return nominal actions whose projections remain close to the nominal command while satisfying the BCBF-induced constraints.

\section{Experiment Details}\label{app: exp_details}

\subsection{Unicycle Environment Details}\label{app: unicycle_experiment_details}

\paragraph{System dynamics.} We consider the following unicycle dynamics for the lane-keeping experiment:
\begin{equation}\label{eq: unicycle_dynamics}
    \begin{bmatrix} \dot{y} \\ \dot{v} \\ \dot{\psi} \end{bmatrix} = 
    \begin{bmatrix} v \sin(\psi) \\ 0 \\ 0 \end{bmatrix} + 
    \begin{bmatrix} 0 & 0 \\ 1 & 0 \\ 0 & 1 \end{bmatrix} 
    \begin{bmatrix} a_\cmdtext \\ r_\cmdtext \end{bmatrix}
\end{equation}
where the state $x$ comprises the lateral position $y$, velocity $v$, and the heading angle $\psi$, with control inputs consisting of acceleration $a_\cmdtext\in[-5,5]$~m/s$^2$ and yaw rate $r_\cmdtext\in[-1,1]$~rad/s. 
We implement the continuous-time dynamics and controller in sampled-data form at 20 Hz.

\paragraph{Safe set.}
The safe set is specified as $\calS = \{x: |y|\le y_{\max}, |\psi|\le \psi_{\max}\}$, where $y_{\max}=1.8$~m and $\psi_{\max}=\pi/3$. 
Thus, we have four safety constraints: 
\begin{align}\notag
    \begin{split}
    h_{\calS,1}(x) = y_{\max} + y, \quad \quad
    h_{\calS,2}(x) &= y_{\max} - y, \\
    h_{\calS,3}(x) = \psi_{\max}+\psi, \quad \quad
    h_{\calS,4}(x) &= \psi_{\max}-\psi.
    \end{split}
\end{align}
Note that the safety specification constrains only the lateral position and heading, since these are the variables directly associated with leaving the lane.

\paragraph{Reference trajectory, reward, episode setup.} 
The unicycle task is to track a sinusoidal reference:
\begin{align}\notag%
    y_{\reftext}(t) = 2.5\sin\!\left(\frac{2\pi t}{10}\right),
    \qquad
    v_{\reftext}(t)=5,
    \qquad
    \psi_{\reftext}(t)=\arcsin\!\left(\frac{\dot y_{\reftext}(t)}{v_{\reftext}(t)}\right),
\end{align}
which lasts 20 sec and deliberately violates the lane constraint, since its lateral amplitude $2.5$~m exceeds $y_{\max}=1.8$~m. The trajectory-following reward uses the normalized tracking error for step $k$
\begin{align}\notag%
  \xi_k =
  \left[
    \frac{y_k-y_{\reftext,k}}{y_{\max}},\;\;
    \frac{v_k-v_{\reftext,k}}{5},\;\;
    \frac{\mathrm{wrap}(\psi_k-\psi_{\reftext,k})}{\psi_{\max}},\;\;
    \frac{a_{\cmdtext,k}}{a_{\max}},\;\;
    \frac{r_{\cmdtext,k}}{r_{\max}}
  \right],
\end{align}
and the step-wise reward is $r_k=-\xi_k^\top L\xi_k$ with $L=\diag(50,20,10,0.05,0.05)$. 
The reward weights are selected purely for nominal sine-wave tracking performance, without safety in mind. 
Episodes have a fixed horizon of $400$ steps, matching the $20$-second sinusoidal reference.

\subsection{Quadrotor Environment Details}\label{app: quadrotor_experiment_details}

\paragraph{System dynamics.}
The quadrotor dynamics for the powerloop tracking experiment are
\begin{align}\label{eq: quadrotor_dynamics}
    \dot{\mathbf{p}} = \mathbf{v}, \quad
    \dot{\mathbf{v}} =  - g\mathbf{e}_3 + R(\mathbf{q}) (\acmd\mathbf{e}_3)&, \quad
    \dot{\mathbf{q}} = \frac{1}{2} \Xi(\mathbf{q}) \wcmd,
\end{align}
where $\Xi(\mathbf{q})$ is the quaternion kinematic matrix. 
The state $x=[\mathbf{p}^\top, \mathbf{v}^\top, \mathbf{q}^\top]^\top  \in \mathbb R^{10}$ consists of the inertial position $\mathbf{p} = [p_x, p_y, p_z]^\top$, the inertial linear velocity $\mathbf{v} = [v_x, v_y, v_z]^\top$, and the unit quaternion $\mathbf{q}=[q_w, q_x, q_y, q_z]^\top$ is the global, inertial orientation. 
The control input $u=[\acmd,\wcmd^\top]^\top$ consists of the mass-normalized thrust, $\acmd\in[0,4g]$, where $g=9.81m/s^2$, and the body rate, $\wcmd=[\omega_x,\omega_y,\omega_z]^\top,$ where $\omega_i\in[-18,18]\text{rad/s},i\in\{x,y,z\}$, is defined with respect to the quadrotor's body frame. 
Here, $\dt = 0.02$ sec, with control frequency also at 50 Hz.

\paragraph{Safe set.}
The safe set for the single-constraint quadrotor task is defined as $\calS = \{ x: p_z\le z_{\text{ceil}} \}$, where $z_{\text{ceil}} = 3$m represents a hard ceiling. 
There is no safety constraint for the ground.
Note that $h_\calS(x)=z_{\mathrm{ceil}}-p_z$ has relative degree 2 for thrust, but the relative degree for bodyrate inputs is higher. 

\paragraph{Reference trajectory, reward, episode setup.} 
The quadrotor tracks a powerloop reference inspired by~\citep{deepdroneacrobatics}, as shown in Fig.~\ref{fig: quadrotor_reference}. 
The reference follows a vertical circular loop of radius 1.5m centered at $[0,0,2]^\top$m, starting at the bottom of the loop with tangential speed 4.5m/s. 
The resulting reference completes one full loop in approximately 2.1 sec.
It is intentionally unsafe, as the loop apex exceeds the ceiling $z_{\mathrm{ceil}}=3$m. 
The reference tangential speed exceeds the free-fall threshold $\varepsilon\sqrt{1.5g}$ ($\varepsilon=1.1$)~\citep{deepdroneacrobatics}. 
The attitude reference simultaneously commands an aggressive $360^\circ$ flip, forcing the controller to trade off near-ceiling translational tracking, agile attitude-rate tracking, and hard safety enforcement at the most dynamically constrained portion of the maneuver.

At step $k$, the reward is the negative weighted tracking cost
\begin{align}
  r_k =& -\Big(
      w_{p,\mathrm{xy}}\|\vb p_{x,y}-\vb p_{\reftext,x,y}\|_2^2
      + w_{p,z}(p_z-p_{\reftext,z})^2
      + w_v\|\vb v-\vb v_{\reftext}\|_2^2 \notag\\
      &+ w_{\mathrm{att}}\|\vb e_{\mathrm{att}}\|_2^2
      + \|\vb\omega_\cmdtext-\vb\omega_{\reftext}\|_{W_{\vb\omega}}^2
      + w_a a_\cmdtext^2
      + w_\Omega\|\vb\omega_\cmdtext\|_2^2
    \Big),\notag
\end{align}
where
$\vb e_{\mathrm{att}}=\operatorname{sgn}(q_{e,w})\,\vb q_{e,\mathrm{xyz}}$ and
$\vb q_e=\vb q_{\reftext}\otimes\vb q^\star$.
The weights are
\begin{align}
  w_{p,\mathrm{xy}} &= 2.5, & w_{p,z} &= 2.0, & w_v &= 4.0, & w_{\mathrm{att}} &= 16.0, \notag
  \\
  W_{\vb\omega} &= \diag(0.10,0.20,0.05),& w_a &= 0.01, & w_\Omega &= 0.01. \notag
\end{align}
As in the unicycle experiment, these weights were selected by tuning an unconstrained tracking policy with no safety mechanism, and purely for nominal powerloop tracking of the aggressive position and attitude references shown in Fig.~\ref{fig: quadrotor_reference}.
The environment uses an episode horizon of 106 steps, matching the powerloop reference.

\begin{figure}[ht!]
    \centering
    \includegraphics[width=0.5\linewidth]{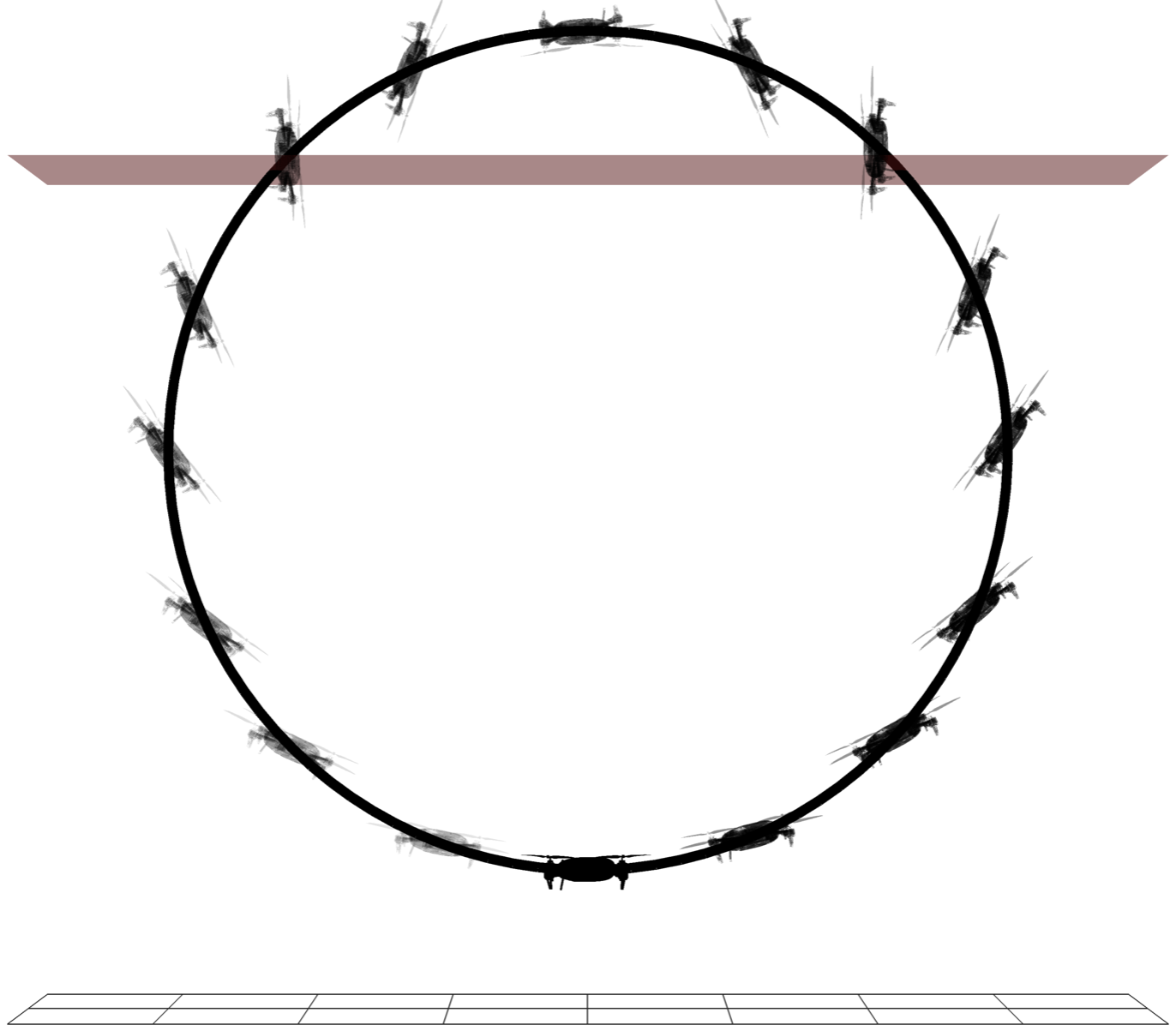}
    \caption{Quadrotor powerloop reference trajectory. Note that the powerloop includes an aggressive attitude reference that commands a full $360^\circ$ flip at the apex.}
    \label{fig: quadrotor_reference}
\end{figure}

\subsection{\oursRA~Implementation Details}\label{app: ps2_implementation_details}

\paragraph{Training backbones.}
Phase~I uses a TD3-style backbone~\citep{td3} with a deterministic actor, twin critics, and Polyak-averaged target networks.
The critics minimize the Huber loss for the safe-arrival Bellman recursion~\eqref{eq: save_bellman}, and episodes terminate upon first entry into $\B$ or $\F$.
Phase~II uses the SAC + CIL procedure in App.~\ref{app: ps2_policy_training_backup_layer}.
Its actor is a tanh-squashed Gaussian, and evaluation uses the deterministic mean followed by the CIL.
Tables~\ref{tab: app_phase1_hparams}--\ref{tab: app_phase2_hparams} report the hyperparameters.

\begin{table}[htpb]
    \centering
    \caption{\oursRA{} Phase I (safe-arrival policy) hyperparameters.}
    \label{tab: app_phase1_hparams}
    {\scriptsize
    \begin{tabular}{lcc}
        \toprule
        \textbf{Quantity} & \textbf{Unicycle} & \textbf{Quadrotor}\\
        \midrule
        Total environment steps & $3\times 10^6$ & $5\times 10^6$\\
        Update schedule & 1 grad. step / 8 env. steps & 1 grad. step / 8 env. steps\\
        Replay buffer / minibatch size & $4\times 10^5$ / $128$ & $4\times 10^5$ / $128$\\
        Actor / critic architecture & MLP, hidden $[128,128]$ & MLP, hidden $[128,128]$\\
        Actor / critic LR (Adam) & $3\times 10^{-4}$ / $3\times 10^{-4}$ & $10^{-4}$ / $3\times 10^{-4}$\\
        Polyak averaging coefficient $\tau$ & $0.0025$ & $0.0025$\\
        Policy update delay & every $2$ critic updates & every $2$ critic updates\\
        Twin critic target & $\min(Q^{1}_{\mathrm{SA},\beta},Q^{2}_{\mathrm{SA},\beta})$ & $\min(Q^{1}_{\mathrm{SA},\beta},Q^{2}_{\mathrm{SA},\beta})$\\
        Critic loss & Huber, $\delta=1.0$ & Huber, $\delta=1.0$\\
        Exploration noise / clip & $0.08$ / $0.10$ & $0.10$ / $0.10$\\
        Safe-arrival discount $\beta$ & $0.92$ & $0.99$\\
        $\dt$ & 0.05 & 0.02\\
        Backup horizon $T$ & 1.0 sec & 2.0 sec\\
        Curriculum $(s_0,\Delta s)$ & $(0.2,\ 0.005)$ & $(0.0,\ 0.10)$\\
        Curriculum success window / threshold & $50$ ep. / $0.9$ & $50$ ep. / $0.8$\\
        Curriculum min episodes between updates & $50$ & $100$\\
        Curriculum mechanism & isotropic radius scaling & region-mixture and per-region radius\\
        \bottomrule
    \end{tabular}
    }
\end{table}
 
\begin{table}[htpb]
    \centering
    \caption{\oursRA{} Phase II (\ours{} policy) hyperparameters.}
    \label{tab: app_phase2_hparams}
    {\scriptsize
    \begin{tabular}{lcc}
        \toprule
        \textbf{Quantity} & \textbf{Unicycle} & \textbf{Quadrotor}\\
        \midrule
        Total environment steps & $10^6$ & $1.5\times 10^6$\\
        Update schedule & 1 grad. step / 8 env. steps & 1 grad. step / 8 env. steps\\
        Replay buffer / minibatch size & $3\times 10^5$ / $64$ & $3\times 10^5$ / $64$\\
        Actor / critic architecture & MLP, hidden $[128,128]$ & MLP, hidden $[256,256]$\\
        Actor LR / critic LR / temp. LR (Adam) & $10^{-4}$ / $3\times 10^{-4}$ / $10^{-4}$ & $5\times 10^{-5}$ / $10^{-4}$ / $5\times 10^{-5}$\\
        Discount $\gamma$ / Polyak coefficient $\tau$ & $0.99$ / $0.005$ & $0.99$ / $0.005$\\
        Entropy temperature & $\alpha_0=0.2,\ \alpha_{\min}=0.01,\ \bar H=-2$ & $\alpha_0=0.1,\ \alpha_{\min}=0.01,\ \bar H=-4$\\
        Gradient norm clip / $Q$ clipping & $5.0$ / $5\times 10^6$ & $5.0$ / $5\times 10^6$\\
        Project actor actions in actor loss & yes & yes\\
        Project target-actor actions in critic backup & yes & yes\\
        Class-$\mathcal K_\infty$ gains $(\alpha_{\calS},\alpha_{\B})$ & $(4.0,\ 2.0)$ & $(4.0,\ 2.0)$\\
        Differentiable QP solver & \texttt{qpax}~\citep{qpax} & \texttt{qpax}~\citep{qpax}\\
        Episode horizon & $400$ steps ($20$ sec) & $106$ steps ($\approx 2.1$ sec)\\
        Warm start & none & vanilla powerloop-tracking SAC\\
        \bottomrule
    \end{tabular}
    }
\end{table}

\subsubsection{Unicycle Implementation Details}

\paragraph{Base set and controller.}\label{app: unicycle_base_set_controller}
Let $x^\star=[0,\ v_{\des},\ 0]^\top$ be the cruising equilibrium, with $v_{\des}=5$ and $u^\star=\vb 0_2$, and define the local error state $e(x) = [\,y,\; v-v_{\des},\; \psi\,]^\top$.
We linearize~\eqref{eq: unicycle_dynamics} about $(x^\star,u^\star)$ and discretize it to retrieve $(A_d,B_d)$. 
We then solve the discrete-time LQR, with $Q_d=\diag(1,1,1)$ and $R_d=\diag(0.01,0.5)$. Let $P\succ0$ be the positive-definite solution of the discrete algebraic Riccati equation and
\begin{align}\label{app_eq: riccati_eqn}
    K=(R_d+B_d^\top P B_d)^{-1}B_d^\top P A_d.
\end{align}
Then, the base controller is the LQR feedback
\begin{align}\label{app_eq: unicycle_base_controller}
    \pi_\B(x)=-K e(x),
\end{align}
where we clip the resulting control input to be within $\U$. 
The base set is the ellipsoid
\begin{align}  \notag
    \B=\{x\in\X: e(x)^\top P e(x)\le c_\B\},
    \qquad
    h_\B(x)=c_\B-e(x)^\top P e(x),
\end{align}
with $c_\B=0.3$. 
The actuator-admissible upper bound for the LQR ellipsoid is
\begin{align}\notag
    \bar c
    = \min\left\{
        \frac{a_{\max}^2}{K_aP^{-1}K_a^\top},
        \frac{r_{\max}^2}{K_rP^{-1}K_r^\top}
    \right\}
    \approx 1.05,
\end{align}
where $K_a$ and $K_r$ are the acceleration and yaw-rate rows of $K$, respectively. 
Since $c_\B=0.3<\bar c\approx1.05$, control-input clipping is inactive in $\B$.
The ellipsoid satisfies $|y|\le0.209$ and $|\psi|\le0.110$, strictly within the safety bounds.

\paragraph{Phase~I design region and reference measure.}
We instantiate the design region as the following box: $\Omega = \{x\in\X: |y|\le 1.8,\, v\in[0,\,12],\, |\psi|\le \frac{\pi}{3} \}$.
The reference measure underlying $\mu_{\mathrm{SA}}$ is taken to be the uniform measure on $\Omega\setminus\B$, which we approximate at evaluation time by a uniform $201\times121\times201$ grid over $(y,v,\psi)$. 
We count the grid points that safely arrive at $\B$ within horizon $T=1.0$ sec under $\piarrL$, and divide by the total number of grid points in $\Omega\setminus\B$. 

\paragraph{Phase~I curriculum and initial-state sampling.}
Initial states are drawn from a curriculum-controlled distribution $\rho_{\mathrm{arr}}$ that progressively expands during training.
Concretely, we sample
\begin{align}\notag
    y \sim U(-y_s, y_s),\quad
    v \sim v_{\des}+ U(-v_s, v_s),\quad
    \psi\sim U(-\psi_s, \psi_s),
\end{align}
where the radii interpolate linearly with a curriculum scale $s\in[0,1]$, and samples inside $\B$ are rejected.
The scale starts at $0.2$ and increases by $0.005$ when the success rate over the last 50 episodes exceeds $90\%$, with at least 50 episodes between updates.
Note that curriculum learning is an optimization heuristic we employ, not a required component of safe-arrival policy training.

\paragraph{Phase~II initial state sampling.}
Candidate initial states are uniformly sampled and accepted only if their rollouts under $\lbpstar$ satisfy the conditions in~\eqref{eq: bcbf_condition}, thereby ensuring $\X_0\subseteq\CT(\pi_b^\star)$.

\subsubsection{Quadrotor Implementation Details}\label{app: quadrotor_phase_I}

\paragraph{Base set and controller.}\label{app: quadrotor_base_set_controller}
For the base set and controller construction for the quadrotor experiment, we use the reduced hover-error state $x_e = [\,p_z-z_\des,\; v_x,\; v_y,\; v_z,\; 2q_{\err,x},\; 2q_{\err,y},\; 2q_{\err,z}\,]^\top \in \R^7,$ where $2q_{\err,\{x,y,z\}}$ are first-order rotation-angle error coordinates obtained from the sign-corrected quaternion error. 
The hover equilibrium is $x_e^\star=\vb 0_7$ with $z_\des=2$ and $u^\star=[g,0,0,0]^\top$. 
In these coordinates, we compute the reduced-state discrete-time linearization about hover to retrieve $(A_d,B_d)$ for $\dt=0.02$. 
We then solve the discrete-time LQR problem with $Q_d = \diag(1.0,0.16,0.16,0.4,0.8,0.8,0.16)$ and $R_d=\diag(0.02,0.012,0.012,0.004)$.
Let $P\succ0$ be the discrete Riccati solution from~\eqref{app_eq: riccati_eqn} and define the base controller as the hover LQR feedback
\begin{align}\notag%
    \pi_\B(x) = u^\star-Kx_e
\end{align}
where we clip the input to be within $\U$. 
The base set is the ellipsoid
\begin{align}\notag
    \B=\{x\in\X:x_e^\top P x_e\le c_\B\},
    \qquad
    h_\B(x)=c_\B-x_e^\top P x_e,
\end{align}
with $c_\B=8.0$. 
The actuator-admissible upper bound associated with the unclipped LQR action is
\begin{align}\notag
    \bar c
    = \min_{i\in\{1,\dots,4\}}
    \frac{\bar u_i^2}{K_iP^{-1}K_i^\top}
    \approx 12.38,
\end{align}
where $\bar u_i$ denotes the one-sided actuator margin around $u^\star$ for control channel $i$. 
Since $c_\B=8<\bar c\approx12.38$, clipping is inactive on $\B$.
Its altitude radius is $0.490$~m, giving $p_z\le2.490<z_{\mathrm{ceil}}$.

\paragraph{Phase~I design region and reference measure.}
The quadrotor state space is $10$-dimensional, so a uniform grid over $\calS$ is computationally intractable.
We therefore anchor the design region to the powerloop reference trajectory itself.
Concretely, we collect rollout traces from a vanilla SAC powerloop tracker trained without any safety mechanism, and define the design region $\Omega$ as the union of perturbation balls around the states visited by these traces.
We construct four sampling pools: vanilla SAC trace states that are safe and away from the ceiling and $\B$ (\emph{general trace}); trace states within $0.25$~m below the ceiling (\emph{near ceiling}); synthetic states below the unsafe reference apex, obtained by interpolating between unsafe reference endpoints (\emph{bridge}); and trace states near $\B$ (\emph{capture shell}).
Perturbed samples outside $\calS$ or inside $\B$ are rejected.
We approximate $\mu_{\mathrm{SA}}$ by drawing $1024$, $1024$, $1024$, and $512$ perturbed initial states from these pools, simulating the safe-arrival policy $\piarrL$ for $T=2.0$ sec, and counting the fraction that safely reach $\B$.

\paragraph{Phase~I curriculum and initial-state sampling.}
Training uses the train split and a curriculum scale $s\in[0,1]$ that shifts sampling from the general-trace and capture-shell pools toward the near-ceiling and bridge pools.
The latter pools receive the largest perturbations.
Curriculum settings are in Table~\ref{tab: app_phase1_hparams}.

\paragraph{Phase~II initial state sampling and warm-start.}
Initial positions are perturbed by $\pm0.1$~m about the powerloop start state, with candidates accepted only if their backup rollouts satisfy~\eqref{eq: bcbf_condition}.
The nominal policy is warm-started from a vanilla powerloop-tracking SAC checkpoint, following~\citep[App.~C.1]{hardnet}.
The CIL remains active throughout Phase~II training and evaluation.

\subsection{Baseline Implementation Details}\label{app: baseline_details}

All baselines use the same environments, task rewards, reference trajectories, initial-state sampler, and ten training seeds as \oursRA{}.
The hyperparameters used for the baselines intentionally match the \oursRA{} Phase II training settings whenever the same quantity is applicable, such as architecture and optimization settings in Table~\ref{tab: app_phase2_hparams}.
Baseline-specific changes are described below. 

\subsubsection{Safe RL via Violation Penalty}
The penalty baselines use the Phase~II SAC backbone without the CIL and optimize $r^{\mathrm{pen}}_k = r^{\mathrm{task}}_k - \lambda_{\mathrm{pen}}\,\mathbf{1}\{x_{k+1}\notin \calS\}$. 
We use $\lambda_{\mathrm{pen}}=1$ for SAC-Pen$_{\mathrm{low}}$ and $\lambda_{\mathrm{pen}}=1000$ for SAC-Pen$_{\mathrm{high}}$. 
All other hyperparameters and configurations remain the same as \oursRA.
To ensure a fair comparison and avoid poor initial exploration, we warm-started the quadrotor penalty baseline with a vanilla SAC tracking policy. 
This provided a strong tracking initialization, identical to our approach for \oursRA{}.

\subsubsection{Safe RL via Constrained Policy Optimization}
Constrained Policy Optimization (CPO)~\citep{cpo} and SAC-Lagrangian~\citep{ha2021sacLag,yang2021wcsac} optimize the task reward with binary cost $c_k=\mathbf{1}\{x_{k+1}\notin\calS\}$ and cost limit zero, using each method's expected-cost objective.

\subsubsection{Safe RL via Verified Certificate}\label{app: safe_rl_baseline_details}

\paragraph{CBF-RL.}
We implement CBF-RL from~\citep{cbfRL}, where a CBF is used for action filtering and reward shaping during training, and the resulting policy is deployed with the same closed-form filter:
\begin{align}\label{eq: cbfRL_filter}
    u_{\cbfrl}(x) = \begin{cases}
        \unom, & a(x)^\top\unom \ge b,\\
        \unom + \frac{b(x)-a(x)^\top \unom}{\| a(x) \|^2}a(x), & \text{otherwise,}
    \end{cases}
\end{align}
where $a(x) = g(x)^\top\nabla h_\C(x)$ and $b(x)=-\nabla h_\C(x)^\top f(x) -\alpha (h_\C(x))$ for a CBF $h_\C$, and $\unom$ is the output from a nominal policy.
Note that the original CBF-RL does not consider control limits. 
In our implementation, however, we choose to clip the filtered control input $u_\cbfrl$: $\usafe = \clip(u_\cbfrl, \U)$.  
During training,~\citep{cbfRL} penalizes unsafe behavior through reward shaping with a penalty term:
\begin{align}\label{eq: cbfRL_penalty}
    r_\cbfrl(x,u) = \min\left( a(x)^\top \unom -b(x), 0 \right) + \exp\left( -\frac{\| \unom - \usafe \|^2}{\sigma^2} -1 \right),
\end{align}
where $\sigma$ is a scaling scalar. 

For the unicycle task, we synthesize an SOS CBF using the formulation of~\citep{backupCBF}, implemented in MATLAB with the YALMIP toolbox~\citep{yalmip} and the MOSEK optimization suite~\citep{mosek}. 

For the quadrotor task, synthesis methods such as HJ reachability analysis and SOSP are not practical, as the system is 10-dimensional, nonlinear, quaternion-valued, and subject to tight actuator bounds. 
We therefore choose the high-order CBF (HOCBF) formulation~\citep{xiao2022hocbf}, which sidesteps formal synthesis by recursively building barrier conditions through Lie derivatives.
Since thrust has relative degree two while body-rate commands have a higher degree, the resulting HOCBF is affine only in $a_{\cmdtext}$, and the filter cannot fully exploit the available control authority. 
This is precisely the issue \oursRA{} avoids, as the induced BCBF constraints are relative-degree-one for all actuators by construction.

The nominal policy uses the common SAC backbone.
During training and evaluation, the CBF-RL filter~\eqref{eq: cbfRL_filter} is applied, with its output clipped to $\U$.
Training adds~\eqref{eq: cbfRL_penalty} to the task reward with weight $w_{\mathrm{cbf\,pen}}$.
We sweep $\alpha\in\{0.01,0.03,0.05,0.1,1.0,2.0\}$, $w_{\mathrm{cbf\,pen}}\in\{10,100\}$, and quadrotor warm-start/non-warm-start variants, with $\sigma=0.5$.


\paragraph{Model Predictive Shielding (MPS).}
MPS~\citep{bastani2021mps} combines a learned nominal policy $\hat\pi$, a learned recovery policy $\pi_{\mathrm{rec}}$, and a fixed equilibrium LQR controller $\pi_{\mathrm{eq}}$.
The shield accepts a nominal action when its successor can safely recover to the invariant base set within the recovery horizon.
Otherwise, it switches to $\pi_{\mathrm{rec}}$, or $\pi_{\mathrm{eq}}$ inside the base set.
If neither base-set membership nor recoverability holds, it uses $\pi_{\mathrm{eq}}$ as a best-effort fallback.
Both learned policies are trained with the differentiable model-based backpropagation-through-time (BPTT) procedure from~\citep{bastani2021mps}.
The nominal policy optimizes the task reward without a shield or safety penalty. The recovery policy is trained independently with the shaped reward from~\citep{bastani2021mps}.
The shield is applied at deployment.
Recovery horizons are 40 steps for unicycle and 106 steps for quadrotor, compared with 20 and 100 steps in \oursRA{}. 

MPS shares parts of the high-level structure of \oursRA{}: both methods use a learned recovery/arrival policy with an LQR-stabilized equilibrium, and both forward-integrate the dynamics under the policy to verify safe arrival to a small invariant set. 
However, MPS fundamentally differs from \oursRA{} by relying on a non-differentiable post-hoc shielding switch. 
Consequently, $\hat\pi$ is trained without exposure to the safety constraints.


\paragraph{\oursABP{}: \oursRA~with an analytic backup policy.} 
\oursABP{} is implemented as the same Phase II \oursRA{} algorithm as \oursSAV{}, but with Phase I safe-arrival policy learning disabled. 
That is, instead of using the learned safe-arrival policy $\piarrL$, \oursABP{} uses an analytic safe-arrival controller. 
The certified base set and LQR base controller remain the same as \oursSAV{}. 
Thus, the comparison between \oursABP{} and \oursSAV{} isolates the effect of the learned safe-arrival policy on the downstream RL task performance.
For the unicycle task, the analytic safe-arrival controller is the same LQR controller as the base controller~\eqref{app_eq: unicycle_base_controller}.
For the quadrotor task, we employ an aggressive cascaded PID-style recovery controller. 
The outer loop generates a virtual acceleration that damps translational motion, regulates altitude toward the hover region, and becomes more conservative near the ceiling. 
This virtual acceleration is converted into a thrust command and desired attitude, while an inner-loop quaternion attitude controller tracks that attitude until the trajectory enters the base set. 
After fixing $\piarr$, we train the \oursABP{} policy with the same Phase II settings as \oursSAV{}.

\section{Additional Results, Ablation Studies, and Computation}\label{app: additional_results}

\paragraph{Checkpoint selection and evaluation.}
Each method (\oursSAV{} and the seven baselines) is trained with ten common seeds.
Periodic evaluations use ten episodes to select one checkpoint per seed, which is then evaluated on 1,000 episodes, yielding 10,000 evaluation episodes per method.
All reported values in Tables~\ref{tab: hero_unicycle},~\ref{tab: hero_quadrotor},~\ref{tab: app_unicycle}, and~\ref{tab: app_quadrotor} are aggregated across these 10,000 episodes per method.

\paragraph{Metrics.}
We compute episode-level tracking RMSEs (root mean square errors), safety indicators, maximum constraint violation, and cumulative exceedance.
An episode is safe if every state satisfies every constraint.
Tables~\ref{tab: hero_unicycle}--\ref{tab: hero_quadrotor} report the interquartile mean (IQM) and 95\% confidence intervals across the ten seed-level summaries, the aggregate safe-episode fraction over all episodes, and the largest violation over all episodes.
Tables~\ref{tab: app_unicycle}--\ref{tab: app_quadrotor} use the same data to report mean $\pm$ standard deviation across seed-level summaries and the minimum per-seed safety rate.

\subsection{Unicycle: Extended Analysis}\label{app: unicycle_extended_analysis}

\paragraph{Phase I analysis.}
The learned safe-arrival policy increases the safe-arrival fraction over $\Omega$ from $0.227$ to $0.326$, while retaining $99.08\%$ of the analytic safe-arrival policy's recoverable grid points.
At $v_{\des}=5$~m/s, its recoverable slice has $1.43$ times the analytic slice area, shown in Fig.~\ref{fig: hero_unicycle}(b).

\begin{table}[htpb]
    \centering
    \caption{Unicycle experiment results, across training seeds, each evaluated for 1,000 episodes. 
    The table reports the mean and standard deviation across the 10 per-seed averages of episode-level metrics.
    ``Worst-seed Safe \%'' is the lowest episode safety rate among the 10 trained models.}
    \label{tab: app_unicycle}
    \resizebox{1.0\linewidth}{!}{
    \begin{tabular}{cl ccc ccc}
        \toprule
        & & \multicolumn{3}{c}{\textbf{Tracking Performance (RMSE)}} & \multicolumn{3}{c}{\textbf{Safety Performance}}\\
        \cmidrule(lr){3-5} \cmidrule(lr){6-8}
        & \textbf{Safe RL} & $y$ (m) & $v$ (m/s) & $\psi$ (rad) &Worst-seed Safe \% & Mean Max. Viol. (m) & Mean Cumul. Viol. (m)  \\
        \midrule
        \multirow{2}{*}{\rotatebox[origin=c]{90}{\texttt{Pen}}}
        & SAC-Pen$_{\text{low}}$ & 0.68 $\pm$ 0.48 & 2.18 $\pm$ 1.13 & 0.26 $\pm$ 0.27 & 0.0\% & 0.77 $\pm$ 0.73 & 54.12 $\pm$ 46.30 \\
        & SAC-Pen$_{\text{high}}$ & 1.13 $\pm$ 0.23 & 2.08 $\pm$ 1.08 & 0.16 $\pm$ 0.05 & 100\% & 0.00 $\pm$ 0.00 & 0.00 $\pm$ 0.00 \\
        \cmidrule(lr){2-8}
        \multirow{2}{*}{\rotatebox[origin=c]{90}{\texttt{Con}}}
        & SAC-Lag. & 0.54 $\pm$ 0.18 & 1.28 $\pm$ 0.41 & 0.12 $\pm$ 0.04 & 0.0\% & 0.40 $\pm$ 0.38 & 26.41 $\pm$ 31.06 \\
        & CPO & 4.83 $\pm$ 6.25 & 4.25 $\pm$ 1.63 & 0.41 $\pm$ 0.37 & 0.0\% & 7.01 $\pm$ 13.66 & 1005.29 $\pm$ 2000.82 \\
        \cmidrule(lr){2-8}
        \multirow{4}{*}{\rotatebox[origin=c]{90}{\texttt{PSRL}}}
        & CBF-RL & 1.76 $\pm$ 0.06 & 3.48 $\pm$ 1.13 & 0.23 $\pm$ 0.03 & 100\% & 0.00 $\pm$ 0.00 & 0.00 $\pm$ 0.00 \\
        & MPS & 0.67 $\pm$ 0.16 & 0.48 $\pm$ 0.08 & 0.17 $\pm$ 0.03 & 96.8\% & 0.004 $\pm$ 0.001 & 0.03 $\pm$ 0.008 \\
        & \textbf{\textcolor{oursABP}{\oursABP}} & 0.98 $\pm$ 0.00 & 0.45 $\pm$ 0.15 & 0.15 $\pm$ 0.00 & 100\% & 0.00 $\pm$ 0.00 & 0.00 $\pm$ 0.00  \\
        & \textbf{\textcolor{oursRA}{\oursSAV}} & 0.53 $\pm$ 0.01 & 0.80 $\pm$ 0.05 & 0.11 $\pm$ 0.01 & 100\% & 0.00 $\pm$ 0.00 & 0.00 $\pm$ 0.00 \\
        \bottomrule
    \end{tabular}
    }
\end{table}

\begin{figure}[htbp]
    \centering
    \includegraphics[width=1.0\linewidth]{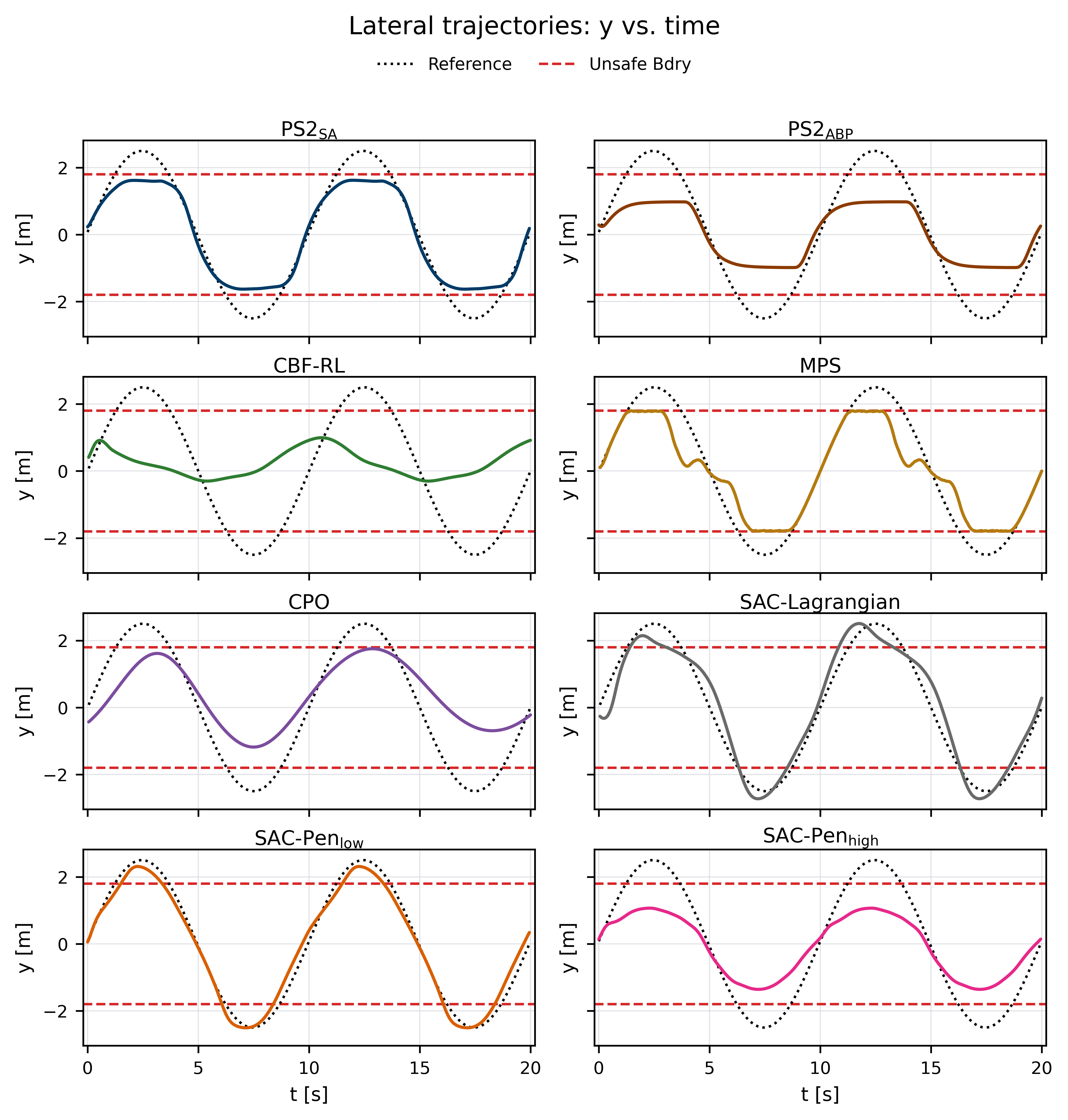}
    \caption{Representative lateral trajectories on the unicycle lane-keeping task.
    Each panel plots the lateral position $y(t)$ of one method over a 20-second rollout.
    The dotted black curve denotes the unsafe sinusoidal reference, and the dashed red lines denote the lane safety boundary $|y|=y_{\max}=1.8\,\mathrm{m}$.
    \textbf{\textcolor{oursRA}{\oursSAV}} closely tracks the reference while remaining inside the safe lane, whereas \textbf{\textcolor{oursABP}{\oursABP}} is safe but more conservative due to the smaller analytic-backup-induced admissible set.
    CBF-RL and SAC-Pen$_{\mathrm{high}}$ also remain safe but sacrifice substantial tracking performance, MPS exhibits intervention-induced nonsmooth behavior near the boundary, and the penalty/CMDP methods with weaker effective safety enforcement violate the lateral constraint.
    }
    \label{fig: unicycle_lateral_traces}
\end{figure}

\paragraph{Phase II analysis.}
The Phase II results in Tables~\ref{tab: hero_unicycle} and~\ref{tab: app_unicycle} are consistent with this enlargement of the safe-arrival set. 
Compared with \oursABP{}, \oursSAV{} reduces lateral error by roughly $46\%$ and heading error by roughly $27\%$, while preserving $100\%$ safety. 
The velocity error of \oursSAV{} is higher than that of \oursABP{} and MPS, but this reflects a different safety-performance tradeoff: \oursABP{} and MPS stay closer to a conservative recovery behavior, whereas \oursSAV{} uses the larger admissible set induced by the learned backup policy to track the unsafe lateral reference more closely. 
After accounting for both the diagonal weights (50 for $y$ vs. 20 for $v$) and the per-component normalizers (1.8 vs. 5), the reward penalizes squared lateral error roughly 19$\times$ more heavily than squared velocity error, making this tradeoff a clear net gain under the task objective. 

Table~\ref{tab: app_unicycle} also exposes the instability of methods whose safety is not enforced pointwise. 
SAC-Pen$_\mathrm{low}$ and SAC-Lagrangian have competitive average tracking on some episodes, but their worst-seed safety is $0\%$, with nonzero mean maximum and cumulative violations. 
CPO is particularly heavy-tailed, with large tracking variance and very large cumulative violations. 
CBF-RL's low performance is driven by its suboptimal SOS CBF, which unnecessarily constrains the resulting policy.
For MPS, the rare violations appear to arise from chattering in the switching shield. 
The MPS rollout shown in Fig.~\ref{fig: unicycle_lateral_traces} illustrates this behavior: the shielded trajectory repeatedly rides the safety boundary and then undergoes abrupt intervention-induced switches.

\subsection{Quadrotor: Extended Analysis}\label{app: quadrotor_extended_analysis}

\paragraph{Phase I analysis.}
Over $\Omega$, the learned safe-arrival policy increases the recoverability rate from $60\%$ for the analytic policy to $84\%$. 
The gains are largest where the analytic safe-arrival policy limits downstream tracking performance. 
On near-ceiling states, recoverability improves from $35.5\%$ to $69.3\%$; on bridge states below the unsafe loop apex, it improves from $47.3\%$ to $78.6\%$; on general trace states, it improves from $77.7\%$ to $96.1\%$; and on capture-shell states, it reaches $100\%$. 
Fig.~\ref{fig: app_quad_backup_full} visualizes this difference. 
The two rows show safe-arrival rollouts from task-relevant quadrotor states, with side and front views. 
While the learned safe-arrival policy bends the trajectories away from the ceiling and into the base set, the analytic policy is less effective from near-ceiling and bridge states. 

\begin{figure}
    \centering
    \includegraphics[width=0.8\linewidth]{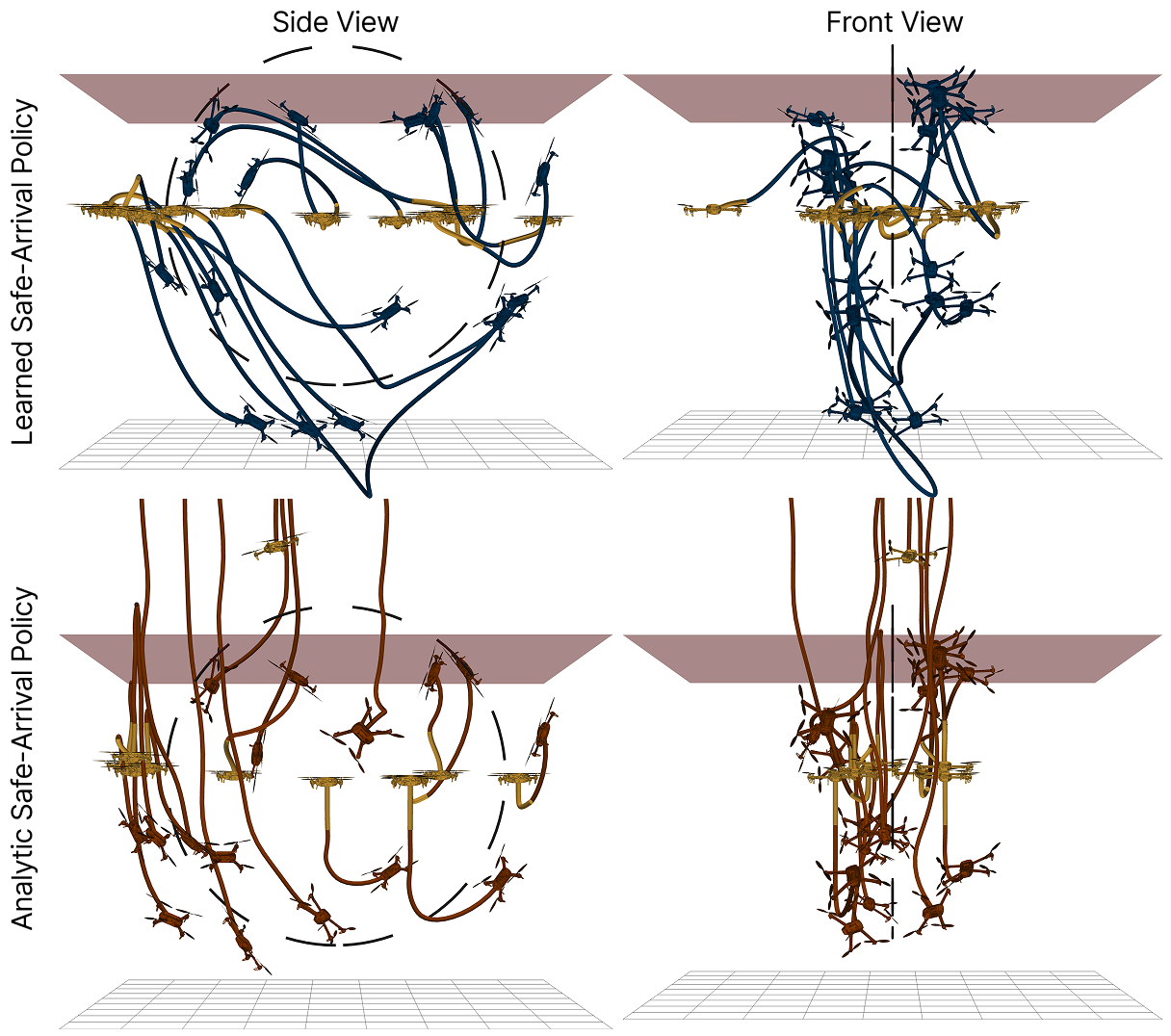}
    \caption{Safe-arrival policies in action: (top) \textbf{\textcolor{oursRA}{$\piarrL$ learned policy}} and (bottom) \textbf{\textcolor{oursABP}{$\piarr$ analytic policy}}. The \textbf{\textcolor{lqrYellow}{base set trajectory segments}} are controlled by the \textbf{\textcolor{lqrYellow}{base controller $\pi_\B$}}.}
    \label{fig: app_quad_backup_full}
\end{figure}

\paragraph{Phase II analysis.}
The Phase II results in Tables~\ref{tab: hero_quadrotor} and~\ref{tab: app_quadrotor} show that the larger learned backup-induced set directly improves the task performance. 
Among the methods with $100\%$ safety in every seed, \oursSAV{} has the best tracking performance.
While \oursABP, CBF-RL, and MPS baselines achieve 100\% safety, they do so by inducing a smaller or less task-aligned admissible set, which forces the task policy to cut the loop and deviate from the reference. 

SAC-Pen$_\mathrm{low}$ tracks parts of the loop well but has $0\%$ worst-seed safety. 
SAC-Pen$_\mathrm{high}$ is almost safe, but still has a worst-seed safety rate of $99.8\%$ and is substantially more conservative than \oursSAV{}. 
Both CPO and SAC-Lagrangian have $0\%$ worst-seed safety. 
CBF-RL's decrease in performance is mainly because the quadrotor HOCBF constraint acts directly only on thrust, with no dependence on the body rate commands that are crucial for agile attitude recovery. 
MPS, by contrast, remains safe in this task, but is much more conservative due to switching into recovery/equilibrium behavior far away from the ceiling.
Fig.~\ref{fig:app_quadrotor_all_trajectories} illustrates these distinctions.

\begin{table}[htpb]
    \centering
    \caption{Quadrotor experiment results, across training seeds, each evaluated for 1,000 episodes. 
    The table reports the mean and standard deviation across the 10 per-seed averages of episode-level metrics.
    ``Worst-seed Safe \%'' is the lowest episode safety rate among the 10 trained models.}
    \label{tab: app_quadrotor}
    \resizebox{1.0\linewidth}{!}{
    \begin{tabular}{cl ccc ccc}
        \toprule
        & & \multicolumn{3}{c}{\textbf{Tracking Performance (RMSE)}} & \multicolumn{3}{c}{\textbf{Safety Performance}}\\
        \cmidrule(lr){3-5} \cmidrule(lr){6-8}
        & \textbf{Safe RL} & $\mathbf{p}$ (m) & $\mathbf{v}$ (m/s) & $\theta_\mathbf{q}$ (rad)& Worst-seed Safe \% & Mean Max. Viol. (m) & Mean Cumul. Viol. (m)  \\
        \midrule
        \multirow{2}{*}{\rotatebox[origin=c]{90}{\texttt{Pen}}}
        & SAC-Pen$_{\text{low}}$ & 1.04 $\pm$ 0.45 & 1.95 $\pm$ 0.65 & 1.17 $\pm$ 0.52 & 0.0\% & 0.15 $\pm$ 0.18 & 2.71 $\pm$ 3.68 \\
        & SAC-Pen$_{\text{high}}$ & 1.25 $\pm$ 0.33 & 2.24 $\pm$ 0.51 & 1.54 $\pm$ 0.39 & 99.8\% & 0.00 $\pm$ 1.2e-6 & 8.0e-7 $\pm$ 2.4e-6 \\
        \cmidrule(lr){2-8}
        \multirow{2}{*}{\rotatebox[origin=c]{90}{\texttt{Con}}}
        & SAC-Lag. & 2.14 $\pm$ 0.70 & 4.29 $\pm$ 0.50 & 2.01 $\pm$ 0.35 & 0.0\% & 0.28 $\pm$ 0.48 & 9.54 $\pm$ 17.98 \\
        & CPO & 10.28 $\pm$ 4.36 & 11.64 $\pm$ 3.35 & 2.19 $\pm$ 0.19 & 0.0\% & 0.42 $\pm$ 0.75 & 16.02 $\pm$ 30.67 \\
        \cmidrule(lr){2-8}
        \multirow{4}{*}{\rotatebox[origin=c]{90}{\texttt{PSRL}}}
        & CBF-RL & 3.82 $\pm$ 1.27 & 5.41 $\pm$ 1.18 & 2.02 $\pm$ 0.27 & 100\% & 0.00 $\pm$ 0.00 & 0.00 $\pm$ 0.00 \\
        & MPS & 2.00 $\pm$ 0.08 & 4.07 $\pm$ 0.25 & 1.41 $\pm$ 0.12 & 100\% & 0.00 $\pm$ 0.00 & 0.00 $\pm$ 0.00 \\
        & \textbf{\textcolor{oursABP}{\oursABP}} & 1.35 $\pm$ 0.15 & 1.70 $\pm$ 0.13 & 0.78 $\pm$ 0.09 & 100\% & 0.00 $\pm$ 0.00 & 0.00 $\pm$ 0.00 \\
        & \textbf{\textcolor{oursRA}{\oursSAV}} & 0.61 $\pm$ 0.10 & 0.88 $\pm$ 0.29 & 0.42 $\pm$ 0.11 & 100\% & 0.00 $\pm$ 0.00 & 0.00 $\pm$ 0.00 \\
        \bottomrule
    \end{tabular}
    }
\end{table}

\begin{figure}
    \centering
    \includegraphics[width=0.9\linewidth]{figures/quadrotor_all_trajectories.png}
    \caption{Quadrotor powerloop trajectories for all evaluated methods.
    Each method is shown with side and front views of a representative rollout.
    The dashed black line denotes the unsafe powerloop reference, and the semi-transparent red plane denotes the hard ceiling at $z_{\mathrm{ceil}}=3$m.
    Note that while we show the ground as a grid, there is no safety constraint for it.
    \textbf{\textcolor{oursRA}{\oursSAV}} tracks the aggressive reference most closely while remaining below the ceiling. \textbf{\textcolor{oursABP}{\oursABP}} and the other certified/shielded baselines are more conservative, while penalty- and CMDP-based baselines either deviate substantially from the maneuver or exhibit ceiling violations.}
    \label{fig:app_quadrotor_all_trajectories}
\end{figure}

\subsection{Additional Study: Quadrotor with Multiple Simultaneous Constraints}\label{app: quadrotor_three_constraints}

The safety constraints are $h_{\calS,1}(x)=3-p_z$, $h_{\calS,2}(x)=1+p_x$, and $h_{\calS,3}(x)=1-p_x$. 
We extend the hover-error coordinates of App.~\ref{app: quadrotor_phase_I} to $x_{e,8}=[p_x,x_e^\top]^\top$, prepend the discrete position weight $0.02\times50=1$ to the diagonal of $Q_d$, and recompute $P_8$ and $K_8$ with unchanged $R_d$. 
The new base set is $\B=\{x:x_{e,8}^\top P_8x_{e,8}\le8\}$, using the same clipped hover-feedback form. 
Its bounds $|p_x|\le0.603$~m and $|p_z-2|\le0.491$~m lie strictly inside all three safety boundaries.

\paragraph{Training setup.}
Architectures, optimization hyperparameters, and budgets in
Tables~\ref{tab: app_phase1_hparams}--\ref{tab: app_phase2_hparams} are unchanged. 
Phase~I trains a fresh safe-arrival policy, adding near-wall and wall-bridge pools on both sides to its sampling and evaluation distributions. 
The existing near-boundary and bridge sampling probabilities are divided equally among the three boundaries. 
Dynamics, reference, reward, and the Phase~II vanilla SAC warm start are unchanged. 
One newly trained, frozen backup is shared across ten Phase~II seeds.

\paragraph{Results.}
Each of the resulting policies is evaluated on the same 1,000 initial states as the single-constraint study. 
Position, velocity, and attitude RMSE IQMs [95\% CIs] are $0.74\,[0.72,0.77]$ m, $1.72\,[1.66,1.78]$ m/s, and $1.14\,[1.10,1.19]$ rad, respectively.
Every seed achieves 100\% sampled-episode safety, with zero maximum violation across all three constraints. 
Phase~I takes $2594.19$~s to train. 
Phase~II training and inference take $51{,}775.69\pm2519.26$~s and $821.52\pm96.29\,\mu\mathrm{s}$, respectively.

\subsection{Additional Study: Recovering the Maximal Invariant Set}\label{app: hj_comparison}

We evaluate how closely the learned backup-induced set approaches the maximal control-invariant set when a longer horizon allows more time to reach the base set.
This study shows that Phase~I can approach the maximal control-invariant set coverage, not merely improve over an analytic backup. 

\paragraph{Training setup.}
For the set-coverage comparison, we train a new safe-arrival policy with an extended backup horizon of $T=3\,\mathrm{s}$, increasing the shorter $T=1\,\mathrm{s}$ horizon used in the standard unicycle tracking experiments.
The dynamics, input bounds, safety constraints, base set and controller, and training curriculum remain as specified in App.~\ref{app: ps2_implementation_details}.
All Phase~I hyperparameters, other than $\beta=0.98$, in Table~\ref{tab: app_phase1_hparams} are unchanged.

\paragraph{HJ reference and set coverage.}
We compare the resulting backup-induced set with a numerical approximation of the maximal control-invariant set computed via HJ reachability~\citep{bansal2017hamilton}.
On a $256\times241$ grid in $(y,\psi)$ at $v=5\,\mathrm{m/s}$, the learned set covers $98.18\%$ of the HJ-safe grid points.

\subsection{Ablation Studies: Role of the Control-Invariant Layer}\label{app: backup-layer-ablation}

We perform paired ablations on the quadrotor powerloop task to isolate the two roles of the control-invariant layer (CIL): providing safety guarantees at deployment and end-to-end constrained policy optimization during training. 
For these ablation studies, we select the best-tracking \oursSAV{} checkpoint among the ten quadrotor seeds and evaluate three policies:
\oursSAV{} is the full policy; 
\oursSAV{} w/o CIL uses the same trained policy network but disables the projection layer at deployment; 
Vanilla+CIL uses the unconstrained tracking policy that warm-started this \oursSAV{} run, but applies the same CIL only at evaluation. 
All active CILs use the same learned backup policy.

\begin{table}[htpb]
    \centering
    \caption{
    Paired ablation on the role of the control-invariant layer (CIL) in the quadrotor powerloop task. 
    All metrics are averaged over 1,000 episodes. 
    Violation metrics are ceiling violations in $p_z$.
    }
    \label{tab: app_ablation_quadrotor}
    \resizebox{1.0\linewidth}{!}{%
    \begin{tabular}{lccc ccc}
        \toprule
        & \multicolumn{3}{c}{\textbf{Tracking Performance}} & \multicolumn{3}{c}{\textbf{Safety Performance}}\\
        \cmidrule(lr){2-4} \cmidrule(lr){5-7}
        \textbf{Safe RL} & $\mathbf{p}$ (m) & $\mathbf{v}$ (m/s) & $\theta_\mathbf{q}$ (rad)& Total Safety \% & Mean Max. Viol. (m) & Mean Cumul. Viol. (m) \\
        \midrule
        Vanilla + CIL & 0.9082 & 1.1215 & 0.4456 & 100\% & 0.00 & 0.00 \\
        \oursSAV w/o CIL & 0.8440 & 2.1063 & 0.8501 & 0.0\% & 0.4896 & 8.4751 \\
        \cmidrule(lr){1-7}
        \textbf{\textcolor{oursRA}{\oursSAV}} & 0.4974 & 0.6328 & 0.2644 & 100\% & 0.00 & 0.00 \\
        \bottomrule
    \end{tabular}
    }
\end{table}

\paragraph{Deployment ablation: removing the CIL in deployment.}
The safety guarantee of \oursRA{} is a guarantee on the composed projected policy $\calP_{\BL}(\pi_\phi)$, not on the raw nominal network $\pi_\phi$ alone. 
Table~\ref{tab: app_ablation_quadrotor} confirms that this distinction is essential. 
When the CIL is disabled after training, the resulting \oursSAV{} w/o CIL crosses the ceiling in every evaluation episode, producing 0\% safety, where any ceiling crossing marks the episode unsafe. 
Figure~\ref{fig: app_albation_quadrotor} visualizes this failure. 
The unfiltered policy's failures are small altitude overshoots. 
The larger cumulative violation reflects these small exceedances accumulated across violating timesteps, rather than large off-grid excursions. 
This shows that the policy network has not simply internalized the hard safety constraint, but has learned to optimize task performance through the projection layer. 
The CIL is therefore the safety-guaranteeing component that must remain active at deployment.

\begin{figure}
    \centering
    \includegraphics[width=0.9\linewidth]{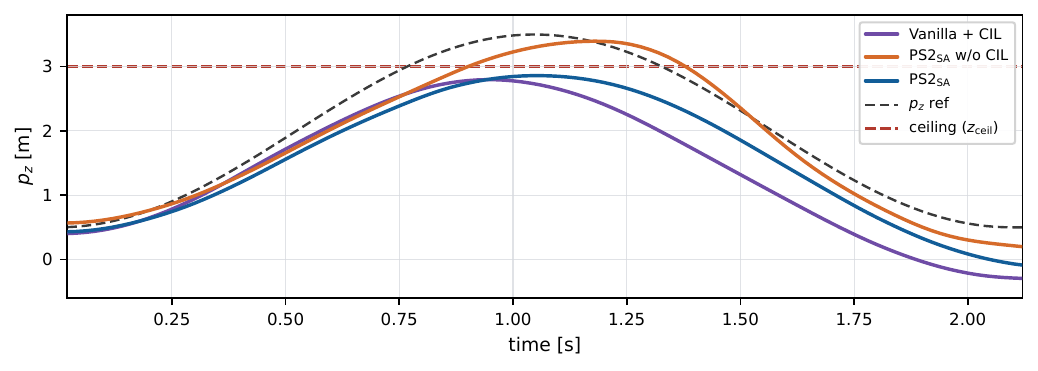}
    \caption{Altitude traces for the control-invariant layer ablation. 
    For each policy, we plot the best-tracking episode, measured by tracking error, among its 1,000 evaluation episodes. 
    The dashed black curve is the unsafe powerloop reference and the dashed red line is the ceiling $z_{\mathrm{ceil}}=3\,\mathrm{m}$. 
    }
    \label{fig: app_albation_quadrotor}
\end{figure}

\paragraph{Training ablation: post-hoc projection without CIL in training.}
The Vanilla+CIL row evaluates the opposite ablation: an unconstrained vanilla tracker is shielded by the CIL only at deployment. 
This policy is safe in all episodes, as expected from the projection onto the BCBF-admissible set. 
However, its tracking performance is substantially worse than the full \oursSAV{} policy. 
Training end-to-end through the CIL improves the position, velocity, and attitude errors by approximately 45.2\%, 43.6\%, and 40.7\%, respectively, compared to the post-hoc shielding. 
Thus, while the CIL alone can enforce safety at evaluation time, exposing the RL policy to the same projection during training is what allows the policy to learn high-performance actions within the certified admissible set. 

\subsection{Computation Details and Time}
\label{app: computation_time}

\paragraph{Computation resources.}
All experiments were run as single-GPU jobs with one NVIDIA V100 GPU. 
We used JAX with CUDA for batched environment rollouts, policy updates, backup-flow integration, and CIL evaluation. 
No multi-GPU training was used. 

\paragraph{Timing protocol.}
We report two types of timing. 
First, Phase I training time is the wall-clock time required to train the \sapolicy{}. 
This cost is paid only for \oursSAV{} and only once per \sapolicy{} design. 
\oursABP{} uses an analytic backup policy and therefore has no Phase I learning cost. 
Second, Phase II training time is the wall-clock time of the \ours{} policy training job with the CIL. 
The Phase II timing includes the repeated construction of the BCBF-admissible set, the CIL solves, and policy/critic updates. 
Inference time is measured as the per-control-step state-to-action latency of the deployed \ours\ policy after JAX compilation and warm-up. 
It includes backup rollout evaluation, BCBF constraint assembly, and the QP solve, but excludes environment stepping and logging. 
All Phase II training and inference statistics in Table~\ref{tab: app_comp_time} are averaged over the same 10 random seeds used in the main experiments.

\paragraph{Phase I safe-arrival training cost.}
Phase I \sapolicy training required 842.24 seconds for the unicycle task and 2408.89 seconds for the quadrotor task. 

\paragraph{Phase II training and deployment cost.}
Table~\ref{tab: app_comp_time} reports the computational cost of \oursRA{}. 
The Phase II training times are on the order of 4 hours for the unicycle task and 12--13 hours for the quadrotor task. 
The measured per-step inference latency is well below the environment control period in both tasks: for \oursSAV{}, the unicycle latency is $0.34$ ms and the quadrotor latency is $0.62$ ms. 

\begin{table}[h!]
    \centering
    \caption{Computation time for \oursRA{}, averaged across 10 Phase II random seeds.}
    \label{tab: app_comp_time}
    \resizebox{1.0\linewidth}{!}{%
    \begin{tabular}{lcc cc}
        \toprule
        & \multicolumn{2}{c}{\textbf{Unicycle: Lane-Keeping}} 
        & \multicolumn{2}{c}{\textbf{Quadrotor: Powerloop}}\\
        \cmidrule(lr){2-3} \cmidrule(lr){4-5}
        \textbf{PS2-RL} 
        & Training Time (s) 
        & Inference Time ($\mu$s) 
        & Training Time (s) 
        & Inference Time ($\mu$s)\\
        \midrule
        \textcolor{oursABP}{\oursABP} & 14278.28 $\pm$ 722.26 & 346.38 $\pm$ 5.25 & 45145.79 $\pm$ 1925.79 & 654.26 $\pm$ 75.37 \\
        \textbf{\textcolor{oursRA}{\oursSAV}} & 15646.84 $\pm$ 1239.18 & 340.95 $\pm$ 18.24 & 45894.58 $\pm$ 5860.16 & 623.02 $\pm$ 59.54 \\
        \bottomrule
    \end{tabular}
    }
\end{table}

\end{document}